\documentclass[twoside,11pt]{article}

\usepackage{jmlr2e}

\ShortHeadings{The Many-to-Many Mapping Between $CCC$ and $MSE$}{Pandit and Schuller}
\firstpageno{1}

\usepackage{amsmath}
\usepackage{graphicx}
\usepackage{enumerate}
\usepackage{natbib}
\usepackage{url} % not crucial - just used below for the URL 

\usepackage{amsmath,amsfonts,amssymb,multicol}
\usepackage{graphicx}
\usepackage{cleveref,pgfplots,cancel}
\usepackage{threeparttable,array,subcaption,wrapfig,float}
\usepackage{bm,xcolor,multirow,tabularx}
\usepackage{mathtools}
\usepackage{xfrac}
\newcolumntype{H}{>{\setbox0=\hbox\bgroup}c<{\egroup}@{}}

\usepackage{accsupp}
\usetikzlibrary{decorations,decorations.markings,decorations.text}
\usetikzlibrary{shapes.misc}
\usepackage{tikzpagenodes}
\usetikzlibrary{tikzmark}
\tikzset{cross/.style={cross out, draw=black, minimum size=2*(#1-\pgflinewidth), inner sep=0pt, outer sep=0pt},
	cross/.default={1pt}}
\usetikzlibrary{arrows, arrows.meta, automata, backgrounds, bending, 
	calc, decorations.text, decorations.pathreplacing,
	fit, matrix, positioning, quotes,
	shadows, shapes,shadows.blur}
\usepgfplotslibrary{fillbetween}

\allowdisplaybreaks
\newcommand{\ie}{i.\,e.,\ }
\newcommand{\eg}{e.\,g.,\ }

\tikzset{
	font={\fontsize{8pt}{9}\selectfont}}
\tikzset{every picture/.style={/utils/exec={\rmfamily}}}
\pgfkeys{/pgf/decoration/.cd,
	distance/.initial=10pt
}  

\pgfdeclaredecoration{add dim}{final}{
	\state{final}{% 
		\pgfmathsetmacro{\dist}{5pt*\pgfkeysvalueof{/pgf/decoration/distance}/abs(\pgfkeysvalueof{/pgf/decoration/distance})} 
		\pgfpathmoveto{\pgfpoint{0pt}{0pt}}             
		\pgfpathlineto{\pgfpoint{0pt}{2*\dist}}   
		\pgfpathmoveto{\pgfpoint{\pgfdecoratedpathlength}{0pt}} 
		\pgfpathlineto{\pgfpoint{(\pgfdecoratedpathlength}{2*\dist}}
		\pgfusepath{stroke} 
		\pgfsetdash{{0.1cm}{0.1cm}{0.1cm}{0.1cm}}{0cm}     
		\pgfsetarrowsstart{latex}
		\pgfsetarrowsend{latex}  
		\pgfpathmoveto{\pgfpoint{0pt}{\dist}}
		\pgfpathlineto{\pgfpoint{\pgfdecoratedpathlength}{\dist}} 
		\pgfusepath{stroke} 
		\pgfsetdash{}{0pt}
		\pgfpathmoveto{\pgfpoint{0pt}{0pt}}
		\pgfpathlineto{\pgfpoint{\pgfdecoratedpathlength}{0pt}}
}}

\tikzset{dim/.style args={#1,#2}{decoration={add dim,distance=#2},
		decorate,
		postaction={decorate,decoration={text along path,
				raise=#2,
				text align={align=center},
				text={#1}}}}}
\Crefname{equation}{Equation}{Equations}
\crefname{equation}{Eq.}{Eqs.}
\begin{document}

\title{The Many-to-Many Mapping Between \\the Concordance Correlation Coefficient, \\ and the Mean Square Error}

\author{\name Vedhas Pandit \email 
%	panditvedhas@gmail.com
vedhas.pandit@informatik.uni-augsburg.edu 
	\\
       \addr Chair of Embedded Intelligence for Health Care and Wellbeing,\\ University of Augsburg, Germany
       \AND
       \name Bj\"orn Schuller \email schuller@informatik.uni-augsburg.edu \\
       \addr Chair of Embedded Intelligence for Health Care and Wellbeing,\\ University of Augsburg, Germany\\
       Group on Language, Audio, and Music (GLAM),\\Imperial College London, United Kingdom}

%\editor{Francis Bach, David Blei and Bernhard Sch{\"o}lkopf}

\maketitle

\begin{abstract}%   <- trailing '%' for backward compatibility of .sty file
	We derive the mapping between two of the most pervasive utility functions, the mean square error ($MSE$) and the concordance correlation coefficient ($CCC$, $\rho_c$). Despite its drawbacks, $MSE$ is one of the most popular performance metrics (and a loss function);  along with lately $\rho_c$  in many of the sequence prediction challenges. Despite the ever-growing simultaneous usage, \eg inter-rater agreement, assay validation, a mapping between the two metrics is missing, till date. While minimisation of $L_p$ norm of the errors or of its positive powers (\eg $MSE$) is aimed at $\rho_c$ maximisation, we \emph{reason} the often-witnessed ineffectiveness of this popular loss function with graphical illustrations. The discovered formula uncovers not only the counterintuitive revelation that `$MSE_1<MSE_2$' does \emph{not} imply `$\rho_{c_1}>\rho_{c_2}$', but also provides the precise range for the $\rho_c$ metric for a given $MSE$. We discover the conditions for $\rho_c$ optimisation for a given $MSE$; and as a logical next step, for a given set of errors. We generalise and discover the conditions for any given $L_p$ norm, for an even $p$. We present newly discovered, albeit apparent, mathematical paradoxes. The study inspires and anticipates a growing use of $\rho_c$-inspired loss functions \eg $\left|\sfrac{MSE}{\sigma_{XY}}\right|$, replacing the traditional $L_p$-norm loss functions in multivariate regressions.
\end{abstract}

\begin{keywords}
Multivariate Analysis, Concordance, Correlation, $L_p$ norms, Mapping
\end{keywords}

\section{Introduction}
\label{s:intro}

The need to quantify inter-rater, inter-device or inter-method agreement arises often in almost every research field 
%-- \eg chemistry, physics, biology, or any sub-domain, \eg astronomy, energy science, ecology, psychology, sociology, or health
\citep{atkinson1998statistical,conroy2003estimation,lombard2002content,deyo1991reproducibility,banerjee1999beyond}.
This includes, for example, a comparison between a gold standard sequence (\eg device measurements) against the prediction sequences from a trained machine learning model, or the annotation sequences from another independent observer.

\newpage
\subsection{Literature survey: Distance and similarity metrics}
For comparisons of this type, 
one of the most popular distance metrics in use today is the mean square error ($MSE$).
%, which is a distance metric.
%the mean square error ($MSE$) is a popular performance metric. 
It measures the average squared error, \ie the average squared difference between the two variables \citep{willmott2005advantages,fisher1920012}.
%One of the many disadvantages of $MSE$ is that it heavily weighs outliers. 
However, 
%the interpretation of $MSE$ metric is incomplete without knowing the scale of measurements. In other words, 
$MSE$ 
%can not be used as a standalone or an absolute metric, but is rather a relative one, as we 
requires a further magnitude comparison against the measurements themselves for any meaningful interpretation. 
$MSE$, as a utility function, has also been criticised because of the unboundedness and the convexity of the function \citep{berger2013statistical}. Furthermore, the $MSE$ metric fails to capture correlated variations of the quantities being measured (\ie whether a greater value of one corresponds to a greater value of the other).  Carl Friedrich Gauss, who himself proposed the square of the error as a measure of loss or inaccuracy, too admitted to $MSE$'s shortcomings and arbitrariness \citep{sheynin1979cf}. 
%His defense to this choice has been quoted to be simply 
He defended his choice to be simply
\emph{``ein bloss auf Principen der Zweckmassigkeit basierende'' }\cite[p.~371]{gauss1860briefwechsel}, or as \emph{``an appeal to mathematical simplicity and convenience''} \citep[p.~6]{lehmann2006theory}. In summary, while popular, $MSE$ does not serve as a standalone reliable performance metric for a good number of reasons. Other metrics based on $L_p$ norm of the errors, \eg mean absolute error ($MAE$), too suffer from exactly the same problems. 

Given the populations $X:=(x_i)_1^N$ and $Y:=(y_i)_1^N$, we have:
%$MSE$ and $MAE$ are defined by: 
\vspace{-0.3cm}
\begin{align}
L_p 
%&
= {\big[\sum _{i=1}^{N}|x_{i}-y_{i}|^p\big]}^\frac{1}{p}, 
%\\
%\implies
MSE
%&
=\frac{
	\displaystyle
	\sum _{i=1}^{N}(x_{i}-y_{i})^2}{N} 
= \frac{{L_2}^2}{N},
%\\
%\quad
%\quad
MAE
%&
=\frac{
	\displaystyle
	\sum _{i=1}^{N}\left|x_{i}-y_{i}\right|}{N}
= \frac{L_1^1}{N}.
\label{eqnLpDef}
\end{align}

A quest for a summary statistic--that effectively quantifies the extent of   similarity and association (\ie dependence or joint variability) between two variables--has led researchers to invent several indices. \cite{brockmeier2017quantifying} propose an unsupervised measure to quantify informativeness of various similarity measures when used to compute correlation matrices.
For comparing the two rankings, metrics such as Kendall's Tau \citep{kendall1938new}, Spearman's rank correlation coefficient \citep{spearman1961proof}, Quotient correlation \citep{zhang2008quotient} have been devised. As for the nominal and ordinal classification tasks,  Cohen's kappa coefficient ($\kappa$) \citep{smeeton1985early,galton1892finger,cohen1960coefficient}, intraclass correlation coefficient ($ICC$) \citep{fisher1925statistical,koch2004intraclass}, sequence-centric distance functions \citep{rieck2008linear}, separation distance and rate \citep{hernandez2012unified,collier2016minimax} are some of the popular metrics. 
While a covariance metric quantifies correlated variations of the quantities being measured, 
similar to $MSE$, 
%it requires the context of magnitudes of the populations for a meaningful interpretation.
this metric too 
is impossible to interpret without knowing
the relative magnitudes of those measurements.
%is meaningless without
%the magnitudes of the quantities being measured.
 A normalised covariance metric, called the Pearson correlation coefficient ($\rho$)  \citep{galton1877typical,galton1877typicala,pearson1895note}, quantifies the strength of the linear relationship between two variables, ignoring the bias and the scale.
 $\rho$ is the covariance of the two variables normalised by the product of their standard deviations. 
%
%
%The Pearson correlation coefficient  or simply the `correlation coefficient' ($\rho$)  \citep{galton1877typical,galton1877typicala,pearson1895note}, concordance correlation coefficient  ($\rho_c$) \citep{lin1989concordance} are arguably the most popular performance measures when it comes to the regression tasks and the ordinal classifications. 
%
%
%
While there exist multiple ways to interpret $\rho$
%a correlation coefficient ($\rho$) 
\citep{lee1988thirteen,taylor1990interpretation,weida1927various,rider1930survey,szekely2007measuring},  $\rho$ essentially represents extent of the linear relationship between two variables. 
%$\rho$ is the covariance of the two variables divided by the product of their standard deviations. 
%
%Thus, given a bivariate population $X:=(x_i)_1^N$ and $Y:=(y_i)_1^N$, 
\begin{align}
%\text{Pearson Correlation Coefficient, }
&\therefore 
%\text{For $X:=(x_i)_1^N$ and $Y:=(y_i)_1^N$, } 
\rho
%&
=
%\displaystyle{
{\frac {\operatorname {cov} (X,Y)}{\sigma _{X}\sigma _{Y}}}
%}
=
%\displaystyle{
{\frac { \sigma_{XY}}{\sigma _{X}\sigma _{Y}}}
%}
%,
%\\  
%&
=\frac{
	%	\displaystyle{
	\sum _{i=1}^{n}(x_{i}-{\mu_{X}})(y_{i}-{\mu_{Y}})
	%	}
}{
	%	\displaystyle{
	\sqrt {\sum _{i=1}^{n}(x_{i}-{\mu_{X}})^{2}}
	%	}
	{\sqrt {\sum _{i=1}^{n}(y_{i}-{\mu_{Y}})^{2}}}
},
%\nonumber%%
\label{eqnrho}
\\
\nonumber%%
%\label{eqndefX}
\text{where: } 
%\quad
&
\sigma_{X} =\sqrt{\frac {\sum_{i=1}^{N}(x_{i}-\mu_{X})^{2}}{N}},
%\\
%\nonumber
%\label{eqndefY}
%\quad
\quad
\sigma_{Y} =\sqrt{\frac {\sum_{i=1}^{N}(y_{i}-\mu_{Y})^{2}}{N}}, \\
&\nonumber%%
%\text{and }
%\quad
\mu_{X}
%&
=\frac{\sum_{i=1}^{N}x_i}{N},
\quad
\mu_{Y}
%&
=\frac{\sum_{i=1}^{N}y_i}{N}, 
\quad
\operatorname{cov(X,Y)}=\sigma_{XY}
%&
=\frac{\sum _{i=1}^{n}(x_{i}-{\mu_{X}})(y_{i}-{\mu_{Y}})}{N}.
\end{align}

%Covariance is a measure of the strength of the joint variability between two variables. When the greater values of one are often associated with the lesser values of anoother, the covariance is negative. When they grow together, the covariance is positive. The magnitude of the covariance is harder to interpret, as it depends largely on how deviated the magnitudes of the two variables are from their respective mean values. It is thus normalised by the standard deviations of both the variables to effectively yield $\rho$. To compute the covariance in the numerator in \Cref{eqnrho}, $X$ and $Y$ are first centred around zero by subtracting the mean of each of the variables separately, before the sum of products of the centred variables is obtained. The scales of the variables, too, are then normalised by the  denominator. $\rho$ is, therefore, a centred and standardised sum of inner-product of the two variables. The magnitude of the denominator can only be greater than or equal to that of the numerator owing to the Cauchy-Schwarz inequality. Thus, $\rho \in [-1,1]$ \citep{lee1988thirteen}. 

% \newpage
While $\rho$ signifies a linear relationship, the $\rho$ measure fails to quantitatively distinguish between a linear relationship and an identity relationship. The $\rho$ measure also fails to quantitatively distinguish between the linear relationship with a constant offset, the one without any offset, and an identity relationship.
In summary, it fails to capture any departure from the 45$^\circ$ (slope = 1) line, \ie any shifts in the scale (slope) and the location (offset). 
Thus, while successful in capturing the precision of the linear relationship, the $\rho$ measure completely misses out on the accuracy. 
%\begin{tikzpicture}[remember picture,overlay]
%\draw[line width=1,black]
%([xshift=40pt,yshift=3.5ex]current page text area.west|-{pic cs:begrhoc})
%rectangle
%([xshift=-24pt,yshift=-3ex]current page text area.east|-{pic cs:endrhoc});
%\end{tikzpicture}
The concordance correlation coefficient  ($CCC$ or $\rho_c$) \citep{lin1989concordance} metric goes a step further, and penalises any deviation from the identity relationship, \ie the non-unity scaling and the non-zero bias. 
%In this paper, we focus only on the bivariate population featuring the continuous regression values, or the ordinal classes. 
%Lin in 1989 proposed $CCC$ or Lin's coefficient ($\rho_c$), which 
$\rho_c$ is a product of $\rho$ with the term $C_b$ that penalises such deviations in the scale and the location \citep{lin1989concordance}.  The $C_b$ component captures the accuracy, while the $\rho$ component represents the precision. Formally,
% 
% $C_b$ is defined as
\begin{align}%{3}
%\label{eqncb}
\nonumber%%
C_b &= 
%\quad
%\displaystyle{
\frac{2}{\left(v + 
	%		\displaystyle{
	\frac{1}{v}
	%		} 
	+ u^2\right)}
%}
, 
%\quad\quad
\hspace{45pt}
\text{where: }&& v=\frac{\sigma_X}{\sigma_Y}=\text{scale-penalty, }
u=\frac{(\mu_X-\mu_Y)}{\sqrt{\sigma_X\sigma_Y}} = \text{shift-penalty.}
%\\
%\nonumber
%\text{where: } v=\frac{\sigma_X}{\sigma_Y}, u&=\frac{(\mu_X-\mu_Y)}{\sqrt{\sigma_X\sigma_Y}} \text{ are the scale- and the location-shift penalties respectively.}
%\\
%\label{eqnsp}
%\text{where: } v &= \quad \displaystyle{\frac{\sigma_X}{\sigma_Y}} \quad&&= \text{the scale shift,}\\
%\label{eqnlp}
%\text{and } u &= \quad \displaystyle{\frac{(\mu_X-\mu_Y)}{\sqrt{\sigma_X\sigma_Y}}} \quad&&=\text{the location shift relative to the scale.}
%\end{align}
\\
%
%Substituting \Cref{eqnsp} and \Cref{eqnlp} in \Cref{eqncb}, we get  
%\begin{align}%{3}
\nonumber%%
\implies C_b  &= 
%\quad
%\displaystyle{
\frac{2}{\left(
	%		\displaystyle{
	\frac{\sigma_X}{\sigma_Y}
	%		} 
	+ 
	%		\displaystyle{
	\frac{\sigma_Y}{\sigma_X}
	%		} 
	+ \left(
	%	\displaystyle{
	\frac{\mu_X-\mu_Y}{\sqrt{\sigma_X\sigma_Y}}
	%	}
	\right)^2\right)}
%}
%, 
%\\
%&
&&\phantom{v}=
%\quad
\displaystyle{
	\frac{2\sigma_X\sigma_Y}{\sigma_X^2 + \sigma_Y^2 + (\mu_X-\mu_Y)^2}
}
.
%\end{alignat}
%\\\nonumber\noindent\rlap{
%	$\rho_c$ is, thus, given by the following:
%}\\
\\
%\begin{alignat}{3}
%\nonumber
\tikzmark{begrhoc}
\therefore \quad   \rho_c &:= \rho C_b 
%    \\
%    \nonumber
%    &
%    = \quad \displaystyle{\frac{2\rho}{(v + \displaystyle{\frac{1}{v}} + u^2)}} 
%    \\
%    \nonumber
%    &
%    = \quad \frac{2\rho}{\displaystyle{\frac{\sigma_X}{\sigma_Y} + \frac{\sigma_Y}{\sigma_X} + \frac{(\mu_X-\mu_Y)^2}{\sigma_X\sigma_Y}} },
%\\
%\nonumber
%            \label{eqnrhocrho}
%&
=  \frac{2\rho\sigma_X\sigma_Y}{\sigma_X^2 + \sigma_Y^2 + (\mu_X-\mu_Y)^2}
%,
%\\
%&
&&\phantom{v}=
%\quad 
\frac{2\sigma_{XY}}{\sigma_X^2 + \sigma_Y^2 + (\mu_X-\mu_Y)^2}.
\label{eqnrhoc}
\tikzmark{endrhoc}
\end{align}

%\newpage
The concordance correlation coefficient ($\rho_c$) has the following characteristics:
% \begin{alignat*}{5}
%     -1 \leq\quad -\left|\rho\right| \leq\quad \rho_c       &\leq\quad &&\left|\rho\right| \leq\quad 1.&&\\
%     \rho_c  &= \quad &&0 \quad \quad&&\text{ if and only if } \rho = 0.\quad&&\\
%     \rho_c  &= \quad &&\rho \quad \text{ if and only if } \sigma_1 = \sigma_2 \text{ and } \mu_1=\mu_2.\\
%     \rho_c  &= \quad &&\pm 1 \quad \text{ if and only if }
%     \\
%             & \text{a.}&&(\mu_X-\mu_Y)^2+ (\sigma_X-\sigma_Y)^2+2\sigma_X\sigma_Y(1\mp\rho)=0,\\
%             & &&\rlap{Or equivalently,}\\
%             &  \text{b.}&&\rho=\pm1, \sigma_X=\sigma_Y, \text{ and } \mu_X=\mu_Y,\\
%             & &&\rlap{Or equivalently,}\\
%             &  \text{c.}&&x_i \text{ and }y_i \text{ are in perfect ($\rho_c=1$) agreement, or}\\
%             & &&x_i \text{ and }y_i \text{ are in perfect reverse &&($\rho_c=-1$) agreement.}
% \end{alignat*}
% 
% 
\begin{align}%{3}
\nonumber%%
-&1 \quad&&\leq\quad -\left|\rho\right| &&\leq\quad  \rho_c \leq\quad \left|\rho\right| \leq\quad 1, \quad\quad\quad sign(\rho_c)=sign(\rho).\\
\nonumber%%
&\rho_c  &&= \quad 0 &&\text{ if and only if: } \quad
%\quad\quad
\rho = 0, 
\text{ \ie} 
\sigma_{XY}=0.\\
\nonumber%%
&\rho_c  &&= \quad \rho \quad &&\text{ if and only if: } \quad
%\quad\quad
\sigma_X = \sigma_Y \text{ and } \mu_X=\mu_Y.\\
\nonumber%%
&\rho_c  &&=\quad 1 \quad &&\text{ if and only if: } \quad \rho=1, \quad \sigma_X = \sigma_Y \text{ and } \mu_X=\mu_Y,\\
\nonumber%%
&\quad&&\quad\quad\text{\ie}&&\text{ if and only if: } \quad x_i=\phantom{-}y_i\quad\forall i : i\in[1,N], i \in \mathbb{N}.\\
\nonumber%%
\text{Likewise, }\quad &\rho_c  &&=\quad -1 \quad &&\text{ if and only if: } \quad x_i=-y_i\quad\forall i : i\in[1,N], i \in \mathbb{N}.
%&\rho_c  &&=\quad -1 \quad &&\text{ if and only if: } \quad \rho=-1, \quad \sigma_X = \sigma_Y \text{ and } \mu_X=\mu_Y,\\
%&\quad&&\quad\quad\text{\ie}&&\text{ if and only if: } \quad x_i=-y_i\quad\forall i : i\in[1,N], i \in \mathbb{N}.
%&\rho_c  &&= \quad \pm 1 \quad &&\text{ if and only if: }
%%    \\
%% & \text{a.}
%%            &&&&&
%\quad
%%\quad
%%\quad
%(\mu_X-\mu_Y)^2+ (\sigma_X-\sigma_Y)^2+2\sigma_X\sigma_Y(1\mp\rho)=0,\\
%% & &&\rlap{Or equivalently,}\\
%\quad&&&\quad\quad\text{\ie}&&\text{ if and only if: }
%%            \\
%% &  \text{b.}
%%            &&&&&
%%\quad
%%\quad
%\quad\rho=\pm1, \sigma_X=\sigma_Y, \text{ and } \mu_X=\mu_Y,\\
%% & &&\rlap{Or equivalently,}\\
%% & &&\rlap{implying, if and only if: }\\
%\quad&&&\quad\quad\text{\ie}&&\text{ if and only if: }
%%            \\
%% &  \text{c.}
%%            &&&&&
%%\quad
%\quad x_i \text{ and }y_i \text{ are in perfect ($\rho_c=1$) agreement, or}
%\\
%&&&&&
%%\quad\quad\quad 
%\quad\quad\quad\quad 
%\quad
%\quad
%\quad 
%x_i \text{ and }y_i \text{ are in perfect reverse ($\rho_c=-1$) agreement.}
\end{align}

\subsection{Literature survey: Extensions, generalisations and criticism of $\rho_c$}
%Since Lin's pioneering work, numerous articles have been published advancing the field. 
The $\rho_c$-measure is based on the distance metric, that is the expected value of the squared difference between the two measurements $X$ and $Y$.  Some have extended the applicability of $\rho_c$ to more than two measurements by proposing new reliability coefficients, \eg the overall concordance correlation coefficient \citep{carrasco2003estimating,barnhart2001modeling,barnhart2002overall}.
Likewise, a more generalised version in terms of the distance function used has also been proposed \citep{king2001robust,king2001generalized}, establishing its similarities with the kappa and weighted kappa coefficients. 
Alternative estimators for evaluating agreement and reproducibility based on the $\rho_c$ have also been proposed \citep{quan1996assessing,st1998evaluating}. Comparing the $\rho_c$ against the previously existing  four intraclass correlation coefficients presented in \citep{shrout1979intraclass,mcgraw1996forming}, Nickerson presents a strong critique of the contributions of the $\rho_c$-metric in evaluating reproducibility \citep{nickerson1997note}. The usability and apparent paradoxes associated with the reliability coefficients have been thoroughly and vehemently debated upon \citep{feinstein1990high,zhao2013assumptions,krippendorff2013commentary}. However, $\rho_c$ remains arguably one of the most popular reproducibility indices, used in a wide range of fields \citep{nishizuka2003proteomic,murtaza2013non,lange1999plurality,ma2013magnetic,lombard2002content,conroy2003estimation}.
%, \eg cancer detection and treatment \citep{nishizuka2003proteomic,murtaza2013non},
% comparison of the analysis method for fMRI brain scans \citep{lange1999plurality},  Magnetic resonance fingerprinting \citep{ma2013magnetic}, intercoder reliability in mass communications \citep{lombard2002content}, estimation of heart-related health risk \citep{conroy2003estimation}.
% \citep{davies1982measuring}
% 1982 Measuring Agreement for Multinomial Data
% Mark Davies and Joseph L. Fleiss
% Not useful, classes, kappa
% 
% \citep{fleiss1973equivalence}
% \citep{echavarria2014using}
% \citep{fay2005random}

\subsection{Literature survey: Growing popularity of $\rho_c$ 
%	(as a performance metric, and as a loss function)
}
The popularity of the measure has encouraged researchers to publish macros and software packages likewise \citep{carrasco2013estimation,crawford2007computer}. When it comes to  instance-based ordinal classification, regression, or a sequence prediction task, the machine learning community likewise has begun adapting $\rho_c$ as the performance measure of choice \citep{Trigeorgis16-AFE,Pandit18-TAA,Pandit18-HGI,Schmitt17-OIT}. Take the case of the `Audio/Visual Emotion Challenges' (AVEC) for example. The shift is noticeable, with early challenges using $RMSE$ as the winning criteria, to now $\rho_c$ in those recently held \citep{Ringeval15-POT,Valstar16-POT,Ringeval17-POT,Ringeval18-A2W}. 
Almost without exception, the winners of these challenges have used deep learning models -- which are trained to model the input to output (the raw data or features to prediction) mapping through minimisation of a loss function \citep{bennett2006interplay}. A loss function nominally captures the difference between a prediction from a model and the desired output; its job, consequently, is to encourage a model to drive the prediction of the model close to the desired value.  While the shift in the community to use the $\rho_c$ measure as a performance metric is definitely underway, no attempts have been made to design a loss function specifically tailored to boost $\rho_c$, barring a few lone exceptions \citep{Weninger16-DTR,Pandit19-IKH}. A few recent studies highlight the deterioration of performance through use of inconsistent loss functions (\ie different from the performance metric), and advocate use of a consistent loss function \citep{Pandit19-IKH,atmaja2020evaluation,Trigeorgis16-AFE}. Yet, none 
%-- to author's best knowledge -- 
provides a mathematically rigorous \emph{reasoning} for this often-witnessed phenomenon.

The loss function used in \citealp{Weninger16-DTR,Trigeorgis16-AFE} is directly the $\rho_c$, which is computationally expensive to use at every training step. This is because, the computation of $\rho_c$ necessitates computation of standard deviations of the gold standard and the prediction, covariance between the gold standard and the prediction, the difference between the mean values at every iteration, and latter operations such as squaring, summing, and the division. Also, with $\rho_c$ as the loss function, the partial derivative of $\rho_c$ with respect to the outputs needs to be recalculated as well, to propagate the error down to the input layers using the backpropagation algorithm in neural networks at every step in the training iteration. In this paper, we therefore identify and isolate workable lightweight functions which directly have an impact on the $\rho_c$ metric. 
% 
% The effort was inspired by the discovery that the predictions with identical mean square error can result in different values of CCC, when compared against the very same gold standard.  
% 
%We achieve this by dividing 
%To this end, we divide
%the CCC into its constituent components through 
We achieve this by reformulating the $\rho_c$ in terms of individual prediction errors.
Recognising the terms that are affected by the error or the prediction `sequence'/ordering alone, we propose a family of candidate loss functions.

\section{Main contributions and organisation of the paper}
%We present next the overall organisation of the paper.
The key contribution of this paper is that it 
invalidates
%disproves
%\emph{disproves} 
the common notion:  \emph{$MSE$ reduction leads to $\rho_c$ improvement}, 
%with equations. 
%We first derive 
by deriving the most crucial, yet the missing many-to-many mapping existing between $\rho_c$ and $MSE$ in \Cref{sec_m2mg}. In the section next, \ie in \Cref{sec_m2mm}, we determine the conditions for $\rho_c$ optimisation (\ie for $\rho_{c_{min}}$ and $\rho_{c_{max}}$), 
%restricting our attention to 
for a fixed value of $MSE$, and derive the equations for both
$\rho_{c_{min}}$ and $\rho_{c_{max}}$
%the minimum and maximum $\rho_c$ 
as a function of $MSE$. 
%Generalising even further, using these derived formulations and through $MSE$ optimisation for a given $L_p$ norm, we present $\rho_c$ optimisation conditions and the corresppnding formulations in \Cref{sec_lp} (for $p\geq 2$) and in \Cref{sec_mae} (for $p=1$).
Using these derived equations, we find the formulations for $\rho_c$ corresponding to the optimised $MSE$  
%for a given $MAE$ in \Cref{sec_mae}, and 
for any given $L_p$ norm in \Cref{sec_lp} (for any $p\geq 2$). 
Upon establishing the fact that efforts for $MSE$ minimisation do not necessarily yield a superior prediction performance in terms of $\rho_c$ in \Cref{sec_truespan}, we generalise \Cref{sec_m2mm} to any given $L_p$ norm value, for any $p$ that is an even natural number in \Cref{sec_m2mLp}.
We then optimise $\rho_c$ for a special case of a given $MSE$; \ie not only a fixed $MSE$ or $L_p$ norm, but also a fixed set of error values
%. To this end, we formally define the problem we attempt to tackle 
in \Cref{sec_prob}. 
%In \Cref{sec_form1} and \Cref{sec_form2}, we rework the $\rho_c$ formulation through two slightly different substitutions in terms of the prediction sequences to arrive at mutually contradictory requirements in terms of the redistribution of error values interestingly. The contradictory requirements are, however, only an apparent paradox; as a careful investigation leads us to a consistent set of requirements in terms of the prediction versus gold standard samples. We present these interesting insights and paradoxes in \Cref{sec_paradox}. 
%Next, we establish conditions that help us  conclusively determine which of the two reformulations yields us a better $\rho_c$ in \Cref{sec_app4_chev}. 
We supplement our findings with illustrations in \Cref{sec_sewa}. Learning from these insights, we present a family of candidate loss functions in \Cref{sec_cost}. In \Cref{sec_concl}, we summarise our findings and present possible future research directions.
%our motivation, the resulting findings and our planned future work as we conclude.

%\newpage

\section{Many-to-many mapping between $MSE$ and $\rho_c$ as a general case}
\label{sec_m2mg}

\begin{theorem}
	For a bivariate population $X:=(x_i)_1^N$, $Y:=(y_i)_1^N$,
	${\rho_c}
	=\left(1+\frac{MSE}{2\sigma_{XY}}\right)^{-1}$.
\end{theorem}
\begin{proof}
	% Consider the Summation
	%Thus, 
	%given a bivariate population $X:=(x_i)_1^N$ and $Y:=(y_i)_1^N$, 
%	From \cite{lin2000total,lin2002statistical}, 
%	From \Cref{eqnrho}
%	we have:
	\begin{align}
	%	\sigma_{X}^2+\sigma_{Y}^2-&2\sigma_{XY}
	%	=
	%	\frac {\sum_{i=1}^{N}
	%	\left[
	%	{(x_{i}-\mu_X)^2
	%		+(y_{i}-\mu_Y)^2
	%		-2(x_{i}-\mu_X)(y_{i}-\mu_Y)}\right]}{N},\\
	%	&=
	%	\frac {\sum _{i=1}^{N}
	%	{(x_{i}-\mu_X-y_{i}+\mu_Y)^2}
	%	}{N}
	%	=
	%	\frac{\sum _{i=1}^{N}
	%	{\left[(x_{i}-y_{i})-(\mu_X-\mu_Y)\right]^2}
	%	}{N},
	%\\
	%	&=
	%	\frac{	\sum _{i=1}^{N}
	%	\left[
	%	{(x_{i}-y_{i})^2
	%		+(\mu_X-\mu_Y)^2
	%		-2(x_{i}-y_{i})(\mu_X-\mu_Y)}\right]}{N},\\
	%&=MSE+
	%\frac {	\sum _{i=1}^{N}\left[
	%	{(\mu_X-\mu_Y)^2
	%		-2(x_{i}-y_{i})(\mu_X-\mu_Y)}\right]}{N} 
	%		\\
	%	&=MSE+
	%	(\mu_X-\mu_Y)
	%\frac {\sum _{i=1}^{N}\left[
	%	{(\mu_X-\mu_Y)
	%		-2(x_{i}-y_{i})}\right]}{N},
	%\\
	%	&=MSE-
	%	(\mu_X-\mu_Y)^2
	%\implies
	%\Aboxed{
	%\sigma_{X}^2\hspace{-2pt}+\hspace{-2pt}\sigma_{Y}^2\hspace{-2pt}+\hspace{-2pt}(\mu_X\hspace{-2pt}-\hspace{-2pt}\mu_Y)^2
	%=
	%	MSE\hspace{-2pt}+\hspace{-2pt}
	%2\sigma_{XY}.
	%}
%	\text{We have, }\quad 
	\sigma_{X}^2+\sigma_{Y}^2+(\mu_X&-\mu_Y)^2
	=
	MSE\hspace{-2pt}+\hspace{-2pt}
	2\sigma_{XY} 
	\quad \quad \quad \because \text{ \Cref{eqnrho}}.
\nonumber%%
	\\
	\nonumber%%
	\text{\cite{lin2000total,lin2002statistical} also corrobarate to the equation presented above.}\span\span
	\\
\nonumber%%
%	\because
	\text{Also, }
	\quad
	\rho_c
	\quad
	&=
	\frac{2\sigma_{XY}}{\sigma_{X}^2+\sigma_{Y}^2+(\mu_X-\mu_Y)^2} 
%	\label{eqnccc2}
	%	\raisetag{20pt}
	\quad\quad
	\text{ $\because$  \Cref{eqnrhoc}}.\\
	\therefore
	\quad
	\rho_c
	\quad
	&=
	\quad
	%	&&
	\frac{2\sigma_{XY}}{MSE+2\sigma_{XY}}
	%	\\
	\quad\quad\quad\quad
	=
	\quad
	%	&=
	%%	&&
	\frac{1}{1+\frac{MSE}{2\sigma_{XY}}},
	%	\\
	%		\therefore
	%		\quad
	%\quad
	\nonumber%%
	\\
	\text{ \ie}
	%\quad
	%\therefore
	\quad
%	\Aboxed{
		\mathbf{\rho_c}
		\quad
		&=
		\quad
		\left(1-\frac{MSE}{MSE+2\sigma_{XY}}\right)
		\quad
		=
		\quad
		\left(1+\frac{MSE}{2\sigma_{XY}}\right)^{-1}.
		\label{eqn_cccmsemap}
%	}.
	\end{align}
\end{proof}
\vspace{-0.75cm}
%\newpage
%\section{Range of $\rho_c$, Given MSE}
%or $L_2$-norm of the error-set 
%\section{Why $MSE$ fails as a loss function to optimise $\rho_c$}
%\section{Why using $MSE$ as a loss function is a misdirected effort for optimising $\rho_c$}
\section{Why $MSE$ as a loss function fails to improve $\rho_c$}
\label{sec_m2mm}
While the minimisation of $MSE$ loss function and the maximisation of $\rho_c$ performance metric are both directed at achieving the perfect identity relationship between the labels (\ie the gold standard) and the predictions, efforts for the minimisation of $MSE$ do not necessarily translate into the maximisation of $\rho_c$, and vice versa. In this section, we reason and prove this fact mathematically by deriving the conditions and the formulations for $\rho_c$ optimisation at a given $MSE$ (\ie given the error-set $L_2$-norm). 
To this end, we find the conditions and formulations for minimum and maximum possible values of $\rho_c$ at a given $MSE$, by making use of the many-to-many mapping between $\rho_c$ and $MSE$ we have derived in \Cref{eqn_cccmsemap}.
%\subsection{$\rho_c$ optimisation, given $MSE$ (\ie the error-set $L_2$-norm)}
\subsection{$\rho_c$ optimisation, given the error-set $L_2$-norm or the $MSE$}
% \begin{alignat}{2}
%\label{sec_m2mm}
Inspired by the discovery that the predictions with identical $MSE$ can map to different $\rho_c$ values, the maximum and minimum $\rho_c$ for a constant $MSE$ are found next. The problem statement is, thus:
%
%\end{alignat}
%We, thus, formulate our problem as follows. 

\textit{
	Given (1) a gold standard time series, $G:=(g_i)_1^N$, and (2) a fixed $MSE$ value, find 
	%	the distribution of $MSE$ into 
	the set(/s) of error values $E:=(e_i)_1^N$ that achieve maximisation and minimsaition of $\rho_c$.
}

\begin{theorem}
	\label{thm41}
	For a given $MSE$, $\rho_c$ is maximised when the errors amounting to $MSE$ are distributed in the same ratio as of the corresponding deviations of gold standard around the mean gold standard.
	% $\rho_c$ maximisation is achieved when constituent errors ($E$) making $MSE$ are divided such that the errors have similar distribution as of $G$; that is, when they are divided in the same ratio as of deviations of $G$ around the mean of $G$. 
	%have the same distribution in the same ratio as of the distribution of $G$.
	%	
	That is,
	\begin{alignat}{2}
	\label{eqn_rhocmax} 
	%%%%%%%%%%%
	%	\Aboxed{
	%	\rho_{{c}_{max}}
	%	=\frac{2\left(1+{\sqrt{\frac{N\cdot MSE}{\sum_{j=1}^{N}{(g_{i}-\mu_G)}^2}}}\right)}{1+\left(1+{\sqrt{\frac{N\cdot MSE}{\sum_{j=1}^{N}{(g_{i}-\mu_G)}^2}}}\right)^2},
	%	\hspace{3pt}
	%%	\\
	%	\text{ when }
	%	\hspace{3pt} 
	%	\begin{array}{ll}
	%	e_i=\left|\sqrt{\frac{N\cdot MSE}{\sum_{j=1}^{N}{(g_{i}-\mu_G)}^2}}\right|\cdot{(g_{i}-\mu_G)},
	%	\\
	%%	\nonumber
	%		\quad \quad \quad \quad 
	%\quad 
	%\forall i : i\in[1,N], i \in \mathbb{N}.\end{array}
	%}
	%%%%%%%%%%%
%	\Aboxed{
		\rho_{{c}_{max}}
		=\frac{2\left(1+{\sqrt{\frac{MSE}{{\sigma_G}^2}}}\right)}{1+\left(1+{\sqrt{\frac{MSE}{{\sigma_G}^2}}}\right)^2},
		\hspace{3pt}
		%	\\
		\text{ when }
		\hspace{3pt} 
		\begin{array}{ll}
		e_i=\left|\sqrt{\frac{ MSE}{{\sigma_G}^2}}\right|\cdot{(g_{i}-\mu_G)},
		\\
		%	\nonumber
		\quad \quad \quad \quad 
		\quad 
		\forall i : i\in[1,N], i \in \mathbb{N}.\end{array}
%	}
	\end{alignat}
\end{theorem}
\begin{proof}
	Let the prediction and the gold standard sequence be 
	% represented by the two time-series 
	$X:=(x_i)_1^N$ and $Y:=(y_i)_1^N$, not necessarily in that order. 
	%	As the formula for $\rho_c$ is symmetric with respect to $X$ and $Y$, which variable represents what sequence does not matter
	%%	, so far as 
	%	for $\rho_c$ computation.
	%	 is concerned. 
	Note that, as the formula for $\rho_c$ is symmetric with respect to $X$ and $Y$. As a result, note that which variable represents what sequence does not matter, so far as $\rho_c$ computation is concerned. 
	%	\begin{align}
	%	\label{eqnmeasures}
	%	\therefore\operatorname {MSE} &:={\frac {1}{N}}\sum _{i=1}^{N}d_i^{2},
	%	\quad
	%	% \displaystyle 
	%	\operatorname {RMSE} :=\sqrt[]{{\frac {1}{N}}\sum _{i=1}^{N}d_i^{2}}=\sqrt[]{\operatorname {MSE}},
	%	\quad
	%	% \displaystyle
	%	\operatorname {MAE} :={\frac {1}{N}}\sum _{i=1}^{N}\left|d_{i}\right|.
	%	\end{align}
	%%%%%	
	%	Note that \Cref{eqnmeasures} is valid also when $d_i:=y_i-x_i$.
	%%%%%	
	%	Because
%	$\because$ 
Because
	$
	\rho_c
	=\left(1+\frac{MSE}{2\sigma_{XY}}\right)^{-1}\nonumber
	$, 
	$\rho_c$ optimisation at a given $MSE$ necessitates $\sigma_{XY}$ optimisation.
		 
	\begin{align}
	%\label{eqnccc2}
	\setlength{\jot}{10pt}
	%	\label{eqnDefD}
	\text{Let} \quad d_i&:=x_i-y_i  %\\ 
		\quad 
	\text{ and }  
		\quad  
	\mu_{D}:=\frac{1}{N}\sum_{i=1}^{N}d_i
	\qquad
	\implies
	\qquad
	\mu_{D} =\mu_X-\mu_Y, 
	\quad
	{MSE}:={\frac {1}{N}}\sum _{i=1}^{N}d_i^{2}.  
	%	\\
	%	\therefore \quad \mu_{D} &=\mu_X-\mu_Y, 
	%	&&\quad\because \text{\Cref{eqndefX,eqndefY,eqnDefD}}
	\label{yzdef}\\
%	\end{align}
	%%%%%
	%	maximisation and minimisation of $\rho_c$ thus effectively translates to maximisation and minimisation of $\sigma_{XY}$ respectively.
	%
%	\begin{align}
\nonumber%%
%	\label{Ncovdef0}
	\therefore
	N\sigma_{XY}&=\sum_{i=1}^{N}(x_i-\mu_X)\cdot(y_i-\mu_Y)
	%	\\
	%	&
	=\sum_{i=1}^{N}(y_i+d_i-\mu_{Y}-\mu_{D})\cdot(y_i-\mu_Y),\\
	&=\sum_{i=1}^{N}(y_i-\mu_Y)^2
	+\sum_{i=1}^{N}d_i\cdot(y_i-\mu_Y)
	-\sum_{i=1}^{N}\mu_D\cdot(y_i-\mu_Y).
\nonumber%%
%	\label{Ncovdef}
	\\
	%\end{alignat}
	%\begin{alignat}{2}
	\text{If } y_i-\mu_Y&:=y_{z_i} 
	%	\\
	\implies
	\sum_{i=1}^{N} \mu_D\cdot
	y_{z_i} 
	%	&
	= 0 
	%\\
	%\nonumber&
	%	\quad
	%	\quad
	\left(\because 
	%\mu_D\cdot\sum_{i=1}^{N} y_{z_i}=\mu_D\cdot 0
	%	\mu_D\cdot
	\sum_{i=1}^{N} (y_i-\mu_Y)=
	%	\mu_D\cdot
	%	\left[
	%	\sum_{i=1}^{N}y_i-\sum_{i=1}^{N}\mu_Y
	%	\right]
	N\mu_Y-N\mu_Y
	=0 
	\right).
\nonumber%%
%	\label{eqnyzi}
	%\\
	%	\end{align}
	%	\begin{align}
	%	\therefore N\cdot\sigma_{XY}
	%	&=
	%	\sum_{i=1}^{N}{y_{z_i}}^2
	%	+\sum_{i=1}^{N}d_i\cdot{y_{z_i}}
	%	-\sum_{i=1}^{N}\mu_D\cdot{y_{z_i}} \text{ (from \Cref{Ncovdef})}\\
	%	\therefore 
	%	N\sigma_{XY}
	%	&=
	%	\sum_{i=1}^{N}{y_{z_i}}^2
	%	+\sum_{i=1}^{N}d_i\cdot{y_{z_i}}
	%	-\frac{1}{N}\left(\sum_{i=1}^{N}{y_{z_i}}\right)\cdot\left(\sum_{j=1}^{N}d_j\right)
	\\
%	\Aboxed{
		\therefore
		N\cdot\sigma_{XY}
		&=
		\sum_{i=1}^{N}{y_{z_i}}^2
		+\sum_{i=1}^{N}d_i{y_{z_i}}.
%	}
	%	\quad 
	%	\left(\because \sum_{i=1}^{N}{y_{z_i}}=0 \text{ from \Cref{yzdef} }\right)
	%-\frac{1}{N}\sum_{i=1}^{N}\left(\left(\sum_{j=1}^{N}d_j\right){y_{z_i}}\right)
	\label{eqncovmax}
	\end{align}
	Thus, 
	%	to maximise $\rho_c$, 
	we need to maximise $N\cdot\sigma_{XY}$ as given by \Cref{eqncovmax} by tuning 
	%	the error sequence 
	$D:=(d_i)_1^N$, while satisfying the condition $\sum_{i=1}^{N}d_i^2-N\cdot MSE=0$ (
	$\because$ \Cref{yzdef}). That is, 
	%	
	%	\noindent Formally speaking, we need to 
	\begin{alignat}{2}
	\nonumber%%
	\text{maximise: }\quad &f(d_1,d_2,\cdots,d_N)&&= N\cdot\sigma_{XY}
	=
	\sum_{i=1}^{N}{y_{z_i}}^2
	+\sum_{i=1}^{N}d_i{y_{z_i}},
	%	-\frac{1}{N}\sum_{i=1}^{N}\left(\left(\sum_{j=1}^{N}d_j\right){y_{z_i}}\right)
	%	\nonumber
	\\
	\text{subject to: }\quad &g(d_1,d_2,\cdots, d_N)&&= \sum_{i=1}^{N}d_i^2-N\cdot MSE=0.\nonumber
	\end{alignat}
	Auxiliary Lagrange function	${\mathcal {L}}(d_1,d_2,\cdots,d_N,\lambda )=f(d_1,d_2,\cdots,d_N)-\lambda \cdot g(d_1,d_2,\cdots,d_N)$ is given by:
	%
%	Auxiliary Lagrange expression is given by:
	%
	%Applying ordinary Lagrange multiplier method,
	%we introduce auxiliary Lagrange expression, defined by
	\begin{align}%{2}
%	{\mathcal {L}}(d_1,d_2,\cdots,d_N,\lambda )&=f(d_1,d_2,\cdots,d_N)-\lambda \cdot g(d_1,d_2,\cdots,d_N),\\
	% {\mathcal {L}}(d_1,d_2,\cdots,d_N,\lambda )
%	\mathcal {L}&=N\sigma_{XY}
%	-\lambda \left(\sum_{i=1}^{N}d_i^2-N\cdot MSE\right),\\
\nonumber%%
	\mathcal {L} &=\sum_{i=1}^{N}{y_{z_i}}^2
	+\sum_{i=1}^{N}d_i{y_{z_i}}
	%-\sum_{i=1}^{N}\mu_D{y_{z_i}}
	%-\frac{1}{N}\sum_{i=1}^{N}\left(\left(\sum_{j=1}^{N}d_j\right){y_{z_i}}\right)
	%\nonumber
	%\\
	%&\quad\quad\quad
	-\lambda \left(\sum_{i=1}^{N}d_i^2-N\cdot MSE\right).\\
	\therefore\nabla_{d_1,d_2,\cdots,d_N,\lambda}
	{\mathcal {L}}
	%(d_1,d_2,\cdots,d_N,\lambda )
	=&0 
	\Leftrightarrow 
	{\begin{cases}
		{y_{z_i}}
		%-\frac{1}{N}\sum_{j=1}^{N}{y_{z_j}}
		-2\lambda d_i=0 \quad \quad \forall  i \in \mathbb{N} : i \in [1,N].
		\\
		%	\quad\quad
		\sum_{i=1}^{N}d_i^2-N\cdot MSE=0.
		\end{cases}}\label{lagrangecondsp}
	%	\\
	%	
	%\because \text{As per \Cref{eqnyzi}, }\nonumber\\
	%\rlap{Substituting 
	%$\frac{1}{N}\sum_{j=1}^{N}{y_{z_j}}=0$
	% from \Cref{eqnyzi} into \Cref{lagrangeconds}}  
	%\nonumber
	%\\
	\end{align}
	\begin{align}%{2}
\nonumber%%
	\therefore 
	{y_{z_i}}-2\lambda d_i=0 
	\quad
	&\forall  i\in \mathbb{N} : i \in [1,N]
	%\quad
	\quad
	%\\
	\text{ and }
	%\quad
	\quad
	\sum_{i=1}^{N}d_i^2-N\cdot MSE=0.
	\\
	%	\therefore 
	%	\sum_{i=1}^{N}d_i^2-N\cdot MSE&=0
	%	\quad\quad
	%	%\\
	%	\text{ and }
	%	\quad\quad
	%	{y_{z_i}}-2\lambda d_i=0 \quad \forall  i\in \mathbb{N} : i \in [1,N]
	%	\\
\nonumber%%
	\therefore d_i=\frac{y_{z_j}}{2\lambda}
	\quad
	&	\forall  i\in \mathbb{N} : i \in [1,N]
	%\\
	\quad
	\implies
	%	\text{ and } 
	\quad
	%	\quad
	\sum_{i=1}^{N}{y_{z_i}}^2 =4\cdot\lambda^2\cdot N\cdot MSE.
	\\
	%	\nonumber
	%	\text{Substituting $2\lambda=\pm\sqrt{\frac{N\cdot MSE}{\sum_{j=1}^{N}{y_{z_i}}^2}}$ in $d_i=\frac{y_{z_j}}{2\lambda}$}
	%	\span
	%	\\
	\span
	\hspace{-30pt}
	\therefore
	d_i=\pm\sqrt{\frac{N\cdot MSE}{\sum_{j=1}^{N}{y_{z_i}}^2}}\cdot{y_{z_j}}
	=\pm\sqrt{\frac{MSE}{\sigma_G^2}}\cdot{y_{z_j}}.
	%\\\nonumber
	%\rlap{ where: $\sigma_G^2$=Mean Squared Deviation of the gold standard $Y:=(y_i)_1^N$}
%\nonumber%%
	\label{MSEDistr}
	\end{align}
	where: ${\sigma_G^2}$=standard deviation of the gold standard $Y$ $:=\frac{1}{N}\sum_{i=1}^{N}{(y_{i}-\mu_Y)}^2=\frac{1}{N}\sum_{i=1}^{N}{y_{z_i}}^2$.
	
	%Clearly, 
	From \Cref{eqncovmax}, 
	$N\sigma_{XY}$ is maximised when 
	$d_i$ and ${y_{z_j}}$ have identical signs. That is,
	\begin{alignat}{2}
\nonumber%%
%	\Aboxed{
		d_i&=\left|\sqrt{\frac{N\cdot MSE}{\sum_{j=1}^{N}{y_{z_i}}^2}}\right|\cdot{y_{z_j}}
		=\left|\sqrt{\frac{MSE}{\sigma_G^2}}\right|\cdot{y_{z_j}}
%	} \hspace{1cm}
%	\because\text{ \Cref{eqncovmax,MSEDistr}}. 
	\end{alignat}
	Thus, $\rho_c$ is maximised when $MSE$ is composed of the errors (\ie \{$d_i$\}) that are equally proportional to the deviations of the gold standard from the mean value (\ie $\{y_{z_i}\}:=\{y_i-\mu_{Y}\}$), and are of the same sign as of that deviations  (\ie signs of $\{y_{z_i}\}$) correspondingly.
	With the understanding that the square-root sign denotes a positive square root, from \Cref{eqncovmax} we have:
	%	we get from \Cref{eqncovmax}:
	\begin{align}%{2}
\nonumber%%
	\sigma_{{XY}_{max}}
	&=\frac{1}{N}\Bigg(\sum_{i=1}^{N}{y_{z_i}}^2\Bigg(1+\sqrt{\frac{MSE}{\sigma_G^2}}\Bigg)\Bigg)
	=\frac{1}{N}\Bigg(N\cdot \sigma_G^2\Bigg(1+\sqrt{\frac{MSE}{\sigma_G^2}}\Bigg)\Bigg),
	%	\text{ ($\because$ \Cref{eqncovmax})}
	\\
\nonumber%%
	%	&={\sqrt{\sigma_G^2+MSE}}\cdot{\sqrt{\sigma_G^2}}
	&=\sigma_G^2+{\sqrt{\sigma_G^2\cdot MSE}}
%	\quad \text{ ($\because$ \Cref{eqncovmax})}.
	\end{align}
	%	We note that we get the same conditions for $\rho_c$ maximisation if $\sigma_{XY}$ in \Cref{Ncovdef0} is expressed in terms of $x_i$, instead of $y_i$.
	%	
	%	Consequently, 
	Thus, 
	%	in either case,
	from \Cref{eqn_cccmsemap}:
	\begin{align}
\nonumber%%
	\rho_{{c}_{max}}
	\quad
	&
	=
	\quad
	\left(1+\frac{MSE}{2\cdot(\sigma_G^2+{\sqrt{\sigma_G^2\cdot MSE}})}\right)^{-1}
	%\\
	%&
	&&
	=
	\quad
	\frac{2\cdot(\sigma_G^2+{\sqrt{\sigma_G^2\cdot MSE}})}{MSE+2\cdot(\sigma_G^2+{\sqrt{\sigma_G^2\cdot MSE}})},
	\\
\nonumber%%
	&
	=
	\quad
	\frac{2\cdot\left(1+{\sqrt{\frac{MSE}{\sigma_G^2}}}\right)}{\frac{MSE}{\sigma_G^2}+2\left(1+{\sqrt{\frac{MSE}{\sigma_G^2}}}\right)}
	&&	
	=
	\quad
	%	\\
%	\hspace{1cm}
%	\implies
%	\span
	%&=\frac{2+2\cdot{\sqrt{\frac{MSE}{\sigma_G^2}}}}{\frac{MSE}{\sigma_G^2}+2+2{\sqrt{\frac{MSE}{\sigma_G^2}}}}\\
%	\Aboxed{
%		\rho_{{c}_{max}}=
		\frac{2\left(1+{\sqrt{\frac{MSE}{\sigma_G^2}}}\right)}{1+\left(1+{\sqrt{\frac{MSE}{\sigma_G^2}}}\right)^2}.
%	}
	\end{align}
\end{proof}
Likewise (cf. \Cref{sec_app0_rhomin}), the condition and formulation for  $\rho_c$ minimisation at a given $MSE$ are presented next. 
%To minimise $\rho_c$, we need to minimise $N\sigma_{XY}$ in \Cref{eqncovmax} by tuning $d_i$ values, subject to the condition $\sum_{i=1}^{N}d_i^2-N\cdot MSE=0$ (
%$\because$ \Cref{eqn_cccmsemap,eqnmeasures}).

%\newpage

\begin{theorem}
		\label{thm42}
	For a given $MSE$, $\rho_c$ is minimised when the errors amounting to $MSE$ are distributed in the same ratio as of the corresponding deviations of gold standard around the mean gold standard, with an opposite sign.		
	%	$\rho_c$ minimisation is achieved when constituent errors ($E$) making $MSE$ are divided such that the errors have distribution exactly opposite of $G$; that is, when they are divided in the same ratio as of deviations of $G$ around the mean of $G$, but with an opposite sign.
	%	have the same distribution in the same ratio as of the distribution of $G$.
	%	
	That is,
	\begin{alignat}{2} 
	\label{eqn_rhocmin}
	%%%%%%
	%	\Aboxed{
	%	\rho_{{c}_{min}}
	%	=\frac{2\left(1-{\sqrt{\frac{N\cdot MSE}{\sum_{j=1}^{N}{(g_{i}-\mu_G)}^2}}}\right)}{1+\left(1-{\sqrt{\frac{N\cdot MSE}{\sum_{j=1}^{N}{(g_{i}-\mu_G)}^2}}}\right)^2},
	%%	\\
	%	\hspace{2pt}
	%	\text{ when }
	%	\hspace{2pt}
	%		\begin{array}{ll}
	%	e_i=-\Big|\sqrt{\frac{N\cdot MSE}{\sum_{j=1}^{N}{(g_{i}-\mu_G)}^2}}\Big|\cdot{(g_{i}-\mu_G)},
	%	\\
	%%	\nonumber
	%		\quad \quad \quad \quad 
	%		\quad 
	%		\forall i : i\in[1,N], i \in \mathbb{N}.\end{array}
	%	}
	%%%%%%%%%%%
%	\Aboxed{
		\rho_{{c}_{max}}
		=\frac{2\left(1-{\sqrt{\frac{MSE}{{\sigma_G}^2}}}\right)}{1+\left(1-{\sqrt{\frac{MSE}{{\sigma_G}^2}}}\right)^2},
		\hspace{3pt}
		%	\\
		\text{ when }
		\hspace{3pt} 
		\begin{array}{ll}
		e_i=-\left|\sqrt{\frac{ MSE}{{\sigma_G}^2}}\right|\cdot{(g_{i}-\mu_G)},
		\\
		%	\nonumber
		\quad \quad \quad \quad 
		\quad 
		\forall i : i\in[1,N], i \in \mathbb{N}.\end{array}
%	}
	\end{alignat}
\end{theorem}

%\vspace{-12pt}\subsection{Remarks}
\begin{remark}
%It is, thus, obvious that 
%%%%%%%%%%% 0 to 6 figure begins
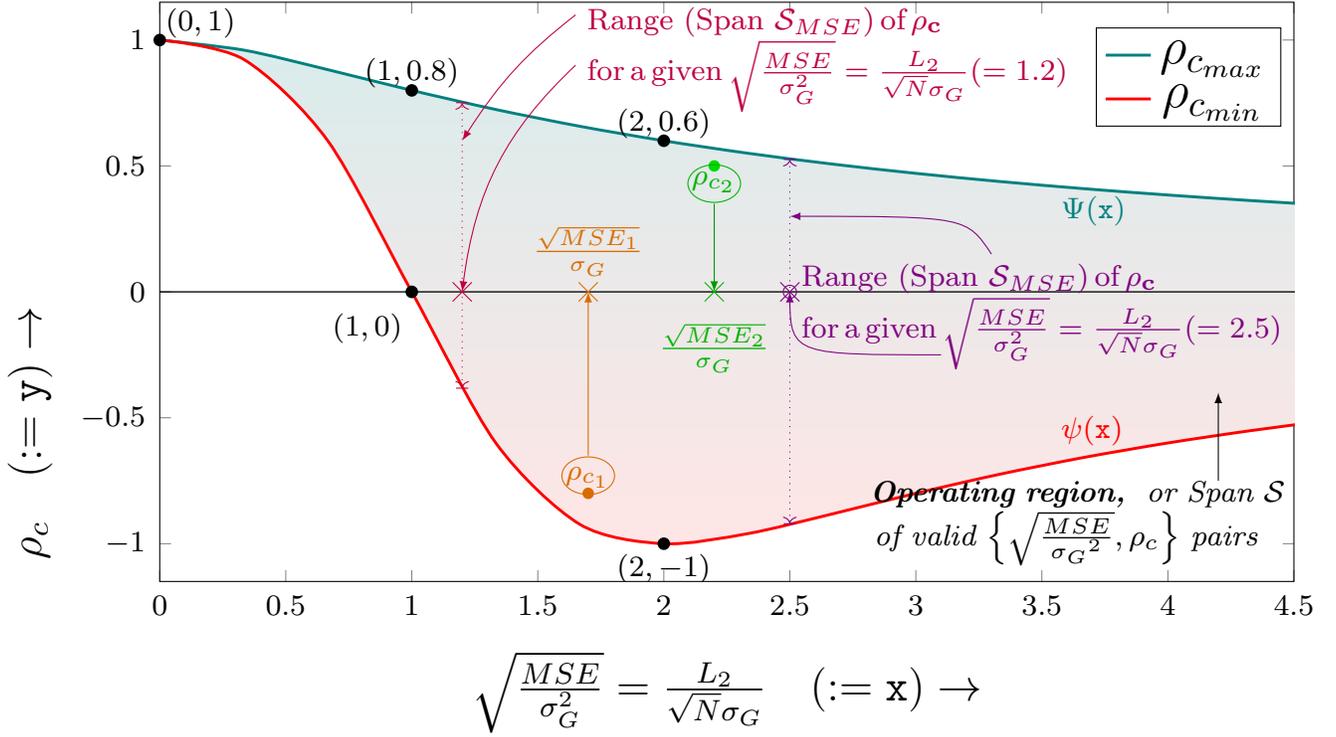
\begin{figure}[!t]
	\centering
	\begin{tikzpicture}
		[ every node/.style={scale=1.4},
%		scale=1.5
		]
	\begin{axis}[
	scale=2.2,
	domain=0:50,
	samples=151,
	smooth,
	no markers,
	xlabel={$\sqrt{\frac{MSE}{\sigma_G^2}}
%		\quad
%		\left(
		=\frac{L_2}{\sqrt{N}\sigma_{G}}
%		\right)
		\quad
		(:=\mathtt{x})
		\rightarrow$ },
	ylabel={$\rho_c
		\quad
		(:=\mathtt{y})		
		\rightarrow$},
	xmin=0,xmax=4.5,
	%extra x ticks={1, 2},
	ymin=-1.15,ymax=1.15,
	unit vector ratio*=0.8 0.8,
	x label style={below,font=\large},
	y label style={left,font=\large},
	%	(0.65,0.6)
	legend style={at={
			%			(axis cs: 5.8, 1.0) % 0 to 6
			%			(axis cs: 4.8, 1.0) % 0 to 5
			(axis cs: 4.45, 1.05) % 0 to 4.5
		},anchor=north east, font=\fontsize{14}{5}\selectfont}, 
	%  legend image post style={size=20pt}, 
	]
	\addplot 
	+[name path=rhomax, line width=1.1pt,teal] 
	{2*(1+x)/(1+(1+x)^2)}; 
	\label{graph1}
	\addplot 
	+[name path=rhomin, line width=1.1pt, red] 
	{2*(1-x)/(1+(1-x)^2)};
	\label{graph2},
	\addplot[shade, top color=teal!15, bottom color=red!10] fill between[of=rhomax and rhomin];
	\addplot [black, solid, line width=0.5]
	{0};
	%\node[circle,inner sep=1pt,fill=red,label=left:{$(2,0.6)$}] at (1,0.6) {};% this replace your `addplot`
	\node[fill,circle,inner sep=1.2pt,label={[xshift=0pt, yshift=-5pt]above:$(2,0.6)$}] at (axis cs: 2,0.6) {};
	\node[fill,circle,inner sep=1.2pt,label={[xshift=0pt, yshift=-5pt]above:$(1,0.8)$}] at (axis cs: 1,0.8) {};
	\pgfplotsset{
		after end axis/.code={
			\node[fill,circle,inner sep=1.2pt,label={[xshift=11pt, yshift=-5pt]$(0,1)$}] at (axis cs: 0,1) {};
		}
	}
	\node[fill,circle,inner sep=1.2pt,label={[yshift=3pt]below:$(2,-1)$}] at (axis cs: 2,-1) {};
	\node[fill,circle,inner sep=1.2pt,label={[xshift=-12pt, yshift=0pt]below:$(1,0)$}] at (axis cs: 1,0) {};
	\draw[violet, dotted] 
	[|<->|] 
	%	(axis cs: 3,-0.8)--(axis cs: 3,0.4706);  % 0 to 5
	(axis cs: 2.5,-0.9231)--(axis cs: 2.5,0.5283);  % 0 to 4.5
	%	
	%%%%%%%%
	%	\draw (45,5.7) edge [dim={ ,-10pt}]  (45,14.5); 
	\node[label={
		%		[xshift=0pt, yshift=-6pt,rotate=-90,text width=2cm,align=center]
		[xshift=0pt, yshift=0pt,rotate=0,
		align=left,violet]
		Range (Span $\mathcal{S}_{MSE}$)\,of\,$\mathbf{\rho_c}$\,\\[0mm]for\,a\,given\,$\sqrt{\frac{MSE}{\sigma_{G}^2}}
		={\frac{L_2}{\sqrt{N}\sigma_{G}}}
		(=2.5)
		$ 
		%		\\[0mm](scaled\,by\,$\sigma_G$)
	}] at 
	%	(axis cs: 4.2,-0.45) {}; % 0 to 6
	%	(axis cs: 3.7,-0.35) {}; % 0 to 5
	(axis cs: 3.5,-0.42) {}; % 0 to 4.5
	\node[fill,cross=3pt,violet,inner sep=1.2pt] at 
	%	(axis cs: 3,0) {}; % 0 to 5
	(axis cs: 2.5,0) {}; % 0 to 4.5
	\draw[violet] 
	%	(axis cs: 3,0) 		% 0 to 5
	(axis cs: 2.5,0)   % 0 to 4.5
	circle (2.5pt);
	%%%%%%%%%% MSE Pointer2
	\draw[-{latex[scale=3.0]},violet] 	
	%	(axis cs: 4.0,-0.2)	.. controls 	(axis cs: 2.8,-0.3)	.. (axis cs: 3.0,0); % 0 to 6
	%	(axis cs: 3.5,-0.2)	.. controls 	(axis cs: 2.3,-0.3)	.. (axis cs: 2.5,0); % 0 to 5
	(axis cs: 3.1,-0.25)	.. controls 	(axis cs: 2.5,-0.25)	.. (axis cs: 2.5,0); % 0 to 4.5
	%%%%%%%%
	\draw[-{latex[scale=3.0]},violet] 	
	%	(axis cs: 4.2,0.23)	.. controls 	(axis cs: 4.1,0.3)	.. (axis cs: 3,0.3);
	%	(axis cs: 3.7,0.2)	.. controls 	(axis cs: 3.6,0.3)	.. (axis cs: 2.5,0.3);
	(axis cs: 3.3,0.15)	.. controls 	(axis cs: 3.2,0.3)	.. (axis cs: 2.5,0.3);
	%ccc1 and cc2
	\node[fill,circle,black!15!orange,inner sep=1.1pt,label={[xshift=0pt, yshift=-4pt,black!15!orange]$\rho_{c_1}$}] at (axis cs: 1.7,-0.8) {};
	\draw [black!15!orange] (axis cs: 1.7,-0.8+0.07) ellipse (0.35cm and 0.25cm);
	\node[fill,circle,black!15!green,inner sep=1.1pt,label={[xshift=0pt, yshift=4pt,black!30!green]below:$\rho_{c_2}$}] at (axis cs: 2.2, 0.5) {};
	\draw [black!30!green] (axis cs: 2.2, 0.5-0.07) ellipse (0.35cm and 0.25cm);
	\node[fill,cross=3pt,black!15!orange,inner sep=1.2pt,label={[xshift=0pt, yshift=-2pt,black!15!orange]above:$\frac{\sqrt{MSE_1}}{\sigma_G}$}] at 
	(axis cs: 1.7,0) {}; % 0 to 4.5
	\node[fill,cross=3pt,black!30!green,inner sep=1.2pt,label={[xshift=0pt, yshift=-2pt,black!30!green]below:$\frac{\sqrt{MSE_2}}{\sigma_G}$}] at 
	(axis cs: 2.2,0) {}; % 0 to 4.5
	\draw[-{latex[scale=3.0]},black!15!orange] 	(axis cs: 1.7,-0.65)--(axis cs: 1.7,0);
	\draw[-{latex[scale=3.0]},black!40!green] 	(axis cs: 2.2,0.35)--(axis cs: 2.2,0);	
	\node[label=
	{
		[xshift=0pt, yshift=0pt,text width=7cm,align=center]
		\textit{
			\textbf{
				Operating region,
			}
			or Span $\mathcal{S}$\\[0mm] of valid $\left\{\sqrt{\frac{MSE}{ {\sigma_{G}}^2} },\rho_c\right\}$ pairs 	
		}
	}
	] 
	%	at (axis cs: 4.7,-1.3)  % 0 to 6
	%	at (axis cs: 4.0,-1.3)  % 0 to 5
	at (axis cs: 3.6,-1.2)  % 0 to 4.5
	{};
	\draw[-{latex[scale=3.0]}] 
	%	(axis cs: 5.5,-0.7)--(axis cs: 5.5,-0.3); % 0 to 6 
	%	(axis cs: 4.7,-0.7)--(axis cs: 4.7,-0.3); % 0 to 5 
	(axis cs: 4.2,-0.75)--(axis cs: 4.2,-0.4); % 0 to 4.5 
	\legend{$\rho_{c_{max}}$ ,$\rho_{c_{min}}$ }
	% legend style={at={(1,1)},anchor=north east}
	%%%%%%%  RANGE ON TOP %%%%%%%%%%%%%%%%%%%%%%%%
	\draw[purple,dotted] 
	[|<->|] 
	%	(axis cs: 1.3,-0.5505)--(axis cs: 1.3,0.7313);  % 0 to 4.5
	(axis cs: 1.2,-0.3846)--(axis cs: 1.2,0.7534);  % 0 to 4.5
%	(axis cs: 1.4,-0.69)--(axis cs: 1.4,0.71);  % 0 to 4.5
	%%%%%%%%
	\node[label={
		[xshift=0pt, yshift=0pt,rotate=0,
		align=left, purple]
		Range (Span $\mathcal{S}_{MSE}$)\,of\,$\mathbf{\rho_c}$\,\\[0mm]for\,a\,given\,$\sqrt{\frac{MSE}{\sigma_{G}^2}}
		={\frac{L_2}{\sqrt{N}\sigma_{G}}}
		(=1.2)
		$ 
	}] at 
	%	(axis cs: 2.3,0.7) {}; % 0 to 4.5
	(axis cs: 2.65,0.6) {}; % 0 to 4.5
	\node[fill,cross=3pt,purple,inner sep=1.2pt] at 
	%	(axis cs: 1.3,0) {}; % 0 to 4.5
	(axis cs: 1.2,0) {}; % 0 to 4.5
	\draw[purple] 
	%	(axis cs: 1.3,0)   % 0 to 4.5
	(axis cs: 1.2,0) {}; % 0 to 4.5
	circle (2.5pt);
	%%%%%%%%%% MSE Pointer2
	\draw[-{latex[scale=3.0]}, purple] 	
	(axis cs: 1.65,0.9) to [bend right=20] (axis cs: 1.2,0);
	\draw[-{latex[scale=3.0]}, purple] 	
	%(9.6,16.5+3) .. controls (5,16.5+3) .. (3,16.5+3);	
	(axis cs: 1.65,1.1) to [bend right=10] (axis cs: 1.2,0.6);	
	%%%%%%%%%%%%%%%%%%%%%%%%%%%%%%%%
	\node[inner sep=1.2pt,label={[xshift=0pt, yshift=0pt,teal]right:$\Psi(\mathtt{x})$}] at (axis cs: 3.5,0.32) {};
	\node[inner sep=1.2pt,label={[xshift=0pt, yshift=0pt,red]right:$\psi(\mathtt{x})$}] at (axis cs: 3.5,-0.55) {};
	\end{axis}
	\end{tikzpicture}
	%\ref{graph1} $\rho_{c_{max}}$ \qquad \ref{graph2} $\rho_{c_{min}}$
	%	\vspace{-1.5cm}
	\caption{Range of $\rho_c$ for a given $MSE$ in proportion to  $\sigma_G^2$, (\ie to  the standard deviation of the gold standard). 
		Note that $\rho_{c_1}$ can be  $<\rho_{c_2}$, even though $MSE_{1}<MSE_{2}$. 
		The span $\mathcal{S}$ of valid  $\left\{\sqrt{\frac{MSE}{\sigma_{G}^2}},\rho_c\right\}=\left\{\mathtt{x},\mathtt{y}\right\}$ pairs is constrained by 
		$\Psi(\mathtt{x})=\frac{2\times(1+\mathtt{x})}{1+(1+\mathtt{x})^2}$ and $\psi(\mathtt{x})=\frac{2\times(1-\mathtt{x})}{1+(1-\mathtt{x})^2}$. 
		%		$\rho_{{c}_{max}}$ and $\rho_{{c}_{min}}$ (cf. \Cref{eqn_psi1,eqn_psi2}).
		%		\ie the curves $\mathtt{y}=\Psi(\mathtt{x})=\frac{2\times(1+\mathtt{x})}{1+(1+\mathtt{x})^2}$ and $\mathtt{y}=\psi(\mathtt{x})=\frac{2\times(1-x)}{1+(1-\mathtt{x})^2}$. 
		\label{fig_maxminrho}}
\end{figure}
%%%%% 0 TO 6 ENDS
%
%Thus, from \Cref{eqn_rhocmax,eqn_rhocmin},
%the $\rho_c$ can vary between 
%$\frac{2\left(1-{\sqrt{\frac{MSE}{\sigma_G^2}}}\right)}{1+\left(1-{\sqrt{\frac{MSE}{\sigma_G^2}}}\right)^2}$
%and
%$\frac{2\left(1+{\sqrt{\frac{MSE}{\sigma_G^2}}}\right)}{1+\left(1+{\sqrt{\frac{MSE}{\sigma_G^2}}}\right)^2}$ 
%for given $MSE$ -- depending on how $MSE$ is split into its constituent errors.
%
% (cf. \Cref{fig_maxminrho}). 
%
%\vspace{-1cm}
%Notice in \Cref{fig_maxminrho} that $\rho_{c_1}<\rho_{c_2}$, even though for the corresponding mean square errors, $MSE_{1}<MSE_{2}$. Thus, 

%Notice in \Cref{fig_maxminrho} that:
%\begin{description}
%	\item
%\noindent
Thus, in a two dimensional space $\mathbb{R}^2:=(\mathtt{X},\mathtt{Y})$, where $\mathtt{x}=\left|\sqrt{\frac{MSE}{\sigma_G^2}}\right|$, $\mathtt{y}=\rho_c$ (cf. \Cref{fig_maxminrho}), 
%and $\Upsilon(\mathtt{t}):=\frac{2\times\mathtt{t}}{1+\mathtt{t}^2}$: 
%	\\
%	\vspace{-12pt}
\begin{flalign}
\hspace{18pt}
&
%\hspace{-0.54cm}
\label{eqn_psi1}
\mathit{\bullet}
\hspace{6pt}
\max_\mathtt{x}(\mathtt{y}):=\rho_{c_{max}}&&=
\quad
%\Aboxed{
	\Psi(\mathtt{x}):=\frac{2\times(1+\mathtt{x})}{(1+(1+\mathtt{x})^2)}
%}
\quad=\quad
\Upsilon(1+\mathtt{x}),\quad \text{and} &\\
\hspace{18pt}
&
%\hspace{-0.54cm}
\label{eqn_psi2}
\mathit{\bullet}
\hspace{6pt}
\min_\mathtt{x}(\mathtt{y}):=\rho_{c_{min}}&&=
\quad
%\Aboxed{
	\psi(\mathtt{x}):=\frac{2\times(1-\mathtt{x})}{(1+(1-\mathtt{x})^2)}
%}
\quad=\quad
\Upsilon(1-\mathtt{x}),\quad \text{where } \Upsilon(\mathtt{t}):=\frac{2\times\mathtt{t}}{1+\mathtt{t}^2}.&
\end{flalign}
%	$\mathtt{Y}_{max}=\rho_{c_{max}}=\Psi(x)=\frac{2\times(1+x)}{1+(1+x)^2}$ and $\mathtt{Y}_{min}=\rho_{c_{min}}=\psi(x)=\frac{2\times(1-x)}{1+(1-x)^2}$.
%	(cf. \Cref{fig_maxminrho}). 
%\newpage
%Notice in \Cref{fig_maxminrho} that:
\begin{itemize}
	\item
$\rho_{c_{min}}$ degrades to $-1$ while $\rho_{c_{max}}$ has only degraded to $0.6$. %(\ie $MSE=4\sigma_G^2$).
That is, 
%  \emph{
while $MSE=0$ does translate to a perfect identity relationship, and consequently $\rho_c=1$, the $\rho_{c_{min}}$ degradation is lot quicker than that for $\rho_{c_{max}}$
with increasing $MSE$.
%  }.
%  	 than $\rho_{c_{max}}$ with increasing $MSE$}.
	\item
	$\rho_{c_1}<$
	%$\text{ can be }<\rho_{c_2}$, 
	%can be less than 
	$\rho_{c_2}$
	even though 
	for the corresponding mean square errors,
	$MSE_{1}<MSE_{2}$. That is, \emph{the $MSE$ reduction does not automatically translate to $\rho_c$ improvement}. 
%\item
%For the sake of completeness, we note here that even if the constituent error set that makes $MSE$ is known fully, $\rho_c$ cannot be estimated. The knowledge of not only the values of the constituent errors, but also their sequence is a prerequsite for estimating $\rho_c$, as we prove later in \Cref{sec_prob}. 
\end{itemize} 
\end{remark}
In summary, as per \Cref{thm41,thm42},
%	eqn_rhocmax,eqn_rhocmin
the $\rho_c$ can vary between 
$\frac{2\left(1-{\sqrt{\frac{MSE}{\sigma_G^2}}}\right)}{1+\left(1-{\sqrt{\frac{MSE}{\sigma_G^2}}}\right)^2}$
and
$\frac{2\left(1+{\sqrt{\frac{MSE}{\sigma_G^2}}}\right)}{1+\left(1+{\sqrt{\frac{MSE}{\sigma_G^2}}}\right)^2}$ 
for given $MSE$ -- depending on how $MSE$ is split into its constituent errors.
For the sake of completeness, we note here that even if the constituent error-set that makes $MSE$ is known fully, $\rho_c$ cannot be estimated. The knowledge of not only the values of the constituent errors, but also their sequence is a prerequsite for estimating $\rho_c$, as we prove later in \Cref{sec_prob}. 
%Having established that the $MSE$ optimisation does not necessarily translate to $\rho_c$ optimisation, 
%%we investigate next the mapping between $\rho_c$ and error-set $L_p$ norm, $p>0$.
%we investigate next the mapping between $\rho_c$ and 
%the $L_p$ norm ($p>0$) of the errors in general, by computing first the minimum and maximum possible $MSE$.
%
%Having established that the $MSE$ optimisation does not necessarily translate to $\rho_c$ optimisation, 
%%we investigate next the mapping between $\rho_c$ and error-set $L_p$ norm, $p>0$.
%we investigate next the mapping between $\rho_c$ and 
%the error-set $L_p$-norm ($p>0$), by computing first the minimum and maximum possible $MSE$, starting with $p=1$.
%Further, as proved later in \Cref{sec_prob}, even if the set of errors that make $MSE$ is completely known, $\rho_c$ cannot be estimated exactly. To that end, knowing not only the values, but also the sequence of the constituent errors is crucial. 
%in determining the prediction performance in terms of $\rho_c$.
%Having established that the $MSE$ optimisation does not automatically translate to $\rho_c$ optimisation, 
%we investigate next the relationship between $\rho_c$ and $L_p, p>0$ norm of the errors in general, by computing first the minimum and maximum possible $MSE$.
%, starting with a given $L_1$ or $MAE$.
%\newpage
%\end{remark}
%or $L_2$-norm of the error-set 
%%%% There was MAE HERE
\section{Why other $L_p$-norms fail as a loss function (even more spectacularly than $MSE$)}
\label{sec_lp}
\newcommand{\explainmpe}{
	We intentionally avoid  using the term $MpE$ for `Mean $p$-Powered Error' (although more consistent with the term $L_{\mathbf{p}}$), since $MPE$ is more popularly the `Mean Percentage Error' in the literature.}
For error reduction, one can also use $MAE$, or mean `Mean $k$-Powered Error (\ie $MkE$\footnote{\explainmpe}) in general instead of $MSE$, \ie choosing an optimal $k$ that could be bigger or smaller than 2.
%Let $MkE$ denote `Mean $k$-Powered Error'. 
%or the `power mean'.
%For $k=2$, $MkE=MSE$, and for $k=1$, $MkE=MAE$.  
\begin{align}
%\because
\text{Because }
MkE := \frac{\left(\sum_{i=1}^{N}|d_i|^k\right)}{N}=\frac{L_k^k}{N}  \quad \implies 
L_k
= N^\frac{1}{k}\cdot MkE^\frac{1}{k}
.
\label{eqn_mke}
\end{align}

%Similar to $MSE$ and $MAE$, mean cubic error or any $L_p$-norm of the error-set (\ie $MkE$ with $k=p$) may be used as a loss function for error reduction.  
Similar to $MSE$, while the minimisation of $MkE$ (\ie minimisation of the errors)and the maximisation of $\rho_c$ are both directed at achieving the perfect identity relationship between the labels (\ie the gold standard) and the predictions, the efforts for minimisation of $MkE$ do not necessarily translate into maximisation of $\rho_c$, and vice versa. In this section, we reason and prove this fact mathematically. 
%by deriving the conditions and the formulations for $\rho_c$ optimisation at a given $MkE$ (\ie given the error-set $L_k$-norm).

While no known direct mapping exists between $MkE$ (for $k\neq2$) and $\rho_c$, the many-to-many mapping existing between $MSE$ and $\rho_c$ (established in \Cref{sec_m2mg}), and the inequality relationship existing between $MkE (k\neq2)$ and $MSE$ can be used to establish the formulations and conditions for the minimum and maximum possible values of $\rho_c$ at the given $MkE$, through optimisation of $MSE$.

%While we do not yet know of any direct mapping existing between $MAE$ and $\rho_c$, we do know of the many-to-many mapping existing between $MSE$ and $\rho_c$ (established in \Cref{sec_m2mg}), and the relationship that exists between $MAE$ and $MSE$. Making use of these two relationships, we establish the conditions for achieving minimum and maximum possible value of $\rho_c$ at the given $MAE$, through optimisation of $MSE$.

\subsection{$\rho_c$ optimisation using $MSE$ optimisation, given the error-set $L_p$-norm, 	$p>0$}
%H\"older's inequality (cf. \Cref{sec_app3}) leads to the following famous identity for $0<r<p$:
\vspace{-12pt}
\begin{alignat}{4}
%\hspace{6pt}
\text{For } \hspace{6pt}
&  0<r<p
%\hspace{60pt}
\text{ as per \Cref{sec_app3}: }
L_p && \leq  L_r  &&&& \leq  N^{\frac{p-r}{pr}} L_p .\label{eqn_lprel}
\\
\therefore %\text{and }
\text{For } \hspace{6pt}
&k>2: \hspace{110pt} 
L_{k}   &&\leq   L_{2}
&&&&\leq   N^{\frac{k-2}{2k}} \cdot L_{k}
%\quad
%\because \text{ \Cref{eqn_mke}}.
\nonumber%%
\\
%\implies
%\hspace{6pt}
%\because
% \text{ \cref{eqn_mke},}\hspace{6pt} 
\therefore\hspace{6pt}
&\text{From \Cref{eqn_mke}: } N^\frac{1}{k}\cdot MkE^\frac{1}{k}   &&\leq   N^\frac{1}{2}\cdot MSE^\frac{1}{2}
&&&&\leq   N^{\frac{k-2}{2k}}  N^\frac{1}{k}\cdot MkE^\frac{1}{k}.
\nonumber%%
\\
\implies
\hspace{6pt}
&
\tikzmark{start1real}
\hspace{90pt}N^{\frac{2-k}{2k}}\cdot MkE^\frac{1}{k}   &&\leq   MSE^\frac{1}{2}
&&&&\leq  MkE^\frac{1}{k}. 
\tikzmark{end1real}
\nonumber%%
\\
\label{eqnl2rel}
\tikzmark{start1}
\text{\ie }
\hspace{6pt}
&\sqrt{MSE_{min}}   &&\leq   
\sqrt{MSE}=
\theta\cdot &&\sqrt{MSE_{min}}
&&\leq \theta_{max}\cdot \sqrt{MSE_{min}},
\\
\nonumber
\text{where: }
\hspace{6pt}
&\sqrt{MSE_{min}}
=\frac{L_k}{\sqrt{N}}
= N^{\frac{2-k}{2k}}\cdot MkE^\frac{1}{k},&& \hspace{12pt}\text{ and }
%\span\span
&&
\sqrt{MSE_{max}} &&=\theta_{max}\cdot \sqrt{MSE_{min}},
%\span\span
\\
\nonumber
\text{and }
\hspace{6pt}
&
\theta_{min}= \hspace{18pt} 1   &&\leq  
\hspace{35pt} 
\theta
:= &&\frac{N^{\frac{1}{2}}MSE^{\frac{1}{2}}}{N^{\frac{1}{k}}MkE^{\frac{1}{k}}}
&&\leq N^{\frac{k-2}{2k}}=\theta_{max}.
\tikzmark{end1}
\\
\nonumber
%&
\rlap{\hspace{-1.5cm}Similarly (cf. \Cref{sec_app01_kl2}), for $k\in(0,2)$: }
\\
\label{eqnl2rellek}
&
\tikzmark{start2}
\sqrt{MSE_{min}} &&\leq   
\sqrt{MSE}=
\theta\cdot &&\sqrt{MSE_{min}}
&&\leq \theta_{max}\cdot \sqrt{MSE_{min}},
\\
\nonumber
\text{where: }
\hspace{6pt}
&\sqrt{MSE_{min}}
=\frac{L_k}{N^{\frac{1}{k}}}
= 
MkE^\frac{1}{k}&& \hspace{12pt} { and }
%\span\span
%\hspace{48pt}
&&\sqrt{MSE_{max}} &&=\theta_{max}\cdot \sqrt{MSE_{min}},
%\span\span
\\
\nonumber
\text{and }
\hspace{6pt}
&\theta_{min}= \hspace{18pt} 1   &&\leq  
\hspace{35pt} 
\theta
:=&& \frac{MSE^{\frac{1}{2}}}{MkE^{\frac{1}{k}}}
&&\leq N^{\frac{2-k}{2k}}=\theta_{max}.
\tikzmark{end2}\\
\therefore\text{For }
&\hspace{6pt} k=1 : \hspace{12pt}
\sqrt{MSE_{min}}   &&\leq 
%\hspace{30pt}
\sqrt{MSE}=\theta\cdot &&\sqrt{MSE_{min}} &&\leq  \theta_{max}\cdot \sqrt{MSE_{min}},
\\
\nonumber
\text{where: }
&\hspace{6pt}
%\quad
\sqrt{MSE_{min}}=\frac{L_1}{N} = MAE && \hspace{12pt}\text{ and }
%\span\span
%\hspace{66pt}
&&\sqrt{MSE_{max}}&&=\theta_{max}\cdot \sqrt{MSE_{min}},
%\span\span
\\
\nonumber
\text{and }
&\hspace{6pt}
\theta_{min}= \hspace{12pt} 1   &&\leq  \hspace{35pt} \theta
&&&&\leq \sqrt{N}=\hspace{12pt}\theta_{max}.
\end{alignat}

The lower the $\theta$ (\ie the lower the $MSE$), the higher is the maximum theoretical limit for $\rho_c$ at the given $MkE$; $\rho_{{c}_{max}}$ being a monotonic function of $MSE$ (cf. \Cref{fig_maxminrho}). However, the same cannot be said for the minimum theoretical limit for $\rho_c$ at the given $MkE$,  since $\rho_{{c}_{min}}$ is not a monotonic function of $MSE$ (cf. \Cref{fig_maxminrho}). Notice that attaining these theoretical limits (\ie $\rho_{{c}_{max}}$ and $\rho_{{c}_{min}}$) is subject to also meeting simultaneously the conditions dictated by \Cref{thm41,thm42} at the given $MkE$ with the given gold standard, 
which can not be guaranteed to be true of any gold standard as a general case. Thus, for the sake of clarity, we denote these theoretical limits at given $L_k$ with $\rho_{{c}_{max'}}$  and $\rho_{{c}_{min'}}$ respectively.
%\Cref{sec_m2mm} dicates that the lower the $MSE$, the higher the $\rho_{{c}_{max}}$; $\rho_{{c}_{max}}$ being a monotonic function of $MSE$ (cf. \Cref{fig_maxminrho}). Therefore, 
%one might be lead to 
%%\emph{falsely} 
%believe 
%(although \emph{incorrectly}, as  will be proved in \Cref{sec_truespan}) 
%that
%for a given $MkE$ as well,
%% -- effectively, for a given value of $L_k$ norm of the errors -- 
%lower the $MSE$, higher the $\rho_{{c}_{max}}$. 
%%From \Cref{eqnl2rel}, the lowest possible value of $\sqrt{MSE} = \frac{MkE^\frac{1}{k}}{N^\frac{k-2}{2k}}$. 
%Because $\rho_{{c}_{min}}$ is not a monotonic function of $\sqrt{MSE}$, computation of $\rho_{{c}_{min}}$ from $MkE$ is not as straight forward (cf. $\phi(\theta\cdot\mathtt{x})$ plots in \Cref{fig_maxminrhoLp}). 
%
%Given $MkE$, $\sqrt{MSE}$ can vary between 
%$\frac{MkE^\frac{1}{k}}{\sqrt{{N^\frac{k-2}{k}}}}$ 
%${N^\frac{2-k}{2k}\cdot MkE^\frac{1}{k}}$ 
%and $MkE^\frac{1}{k}$ 
%(cf. \Cref{{eqnl2rel}}), for $k>2$. 
%\hspace{-12pt}
%
%For $k>2$, as per \Cref{eqnl2rel}:
%
\begin{alignat}{2}
\label{eqnrhocmaxmink}
\therefore
\text{From \Cref{eqnl2rel}, for } k\geq 2:  
\quad
&
	\rho_{{c}_{max'}}= 
	\frac{2\Big(
		1+{\frac{L_k}{\sqrt{N}\sigma_G}}
		\Big)}{1+\Big(1+
		{\frac{L_k}{\sqrt{N}\sigma_G}}
		\Big)^2}
	=
	\Psi\left(\frac{L_k}{\sqrt{N}\sigma_G}\right),
%\quad, \text{ and }
%\label{eqnrhocmaxk}
\nonumber
\\
%\text{as per \Cref{eqnl2rel}} 
%\quad
\text{ and }\quad
&
	\rho_{{c}_{min'}}
	=
	\frac{2\Big(
		1-{\theta\frac{L_k}{\sqrt{N}\sigma_G}}
		\Big)}{1+\Big(1-
		{\theta\frac{L_k}{\sqrt{N}\sigma_G}}
		\Big)^2} 
	=\psi\left(\frac{\theta\cdot L_k}{\sqrt{N}\sigma_G}\right),
	%1\leq\theta\leq {N}^{\frac{k-2}{2k}}
	\theta\in\left[1,N^{\frac{k-2}{2k}}\right].
%\text{ (cf. \Cref{fig_maxminrhoLp})}
\\
%\nonumber
%\text{Similarly, for }  
\text{From \Cref{eqnl2rellek}, for } k \in (0,2]: 
\quad
&
	\rho_{{c}_{max'}}=
	\Psi\left(\frac{L_k}{N^\frac{1}{k}\sigma_G}\right),
	\rho_{{c}_{min'}}=
	\psi\left(\frac{\theta\cdot L_k}{N^\frac{1}{k}\sigma_G}\right),
	\theta\in\left[1,N^{\frac{2-k}{2k}}\right].
\label{eqnrhocmaxmink2}
\end{alignat}
%
%These derivations assume $k\geq2$, and thus, ${N^\frac{2-k}{2k}}\leq1$. 
%Consequently, a deviation of (\ie a reduction in) the multiplying factor ${N^\frac{2-k}{2k}}$ away from the maximum possible value of $1$ is an increase in ${N^\frac{k-2}{2k}}$ (due to increase in $N$) is as illustrated in \Cref{fig_maxminrhoLp}.
%
%\begin{tikzpicture}[remember picture,overlay]
%\draw[line width=1,black]
%([xshift=18pt,yshift=6ex]current page text area.west|-{pic cs:beg-psi})
%rectangle
%([xshift=-24pt,yshift=-4.5ex]current page text area.east|-{pic cs:end-psi});
%\end{tikzpicture}
Thus, from \Cref{eqnrhocmaxmink,eqnrhocmaxmink2}, irrespective of the value of $k>0$ (\ie whether $k\geq2$ or $k\leq2$):
\begin{alignat}{1}
\tikzmark{beg-psi}
%\therefore
%\forall k \in (0,\infty),\quad
\rho_{{c}_{max'}}=\Psi(\mathtt{x}), 
%\hspace{1pt}
\mathtt{x}\in[0,\infty),
\text{ and }
%\hspace{3pt}
\rho_{{c}_{min'}}=
\begin{cases}
\psi(\theta_{max}\cdot\mathtt{x}), &
\hspace{-6pt} 
\mathtt{x}\in\big[0,\frac{2}{\theta_{max}}\big].\\
-1=\psi(\theta_0\cdot \mathtt{x}),
&\hspace{-6pt} \mathtt{x}\in\big[\frac{2}{\theta_{max}},2\big],
\hspace{6pt} \theta_0:=\frac{2}{\mathtt{x}}, 
\hspace{6pt}
\theta_{max}:=N^{\left|\frac{k-2}{\mathtt{2k}}\right|}\\
\psi(\mathtt{x}),  &\hspace{-6pt} \mathtt{x}\in[2,\infty).
\end{cases}
%\forall k \in (0,\infty)
\tikzmark{end-psi}
\label{eqn_generalmaxmin}
\\
\nonumber
\text{in the $\mathbb{R}^2=\left\{\mathtt{X},\mathtt{Y}\right\}=\left\{\frac{L_k}{\sqrt{N}\sigma_G},\rho_c\right\}$ space for $k\geq2$, and in the $\left\{\mathtt{X},\mathtt{Y}\right\}=\left\{\frac{L_k}{N^{\frac{1}{k}}\sigma_G},\rho_c \right\}$ space for $k\leq2$.}\span
\end{alignat}

For the sake of completeness, we note here
that the 
range for $\rho_{c}$
at a given $MkE$ is typically even smaller than the one dictated by \Cref{eqn_generalmaxmin} above (as we establish later in
\Cref{sec_truespan}). 
% Error distribution
% ei=ej
% MSE=Nei^2/N=ei^2
% MAE=Nei/N=ei
% sqrt(MSE)/MAE=1
% theta=1		MSEmax=MAE
%
% ei/N=MAE
% ei^2/N=MSE
% sqrt(MSE)=ei/sqrt(N)=MAE/sqrt(N)
% MAE/sqrt(MSE)=sqrt(N)		MSEmin=MAE/sqrt(N)
%
% MAE given
%
% 1st case: ei=ej=MAE
% MSE = N*MAE^2/N = MAE^2
% sqrt(MSE)/MAE=1
% theta=1		
%
% 2nd case: ei=MAE*N, rest 0
% MSE = (MAE^2*N^2)/N=MAE^2/N
% MAE/sqrt(MSE)=sqrt(N)
% MSEmin=MAE/sqrt(N)
%
% MSE given
%
% 1st case: ei=ej=sqrt(MSE)
% MAE = sqrt(MSE)
% sqrt(MSE)/MAE=1
% theta=1		
%
% 2nd case: ei^2/N=MSE, rest 0
% ei = sqrt(N*MSE)=N*MAE
% MAE = sqrt(N*MSE)/N
% MAE = sqrt(MSE)/sqrt(N)
% sqrt(MSE) = MAE * sqrt(N)
%
\newcommand{\maemseboundaries}{
$
	\begin{aligned}
	\text{When }& e_i=\pm e_j, &\forall i,j\in[1,N], & \quad i,j\in\mathbb{N} &&\implies 
%	{MAE}=e_i=\sqrt{MSE} 
	{MAE}=|e_i|, \quad MSE=e_i^2
	&&\implies\mathop{\theta}_{k=1}:=\frac{\sqrt{MSE}}{MAE}=1.
	\nonumber
	\\
	\text{When }& e_j=0, & \exists!i, \forall j\neq i, i,j\in[1,N], & \quad i,j\in\mathbb{N} &&\implies 
	{MAE}=\frac{|e_i|}{N},\quad MSE={\frac{e_i^2}{N}}
%	{MAE}=\frac{e_i}{N}=\frac{\sqrt{MSE}}{\sqrt{N}}
%	\sqrt{N\cdot MSE}=e_i={N\cdot MAE} 
	&&\implies\mathop{\theta}_{k=1}:=\frac{\sqrt{MSE}}{MAE}=\sqrt{N}.\nonumber
	\end{aligned}
$
}
This is because the  boundary conditions for $\theta$ (\eg $\theta=\frac{\sqrt{MSE}}{MAE}=\{1,\sqrt{N}\}$) are met only when the error coefficients are constant valued 
-- either entirely, or except at one instance \citep{willmott2005advantages} \footnote{\maemseboundaries}. 
%-- either entirely (\ie $e_i=e_j=\sqrt{MSE}={MAE}, \forall i,j\in[1,N], i,j\in\mathbb{N}, \text{ resulting in } \theta=1$), or except at one instance ($e_i=\sqrt{N\cdot MSE}={N\cdot MAE},e_j=0, \exists!i, \forall j\neq i, i,j\in[1,N], i,j\in\mathbb{N}, \text{ resulting in } \theta=\sqrt{N}$) \citep{willmott2005advantages}. 
Simultaneously, $\rho_{{c}_{max'}}$ and $\rho_{{c}_{min'}}$ are obtained if and only if the error coefficients are in the same ratio as of the deviations of the corresponding gold standards (cf. \Cref{thm41,thm42}), which forces the gold standard to be constant valued likewise -- \ie either entirely, or except at one instance -- which is not true in general. Thus, assuming $\rho_{c_max'}$ from \Cref{eqn_generalmaxmin} to be the true maximum limit of $\rho_c$ at any given $L_k$ norm is equivalent of defining the gold standard to be constant-valued for at least $N-1$ instances.

Nonetheless, plotting the incorrect span of valid $\{\rho_c, L_k\}$  pairs as dictated by \Cref{eqn_generalmaxmin} gives us a few new insights still -- as to how assuming even this incorrect line of argument 
(\ie a more optimistic  $\rho_{c_{max}}$ as a function of $MkE$, and consequently, $\rho_c$ maximisation through $MkE$ minimisation) leads us to a more discouraging end-result to the contrary. We discuss in \Cref{sec_m2mLp} the derivations for true span $\mathcal{S}$ of valid $\left\{\rho_{c},L_p\right\}$ pairs.

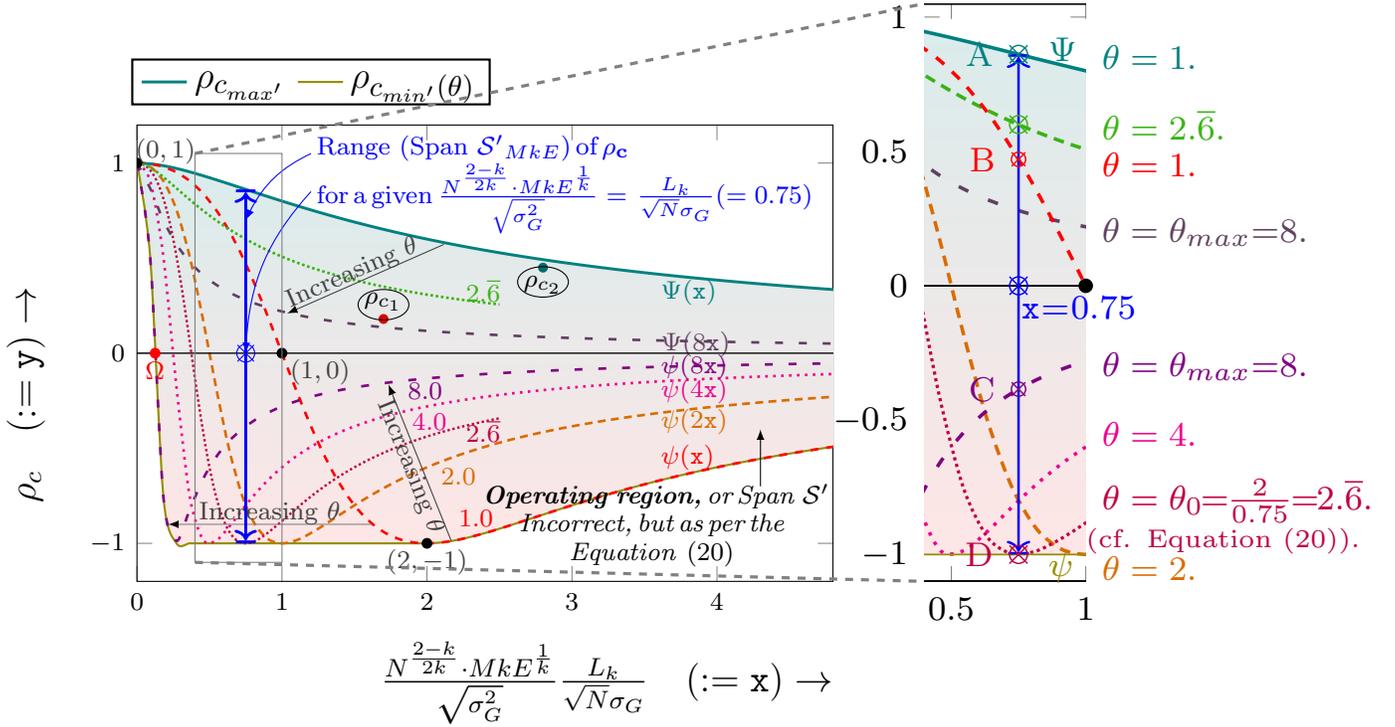
\begin{figure}[!t]
		\vspace{-1.3cm}
		\hspace{1.2cm}
%	\begin{center}
%\centering
%\hspace{0.4cm}
%	\noindent
%	\resizebox{\textwidth}{!}{
	\begin{tikzpicture}
	[thick,scale=1.35, every node/.style={scale=1.2},
%	every edge quotes/.append style={font=\scriptsize, align=center, auto}% style for edge labels
	trim left=(ax1.south west),trim right=(ax2.south east),
	declare function={
		func(\x)= (\x < 0.25) * 2*((1-8*x))/(1+(1-8*x)^2)   +
		and(\x >= 0.25, \x < 2) * (-1)     +
		(\x >= 2) * (2*(1-x)/(1+(1-x)^2))
		;
	}]
	\begin{axis}[
	name=ax1,
	domain=0:50,
	samples=551,
	smooth,
	no markers,
	xlabel={$
		%		\mathtt{x}=
		\frac{N^\frac{2-k}{2k}\cdot MkE^\frac{1}{k} }{\sqrt {\sigma_G^2} } 
%		\quad
%		\left(= 	
		\frac{L_k}{\sqrt{N}\sigma_G }
		\quad
		(:= \mathtt{x})
%		\right)
		\rightarrow$   },
	ylabel={$
		%		\mathtt{y}=
		\rho_c 
		\quad
		(:= \mathtt{y})
		\rightarrow$},
	scale=1.,
	xmin=0,xmax=4.8,
	%extra x ticks={1, 2},
	ymin=-1.2,ymax=1.2,
	unit vector ratio*=0.8 1.05,
	x label style={below,font=\large},
	y label style={left,font=\large},
	legend columns=2, 
	legend style={
		at={
			%			(0.328,0.9)
			(axis cs: 1.2,1.55)
		},
		anchor=north,
		/tikz/column 2/.style={
			column sep=2pt,
		},
		font=\fontsize{13}{5}\selectfont},
	%	legend style={
	%		at={
	%%			(0.65,0.6)
	%			(axis cs: 4.0,1.1)
	%		},anchor=north east,
	%		 font=\fontsize{14}{5}\selectfont}, 
	%	  legend image post style={size=20pt}, 
	%	set layers,
	%	cell picture=true,
	]
	\node[fill,circle,black!15!red,inner sep=1.1pt,label={[xshift=0pt, yshift=-3pt,black
%		black!15!orange
		]above:$\rho_{c_1}$}] at (axis cs: 1.7, 0.18) {};
	\node[fill,circle,black!15!teal,
	inner sep=1.1pt,label={[xshift=0pt, yshift=3pt,black
%		black!30!green
		]below:$\rho_{c_2}$}] at (axis cs: 2.8, 0.45) {};
	\draw [black
%	black!15!orange
	] (axis cs: 1.7,0.18+0.075) ellipse (0.25cm and 0.15cm);
	\draw [black,
%	black!15!green
	] (axis cs: 2.8,0.45-0.075) ellipse (0.25cm and 0.15cm);
	\addplot 
	+[name path=rhomax, line width=1.1pt,teal] 
	{2*(1+x)/(1+(1+x)^2)}; 
	\label{graphkmax}
	\addplot +[name path=rhominall,olive,thick] {func(x)};
	%%%%%%%%%
	\addplot[domain=0:50, red, line width=1,dashed] {2*(1-x)/(1+(1-x)^2)};
	\addplot[ black!15!orange, line width=1,densely dashed] {2*(1-2*x)/(1+(1-2*x)^2)};
	\addplot[purple, line width=0.9,densely dotted, domain=0:2.5] {2*(1-2.66667*x)/(1+(1-2.66667*x)^2)};
	\addplot[teal!50!purple, line width=1,loosely dashed] {2*(1+8.0*x)/(1+(1+8.0*x)^2)};
	\addplot[green!70!purple, line width=1,densely dotted, domain=0:2.5] {2*(1+2.66666666*x)/(1+(1+2.66666666*x)^2)};
	\addplot[ magenta, line width=1,dotted] {2*(1-4*x)/(1+(1-4*x)^2)};
	\addplot[ violet, line width=1,loosely dashed] {2*(1-8*x)/(1+(1-8*x)^2)};
	%%%%%%%%%
	\node[inner sep=1.2pt,label={[xshift=0pt, yshift=0pt,teal]right:$\Psi(\mathtt{x})$}] at (axis cs: 3.5,0.32) {};
	\node[inner sep=1.2pt,label={[xshift=0pt, yshift=0pt,teal!50!purple]right:$\Psi(8\mathtt{x})$}] at (axis cs: 3.5,0.05) {};
	\node[inner sep=1.2pt,label={[xshift=0pt, yshift=0pt,red]right:$\psi(\mathtt{x})$}] at (axis cs: 3.5,-0.55) {};
	\node[inner sep=1.2pt,label={[xshift=0pt, yshift=0pt,black!15!orange]right:$\psi(2\mathtt{x})$}] at (axis cs: 3.5,-0.36) {};
	\node[inner sep=1.2pt,label={[xshift=0pt, yshift=0pt,magenta]right:$\psi(4\mathtt{x})$}] at (axis cs: 3.5,-0.2) {};
	\node[inner sep=1.2pt,label={[xshift=0pt, yshift=0pt,violet]right:$\psi(8\mathtt{x})$}] at (axis cs: 3.5,-0.074) {};
	%%%%%%%%%
	%	\addplot 
	%	+[name path=rhomin1, line width=1.1pt, red, thick, dashed] 
	%	{2*(1-x)/(1+(1-x)^2)};
	%	\label{graph1k},
	%	\addplot 
	%	+[name path=rhomin2, line width=1pt, black!15!orange, dotted] 
	%	{2*(1-2*x)/(1+(1-2*x)^2)};
	%	\label{graph2k},
	%	\addplot[purple, line width=0.6,solid, domain=0.2:2.5] {2*(1-2.8*x)/(1+(1-2.8*x)^2)};
	%	\addplot 
	%	+[name path=rhomin5, line width=1pt, magenta, dotted] 
	%	{2*(1-4*x)/(1+(1-4*x)^2)};
	%	\label{graph5k},
	%	\addplot 
	%	+[name path=rhomin10, line width=1pt, violet, dashed] 
	%	{2*(1-8*x)/(1+(1-8*x)^2)};
	%	\label{graph10k},	
	\addplot[shade, top color=teal!15, bottom color=red!10] fill between[of=rhomax and rhominall];
	\addplot [black, solid, line width=0.5] 
	{0};
	\node[fill,circle,inner sep=1.2pt,label={[xshift=9pt, yshift=-6pt, darkgray]above: $(0,1)$}] at (axis cs: 0,1) {};
	\node[fill,circle,inner sep=1.2pt,label={[yshift=4pt,darkgray]below:$(2,-1)$}] at (axis cs: 2,-1) {};
	\node[fill,circle,inner sep=1.2pt,label={[xshift=12pt, yshift=4pt,darkgray]below:$(1,0)$}] at (axis cs: 1,0) {};
	%	\node (C) at (12.8, 7.0) {};
	%	\node (D) at (0.7, 7.0) {};
	%	\node (C) at (17.0, 2+3) {};
	%	\node (D) at (0.9, 2+3) {};
	\node (C) at (axis cs: 1.7, -0.9) {};
	\node (D) at (axis cs: 0.13, -0.9) {};
	% arrows
	\path[-{latex[scale=3.0]}, draw=gray]
	(C) edge node[sloped, anchor=right, above, xshift=0pt,yshift=-4pt, darkgray] {Increasing
		$\theta$
		%		\,$N^\frac{k-2}{2k}$
	}  
	(D);
	%	\node (A) at (32, 3.0+3) {};
	%	\node (B) at (16, 10+3) {};
	\node (A) at (axis cs: 2.2, -1.05) {};
	\node (B) at (axis cs: 1.7, -0.1) {};
	% arrows
	\path[-{latex[scale=3.0]},draw=darkgray]
	(A) edge node[sloped, anchor=left, below, xshift=0pt,yshift=4pt, darkgray] {Increasing
		$\theta$ 
		%		$N^\frac{k-2}{2k}$
	}  	(B);
	\node (E) at (axis cs: 2.2, 0.6) {};
	\node (F) at (axis cs: 0.95, 0.18) {};
	% arrows
		\path[-{latex[scale=3.0]}, draw=darkgray]
		(E) edge node[anchor=left, sloped, xshift=-3pt,yshift=-4pt,above, darkgray] {Increasing 
			$\theta$
			%		$\sqrt{N}$
		}  (F);
	%	\node[
	%	label={[darkgray]right
	%		:10.0}
	%	] at (31.2-2.8,9.8+3+2) {};
	%	\node[
	%	label={[darkgray]right
	%		:5.0}
	%	] at (34.2-2.5,10.3+3+2) {};
	%	%	\node[
	%	%	label={
	%	%		:3.0}
	%	%	] at (24,6+3) {};
	%	\node[
	%	label={[darkgray]right:2.0
	%		}
	%	] at (33.9-1.8,11.7+3+2) {};
	%	%	\node[
	%	%	label={
	%	%		:1.5}
	%	%	] at (29,3.6+3) {};
	%	\node[
	%	label={[darkgray]right
	%		:1.0}
	%	] at (35.6-2.8,13.4+3+2) {};
	%%%%%%%%%%%%%%%%%%
	\node[
	label={[violet]right
		:8.0}
	] at (axis cs: 1.7, -0.2) {};
	\node[
	label={[magenta]right
		:4.0}
	] at (axis cs: 1.73, -0.36) {};
	%	(21.6-2,7.0+3+1.5) {};
	%	\node[
	%	label={[gray]right
	%		:3.0}
	%	] at (24,6+3) {};
	\node[
	label={[purple]right
		:$2.\overline{6}$}
	] at (axis cs: 2.1, -0.41) {};
	\node[
	label={[green!70!purple]right
		:$2.\overline{6}$}
	] at (axis cs: 2.1, 0.33) {};
	\node[
	label={[black!15!orange]right
		:2.0}
	] at (axis cs: 1.93, -0.65) {};
	%	(26.4-2,4.8+3+2) {};
	%	\node[
	%	label={[gray]right
	%		:1.5}
	%	] at (29,3.6+3) {};
	\node[
	label={[red]right
		:1.0}
	] at (axis cs: 2.05,-0.87) {};
	%%%%%%%% P AND Q %%%%%%%%%%	
	%	\node[fill,cross,inner sep=1.2pt,teal,label={[teal,xshift=0pt, yshift=-3pt]:P}] at 
	%	(40,18+3)
	%{};
	%	\node[fill,cross,inner sep=1.2pt,red,label={[red,xshift=0pt, yshift=-3pt]:Q}] at 
	%	(35,1+0.3+3)
	%{};
	%%%%%%%%%%%%%%	
	\node[label=
	{
		[text width=5cm, align=center]
		\textit{
			\textbf{Operating region,}\,or\,Span $\mathcal{S'}$\\[0mm]
			Incorrect,\,but\,as\,per\,the\\[0mm] \Cref{eqn_generalmaxmin}
		}
	}
	] 
	%	at (120,9) 
	%at (47,-1.2+0.8)
	%at (axis cs: 3.3,-0.4)  
	at (axis cs: 3.55,-1.25)  
	{};  
	\draw[-{latex[scale=3.0]}] (axis cs: 4.3,-0.7)--(axis cs: 4.3,-0.4);
	%%%%%%%% RANGE ADDED BEGINS%	
	\draw[blue, line width=1.1pt] 
	[|<->|] 
	%(3,0.8+3)--(3,20.8+3); 
	%(7.5,0.8+3)--(7.5,19.9+3); 
	(axis cs: 0.75,-1) -- (axis cs: 0.75,0.8615) ;
	%%%%%%%%
	%
	\node[label={
		[xshift=0pt, yshift=0pt,rotate=0,
		align=left, blue]
		Range (Span $\mathcal{S'}_{MkE}$)\,of\,$\mathbf{\rho_c}$\,\\[0mm]for\,a\,given\,$\frac{N^\frac{2-k}{2k}\cdot MkE^\frac{1}{k} }{\sqrt {\sigma_G^2}}=$
		$\frac{L_k}{\sqrt{N}\sigma_G}
		(=0.75)
		$ 
	}] at 
	%(20+6-0.5,10.8+3+4-2) {}; 
	%15.8 =0.2 14=0,  bigY= smally*9+14
	(axis cs: 2.95,0.5) {};
	\node[fill,cross=3pt,blue,inner sep=1.2pt] at 
	%(3,14) {};
	%(7.5,11+3) {};
	(axis cs: 0.75,0) {};
	\draw[blue] 
	%(3,11+3) circle (2.5pt);
	%(7.5,11+3) circle (2.5pt);
	(axis cs: 0.75,0) circle (2.5pt);
	%%%%%%%%
%		\node[fill,cross=3pt,red,inner sep=1.2pt] at 
%	%(3,14) {};
%	%(7.5,11+3) {};
%	(axis cs: 0.125,0) {};
%	\draw[red] 
%	%(3,11+3) circle (2.5pt);
%	%(7.5,11+3) circle (2.5pt);
%	(axis cs: 0.125,0) circle (2.5pt);
	\node[fill,circle,inner sep=1.2pt,red,label={[xshift=0pt, yshift=3pt, red]below: $\Omega$}] at (axis cs: .125,0){};
	%%%%%%%%%% MSE Pointer2
	\draw[-{latex[scale=5.0]},blue] 	
	%(9.6,14+3).. controls (4,13+3) .. (3,11+3);
	%
	%(9.6+4+2,14+3+4+0.5).. controls (4+4,13+3+3) .. (3+4,11+3);
	%(9.6+4+2+0.5-1,14+3+4+0.5-2) to [bend right=40] (7.5,11+3);
	(axis cs: 1.2,0.85) to [bend right=20] (axis cs: 0.75,0);
	\draw[-{latex[scale=3.0]},blue] 	
	%(9.6,16.5+4) .. controls  (5,16.5+4) .. (3,16.5+4);	
	%
	%(9.6+4+2,16.5+4+4) .. controls  (7,16.5+4+4) .. (7,16.5+4+2);	
	%(9.6+4+2+0.5-1,16.5+4+4-2) to [bend right=20] (7.5,21.5);
	(axis cs: 1.2,1.05) to [bend right=20] (axis cs: 0.75,0.7);
	%%%%%%% RANGE ADDED ENDS
	\legend{$\rho_{c_{max'}}$ ,$\rho_{c_{min'}(\theta)}$ }
	% legend style={at={(1,1)},anchor=north east}
	% define coordinates at bottom left and top left of rectangle
	\coordinate (c1) at (axis cs:0.4,1.05);
	\coordinate (c2) at (axis cs:0.4,-1.1);
	% draw a rectangle
	\draw [gray] (c1) rectangle (axis cs:1,-1.1);
	\end{axis}
	\begin{axis}[name=ax2,
	at={($(ax1.south east)-(1.5cm,0)$)},
	xtick distance=0.5,
	ytick distance=0.5,
	xtick={0.5,1},
	domain=0:50,
	samples=551,
	smooth,
	no markers,
	scale=1.,
	xmin=0.4,xmax=1.0,
	%extra x ticks={1, 2},
	ymin=-1.1,ymax=1.05,
	unit vector ratio*=1.0 1.0,
	x label style={below,font=\large},
	y label style={left,font=\large},
	y axis line style= { draw =none },
	x axis line style={draw opacity=1},
	%	scaled ticks=false,
	%	xmin=0.4, xmax=1,
	%	ymin=-1,ymax=1,
	%	xlabel=In,
	%	ylabel=out,
	set layers,
	cell picture=true,
	]
	\addplot
	+[name path=rhomax2,
	domain=0:50, teal, line width=1] {2*(1+x)/(1+(1+x)^2)}; 
		\addplot[teal!50!purple, line width=1,loosely dashed] {2*(1+8*x)/(1+(1+8*x)^2)};
		\addplot[green!70!purple, line width=1,densely dashed] {2*(1+2.66666*x)/(1+(1+2.6666666*x)^2)};
	\addplot[domain=0:50, red, line width=1,dashed] {2*(1-x)/(1+(1-x)^2)};
	\addplot[ black!15!orange, line width=1,densely dashed] {2*(1-2*x)/(1+(1-2*x)^2)};
	\addplot[purple, line width=0.9,densely dotted] {2*(1-2.66667*x)/(1+(1-2.66667*x)^2)};
	\addplot[ magenta, line width=1,dotted] {2*(1-4*x)/(1+(1-4*x)^2)};
	\addplot[ violet, line width=1,loosely dashed] {2*(1-8*x)/(1+(1-8*x)^2)};
	\addplot+[domain=0:50, name path=rhominall2,line width=0.5, olive, solid]
	{func(x)};
	\addplot [black, solid, line width=0.5] 
	{0}; 
	\addplot [black, solid, line width=0.7] 
	{1.05}; 
	\addplot [black, solid, line width=0.7] 
	{-1.1}; 
	\addplot[shade, top color=teal!15, bottom color=red!10] fill between[of=rhomax2 and rhominall2];
	\draw[blue, line width=0.7] 
	[|<->|] 
	%(3,0.8+3)--(3,20.8+3); 
	%(45,39+10)--(45,230+10); 
	(axis cs: 0.75,-1)--(axis cs: 0.75,0.8615); 
	\node[fill,cross=2.5pt,blue,inner sep=1.2pt,label={[blue,xshift=-6pt, yshift=-3pt]right:
%		$\mathtt{x}\hspace{-2pt}=\hspace{-2pt}0.75$
	}] at 
	(axis cs: 0.75,0) {};
	\draw[blue] 
	%(3,11+3) circle (2.5pt);
	%(45,140+10) 
	(axis cs: 0.75,0)
	circle (2.5pt);
	%%%%%%%% ABCD %%%%%%%%%%	
	\node[fill,cross=2.5pt,inner sep=1.2pt,teal,label={[teal,xshift=0pt, yshift=0pt]left:A}] at 
	(axis cs: 0.75,0.8615)
	%	(45,225+10)
	{};
	\draw[teal] 
	(axis cs: 0.75,0.8615)
	%(45,225+10) 
	circle (2.5pt);
	\node[fill,cross=2pt,inner sep=1.2pt,red,label={[red,xshift=0pt, yshift=0pt]left:B}] at 
	%(45,190+10)
	(axis cs: 0.75,0.4706)
	{};
	\draw[red] 
	%(45,190+10)
	(axis cs: 0.75,0.4706)
	circle (2pt);
	\node[fill,cross=2pt,inner sep=1.2pt,violet,label={[violet,xshift=0pt, yshift=0pt]left:C}] at 
	%(45,105+10)
	(axis cs: 0.75,-0.3846)
	{};
	\draw[violet] 
	%(45,105+10) 
	(axis cs: 0.75,-0.3846)
	circle (2pt);
	\node[fill,cross=2.5pt,inner sep=1.2pt,purple,label={[purple,xshift=0pt, yshift=0pt]left:D}] at 
	%(45,40+10)
	(axis cs: 0.75,-1)
	{};
	\draw[purple] 
	%(45,40+10) 
	(axis cs: 0.75,-1)
	circle (2.5pt);
	\node[fill,cross=2.5pt,inner sep=1.2pt,green!70!purple] at 
	(axis cs: 0.75,0.6)
	{};
	\draw[green!70!purple] 
	(axis cs: 0.75,0.6)
	circle (2.5pt);
	%%%%%%%%%%%%%%
	\pgfplotsset{
		after end axis/.code={
			\node[fill,circle,inner sep=1.2pt,label={[xshift=6pt, yshift=-6pt, darkgray]}] at (axis cs: 1,0) {};
			\node[anchor=west,right,olive] at 
(axis cs: 0.8,-1.05){$\psi$};
			\node[anchor=west,right,teal] at 
			(axis cs: 0.8,0.88){$\Psi$};
			\node[anchor=west,right,teal] at 
			(axis cs: 1,0.85){$\theta=1$.};
			\node[anchor=west,right,red] at (axis cs: 1,0.45){$\theta=1$.};
			\node[anchor=west,right,green!70!purple] at (axis cs: 1,0.6){$\theta=2.\overline{6}.$};
			\node[anchor=west,right,teal!50!purple] at (axis cs: 1,0.2){$\theta=\theta_{max}$=8.};
			\node[anchor=west,right,blue] at (axis cs: .7,-0.08){$\mathtt{x}\hspace{-2pt}=\hspace{-2pt}0.75$};
			\node[anchor=west,right,violet] at (axis cs: 1,-0.3){$\theta=\theta_{max}$=8.};
			\node[anchor=west,right,magenta] at (axis cs: 1,-0.55){$\theta=4$.};
			\node[anchor=west,right,purple, align=right
			] at (axis cs: 1,-0.8){$\theta=\theta_0\hspace{-3pt}=\hspace{-3pt}\frac{2}{0.75}\hspace{-3pt}
				=\hspace{-3pt}2.\overline{6}.$};
			\node[anchor=west,right,purple, text width=3cm,font=\fontsize{6pt}{6pt}\selectfont, align=right] at (axis cs: 0.6,-0.95){
				(cf. \Cref{eqn_generalmaxmin}).
%				$\big(
%				%	\text{for\,}
%				\because
%%				L_k=0.75\sqrt{N}\sigma_{G},
%				\text{\cref{eqn_generalmaxmin}}\big).$
			};
			%%%%%%
			%\node[anchor=west,right,purple, text width=3cm,font=\fontsize{6pt}{6pt}\selectfont, align=right
			%] at (axis cs: 0.55,-0.76){$\theta=\theta_0
			%%	\in(2,4)
			%	=2.8.
			%	$
			%	\\[0mm]
			%%	(this example).
			%	$\big(\text{for\,}L_k=0.75\sqrt{N\sigma_{G}^2}\big).$
			%	};
			%%%%%%
			\node[anchor=west,right,black!15!orange] at (axis cs: 1,-1.05){$\theta=2$.};
%			\node[anchor=west,right,purple, xshift=0pt, yshift=0pt,text width=3.3cm, align=right] at (axis cs: 0.45,-1.2){$\theta=\theta_0:\exists\,\,\theta_0\in(1,\theta_{max})$\\[0mm] 
%				%	\hspace{50pt}
%				(in general).};
		}
	}	
	%%%%%%%%
	\end{axis}
	\draw [dashed, line width=1.2,gray] (c1) -- (ax2.north west);
	\draw [dashed, line width=1.2,gray] (c2) -- (ax2.south west);
	\end{tikzpicture}
	\vspace{-0.9cm}
	\caption{Range of $\rho_c$ as per \Cref{eqn_generalmaxmin}, 
		%		for a given $MkE$:= Mean k-Powered Error = $\frac{[L_k \text{norm of the error}]^k}{N}$, 
		as a function of given $MkE=\frac{L_k}{N}, k>2$ with respect to the gold standard consisting of 
		$N$ samples, standard deviation of $\sigma_G^2$, $1\leq\theta\leq {N}^{\frac{k-2}{2k}}$.
		Notice the increase in the operating region compared to \Cref{fig_maxminrho} to $\mathcal{S'}$, due to $\rho_{c_{min'}}=\psi(\theta\cdot\mathtt{x})$ with increasing $\theta$. 
		While each point in $\mathcal{S'}$ maps to a unique $\left\{L_k,\rho_c\right\}$ pair, each maps to infinitely many $\{L_2\}$ or $\{MSE\}$ values  -- except those points lying exactly on the $\rho_{{c}_{max'}}$ $\rho_{{c}_{min'}}$ curves which map to only one $MSE$. 
		Similar to \Cref{fig_maxminrho}, 
		%		Here as well,
		$MkE_1\leq MkE_2$ does not guarantee $\rho_{c_1}\leq \rho_{c_2}$. 
		%		\ie a reduction in $MkE$ or $MSE$ does not always translate to $\rho_c$ improvement. 
		For the sake of completeness, we note further that the true $\rho_{{c}_{max}}$ and consequently, the true span $\mathcal{S}$ is even smaller than the one shown above (cf. \Cref{sec_truespan}).
		%		The true span $\mathcal{S}$ is even more than the one shown in the figure above (cf. \Cref{sec_truespan}).
		%	More the number of samples ($N$), more likely the disorder in the error coefficients. Consequently, more is the likelihood that predictions are not well-correlated with the gold standard, more is the likelihood of getting low value of $\rho_c$, even for an identical $L_k$ error.
		\label{fig_maxminrhoLp}}
\end{figure}

\begin{remark}
	Thus, for a two dimensional space $\mathbb{R}^2:=(\mathtt{X},\mathtt{Y})$, where $\mathtt{x}=\left|\frac{L_k}{\sqrt{N}\sigma_G}\right|$, $\mathtt{y}=\rho_c$ (cf. \Cref{fig_maxminrhoLp}), 
	\begin{itemize}
		\item
		\newcommand{\MSEsmallenough}{That is, for $MSE\leq\sigma_{G}^2$ (cf. $\sqrt{\sfrac{MSE}{\sigma_{G}^2}}\in[0,1]$ in \Cref{fig_maxminrho})}
%		While not in terms of guaranteed $\rho_c$ `maximisation' (cf. \Cref{sec_m2mg}), a minimisation of $MSE$ 
		While reducing $MSE$ does not always guarantee $\rho_c$ improvement (cf. \Cref{sec_m2mg}),
		$MSE$ reduction through enough number of iterations
		to a small enough value \footnote{\MSEsmallenough} guarantees a positive $\rho_c$ nonetheless. However, 
%		a similar strategy of $MkE$ minimisation for $\rho_c$ maximisation would fail even more spectacularly. In this case,
		a negative $\rho_c$ remains a possibility for $L_k$ as small as $\frac{\sqrt{N}\sigma_{G}}{N^{\frac{k-2}{2k}}}$ with $\theta=\theta_{max}$ (cf. \Cref{fig_maxminrhoLp}). While $N$ typically is very large, even for $N$ as small as 64 (\ie $\sqrt{N}=8$), this translates to the 
		rapid deterioration of $\rho_{c_{min'}}$ from point $(0,1)$ to point $\Omega$ \ie $(0.125,0)$ -- as opposed to $(1,0)$. 
		\item
		Higher the $N$, more diminished is the effect of reduction of $MkE$ in terms of improving $\rho_c$,
		%		region to the left of point $(\frac{1}{\sqrt{N}},0)=(0.125,0)$ instead of that of point $(1,0)$
		%Increase in $N$ entails 
		%a more permissible increase in $\theta$, \ie a more permissible increase in $MSE$.
%		\item
%		Higher the $N$, the 
		\ie quicker the deterioration of $\rho_c$ from $\rho_{c_{min}}=1$, \ie point $(0,1)$ to  $\rho_{c_{min}}=-1$. 
		\item
		Increase in $N$ implies an increase in the range of valid $\theta$ values, \ie effectively an increase in the range of possible $MSE$ values for any given $L_k$-norm.
		\item
		Increase in the range for $\theta$ manifests into an increase in the operating region with a negative $\rho_c$, which asymptotically approaches the negative $\mathtt{Y}$-axis.
		%, owing to the one-to-many mapping between $L_1$ and $L_2$ 
		%(cf. \Cref{eqn_maemse} and \Cref{fig_maxminrhomae}). 
		\item 
		Consistent to findings from \Cref{sec_m2mg,fig_maxminrho}, the reduction in $MSE$ at a constant $L_k$ norm (even if assumed to be achievable to its lowest limit for the given gold standard) does not always translate to improvement in $\rho_c$.
\begin{description}
	\item
	For example, starting with $\theta=\theta_{max}=8$ at a given ${\frac{L_k}{\sqrt{N}\sigma_{G}}}(=0.75)$, the $MSE$ reduction (\ie effectively the $\theta$ reduction) causes further deterioration of $\rho_{{c}_{min'}}$ to $-1$, \ie when $\theta$ becomes equal to 
$\frac{2}{\mathtt{x}}=\frac{2}{0.75}=2.\overline{6}$ (the trajectory $C \rightarrow D$). While further reduction in $MSE$ does begin to make $\rho_{{c}_{min'}}$ more positive, the $\rho_{{c}_{min'}}$ remains overall negative until $\theta$ is reduced to $\frac{\theta_0}{2}=\frac{1}{\mathtt{x}}=\frac{1}{0.75}=1.\overline{3}$. 
%inspecting the variation of the plots with the value of $\theta$. 
In summary, $MSE$ reduction at a given $L_p$ norm does not necessarily result in a $\rho_c$ improvement.
\end{description}
\newcommand{\explainrel}{
\hspace{2.1cm}
	Let $\rho_c=\mathtt{y}=\psi(\theta_1\mathtt{x})=\psi(\theta_2\mathtt{x})
	\hspace{18pt}
	\implies
	\hspace{10pt}
	\frac{2(1-\theta_1\mathtt{x})}{1+(1-\theta_1\mathtt{x})^2}
	\quad=\quad
	\frac{2(1-\theta_2\mathtt{x})}{1+(1-\theta_2\mathtt{x})^2}
	$, where $\mathtt{x}=\frac{L_k}{\sqrt{N}\sigma_{G}}$.
	\begin{align*}
	%\text{Let } \rho_c=\mathtt{y}=\psi(\theta_1\mathtt{x})=\psi(\theta_2\mathtt{x})
	%\\
	%\therefore
	%\implies
	%\frac{2(1-\theta_1\mathtt{x})}{1+(1-\theta_1\mathtt{x})^2}
	%=
	%\frac{2(1-\theta_2\mathtt{x})}{1+(1-\theta_2\mathtt{x})^2}
	%\\
	%\therefore
	%%\implies
	%(1-\theta_1\mathtt{x})(1+(1-\theta_2\mathtt{x})^2)=(1-\theta_2\mathtt{x})(1+(1-\theta_1\mathtt{x})^2)
	%\\
	&\therefore&
	(1-\theta_1\mathtt{x})+(1-\theta_1\mathtt{x})(1-\theta_2\mathtt{x})^2
	\quad
	&
	=
	\quad
	(1-\theta_2\mathtt{x})+(1-\theta_2\mathtt{x})(1-\theta_1\mathtt{x})^2.
	\\
	&\therefore&
	(1-\theta_1\mathtt{x})(1-\theta_2\mathtt{x})(
	%\cancel{
	(1-\theta_2\mathtt{x})-(1-\theta_1\mathtt{x})
	%}
	)
	\quad
	&
	=
	\quad
	%\cancel{
	(1-\theta_2\mathtt{x})-(1-\theta_1\mathtt{x}).
	%}
	\\
	&\therefore&
	(1-\theta_1\mathtt{x})(1-\theta_2\mathtt{x})
	=1
	\implies
	\hspace{30pt}
	\theta_1\theta_2\mathtt{x}
	\quad
	&
	=
	\quad
	\theta_1+\theta_2.
	\\
	&\therefore&
	\theta_2 
	\quad
	&
	=
	\quad
	\frac{\theta_1}{\mathtt{x}\theta_1-1}, \text{ where } \mathtt{x}=\frac{L_k}{\sqrt{N}\sigma_{G}}.
	\end{align*}
}
\item
While each point in $\mathcal{S'}$ maps to a unique $\left\{L_k,\rho_c\right\}$ pair, each maps to infinite number of $MSE$ values (except the points on the $\rho_{{c}_{max'}}$ and $\rho_{{c}_{min'}}$ curves which map to only one $MSE$, \eg the points $A$ and $D$ in \Cref{fig_maxminrhoLp}). This is because, to be in $\mathcal{S'}$, every $\left\{\frac{L_k}{\sqrt{N}\sigma_{G}},\rho_c\right\}\bigg\rvert_{MSE= \frac{\theta\cdot L_k}{\sqrt{N}\sigma_{G}}}$ needs to only satisfy the condition $\rho_c\in\left[\psi\left(\frac{\theta\cdot L_k}{\sqrt{N}\sigma_{G}}\right), \Psi\left(\frac{L_k}{\sqrt{N}\sigma_{G}}\right)\right]$ such that $\theta\in[1,\theta_{max}]$, which is true for infinite number of $\theta$ (except the points on the $\rho_{{c}_{max'}}$ and $\rho_{c_{min'}(\theta)}$ plots, where $\theta=1$ and $\theta=\theta_0$ only respectively, \ie no variation in $\theta$ is allowed). 
\begin{description}
	\item
For example, for every point on the segment $AD$ in \Cref{fig_maxminrhoLp}, $\frac{L_k}{\sqrt{N}\sigma_{G}}=0.75$. Thus, for every point on the segment $AD$ where $\rho_c<0$ except the point $D$, $\psi(\theta\times0.75)=\rho_{c_{given}}$ has two solutions $\mathtt{\theta}_1$ and $\mathtt{\theta}_2$ that are  
related by an equation \footnote{\explainrel}: $\theta_2=\frac{\theta_1}{\frac{L_k}{\sqrt{N}\sigma_{G}}\theta_1-1}=\frac{\theta_1}{0.75\theta_1-1}$. As an example, for the point $C$ which is known to simultaneously lie on the curve $\phi(\theta_{max}\times\mathtt{x})=\phi(8\times\mathtt{x})$, one of the two $\theta$ values is 8, and thus, $\left\{\theta_1,\theta_2\right\}= \left\{\theta_1,\frac{\theta_1}{0.75\theta_1-1}\right\}=\left\{8.1.6\right\}$. Thus, $\rho_c\rvert_C\in\left[\psi\left(\frac{\theta_0\cdot L_k}{\sqrt{N}\sigma_{G}}\right), \Psi\left(\frac{L_k}{\sqrt{N}\sigma_{G}}\right)\right]$, \ie 
$\psi\left(\frac{\theta_0\cdot L_k}{\sqrt{N}\sigma_{G}}\right)
\leq
\rho_c\rvert_C
\leq
 \Psi\left(\frac{L_k}{\sqrt{N}\sigma_{G}}\right)$
is true for every value of $\theta_{0} \in [\theta_1,\theta_2]=[1.6,8]$. That is, while point $C$ maps to $\left\{\frac{L_k}{\sqrt{N}\sigma_{G}}=0.75, \rho_c=\phi(0.75\times8)=-0.3846\right\}$, it can map to $\sqrt{MSE}\in[\quad0.75\sigma_{G}\times1.6, \quad 0.75\sigma_{G}\times8\quad]$, \ie $1.44\sigma_{G}^2\leq MSE\leq 36\sigma_{G}^2$.
\end{description}
	\end{itemize}
Identical observations can be made for $k<2$ as well, the only difference being $\mathtt{x}={L_k}{N^{\frac{-1}{k}}}$ (cf. \Cref{sec_app_mae}).
\end{remark}

\section{True span of valid $\left\{L_p, \rho_{c}\right\}$ pairs, given a gold standard sequence}
\label{sec_truespan}
%%%%%%%
%Rough work to know the MSE optimisation conditions
% When all errors are equal = e
% L_p = (N*e^p)^(1/p) = e*N^(1/p)
% L_2 = (N*e^2)^(1/2) = e*N^(1/2)
% MSE = e^2, MSE^(0.5) = e
% MkE = e^k, MkE^(1/k) = e
% i.e. MSE maximisation for k>2,
%      MSE minimisation for k<2
% When all but one error = e, the rest = 0.
% L_p = (e^p)^(1/p) = e
% L_2 = e
% MSE = e^2 /N, MSE^(0.5) = e* N^(-1/2)
% MkE = e^k /N, MkE^(1/k) = e* N^(-1/k)
% i.e. MSE minimisation for k>2, 
%      MSE maximisation for k<2
%%%%%%%

In \Cref{sec_lp}, we noted that the \Cref{eqn_generalmaxmin} does not represent the true span of $\rho_c$. While the true $\rho_{{c}_{max}}\leq\rho_{{c}_{max'}}$ always, the true $\rho_{c_{min}}\geq \rho_{c_{min'}}$ need not necessarily be true.
This is because $\rho_{c_{max'}}\rvert_{L_{k_{given}}}$ not only assumes that $MSE=MSE_{min}$, but also that the error coefficients and the corresponding deviations of the gold standard from the mean are in equal ratio (cf. \Cref{thm41}); the two mutually contradictory assumptions in general. 
However, the $\rho_{c_{min'}}(\theta)\rvert_{L_{k_{given}}}$  is not conditioned upon $MSE=MSE_{min}$ or $MSE=MSE_{max}$, but rather at some $\theta_0\times MSE_{min}: \exists \theta_0\in[1,\theta_{max}]$. Thus, the error coefficients in this case are not constrained to simultaneous assumptions that can be mutually contradictory. 

\newcommand{\mkemseboundaries}{
	$
	\begin{aligned}
		\text{When }& e_j=0, & \exists!i, \forall j\neq i, i,j\in[1,N], & i,j\in\mathbb{N} &\implies 
	MSE={\frac{e_i^2}{N}}\quad 	{MkE}=\frac{|e_i|^k}{N}
	%	{MAE}=\frac{e_i}{N}=\frac{\sqrt{MSE}}{\sqrt{N}}
	%	\sqrt{N\cdot MSE}=e_i={N\cdot MAE} 
	&\implies\mathop{\theta}_{k<2}:=\frac{\sqrt{MSE}}{MkE^{\frac{1}{k}}}=\frac{N^{\frac{1}{k}}}{N^{\frac{1}{2}}}=N^{\frac{2-k}{2k}}.\nonumber
	\\
	&&&&&\implies\mathop{\theta}_{k>2}:=\frac{\sqrt{MSE}}{N^{\frac{2-k}{2k}}\cdot MkE^{\frac{1}{k}}}=1.\nonumber
	\\
	\text{When }& e_i=\pm e_j, &\forall i,j\in[1,N], & i,j\in\mathbb{N} &\implies 
	%	{MAE}=e_i=\sqrt{MSE} 
	MSE=e_i^2, \quad {MkE}=|e_i|^k
	&\implies\mathop{\theta}_{k<2}:=\frac{\sqrt{MSE}}{MkE^{\frac{1}{k}}}
	\phantom{=\frac{N^{\frac{1}{k}}}{N^{\frac{1}{2}}}}
	=1.\nonumber
	\\
	&&&&&\implies\mathop{\theta}_{k>2}:=\frac{\sqrt{MSE}}{N^{\frac{2-k}{2k}}\cdot MkE^{\frac{1}{k}}}=N^{\frac{k-2}{2k}}.
	\nonumber
	\end{aligned}
	$
}
%\subsection{$L_p$ boundary conditions}
Specifically, for $k>2$, $MSE\rvert_{L_{k_{given}}}$ is minimised to $MSE_{min}$ \emph{only when \footnote{\mkemseboundaries}} the error-set $E$ consists of all zeros except one $e_i=\pm N^\frac{1}{k}\cdot MkE^\frac{1}{k}, \exists i:1\leq i\leq N$ -- thus, $2N$ possible sequences. Likewise, $MSE\rvert_{L_{k_{given}}}$ is maximised (\ie $MSE_{max}$) \emph{only when} the $L_p$ norm is divided into errors of equal magnitudes, \ie $e_i=\pm MkE^\frac{1}{k}, \forall i:1\leq i\leq N$ -- thus, $2^N$ possible sequences. For $k<2$, $MSE\rvert_{L_k}$ is minimised and maximised at identically opposite conditions. 
%
%Thus, the argument that the two extremities $MSE_{min}$ and $MSE_{max}$ correspond to the true $\rho_{c_{min}}$ and the true $\rho_{c_{max}}$ (\eg points A and D in \Cref{fig_mke_mse}) makes a false assumption about the gold standard sequence itself, since 
Simultaneously, $\rho_c$ optimisation in either direction requires that the magnitudes of the errors are directly proportional to the corresponding deviations of the gold standard sequence instances from the gold standard mean 
(cf. \Cref{thm41,thm42}). These simultaneous conditions dictate that the gold standard is constant valued -- either entirely, or except at one instance -- which is not true for a gold standard as a general case. Thus, unless these conditions are satisfied for a given gold standard, the points $A$ and $C$ in \Cref{fig_maxminrhoLp}  map to very different values of $L_k$; \ie the two points are not attainable at a given $L_k$. Thus, the allowed scope of the trajectory of $MSE$ minimisation at a given $L_p$ norm is quite restricted compared to $\mathcal{S'}$. The restriction imposed (\ie the true span $\mathcal{S}$) is evidently a function of the given gold standard sequence, since that sequence alone dictates the optimised distribution and sequence of errors. 
\begin{remark}
%%% BEGINS NEW MODIFY TRUE REGION for 1.5 to 8 %%%%%
\begin{figure}[!t]
	%	\vspace{-2cm}
	\centering
	\begin{tikzpicture}
		[every node/.style={scale=1.2},]
	\begin{axis}[
	%	width=0.5\textwidth,
	scale=1.99,
	domain=0:50,
	samples=151,
	smooth,
	no markers,
	xlabel={$\sqrt{\frac{MSE}{\sigma_G^2}}
%		\quad
%		\Big(
		=\theta
		\cdot
		%		\sqrt{\frac{MSE_{min}}{\sigma_G^2}}\bigg\rvert_{L_k=C_0}
		{\frac{\sqrt{MSE_{min}}\rvert_{L_k=c_0}}{\sigma_G}}
		\quad
		(
		:= \mathtt{x}
		:= \theta\cdot\mathtt{x'}
		)
%		\Big)
		\rightarrow$    },
	ylabel={$\rho_c
		\quad
		(:= \mathtt{y})
		\rightarrow$},
	xmin=0,xmax=9.5,
	%extra x ticks={1, 2},
	ymin=-1.2,ymax=1.1,
	unit vector ratio*=0.5 0.95,
	x label style={below,font=\large},
	y label style={left,font=\large},
	legend columns=4, 
	%	minor x tick num=9,
	%	minor y tick num=9,
	%	grid=both,
	legend style={
		at={(axis cs: 4.5,1.4)},
		anchor=north,
		/tikz/column 2/.style={
			column sep=2pt,
		},
		font=\fontsize{10}{5}\selectfont},
	]
	\addplot 
	+[name path=rhomax, line width=1.1pt,dashed,black!30!green] 
	{2*(1+x)/(1+(1+x)^2)}; 
	\label{graph1tf}
	\addplot 
	+[name path=rhomin,line width=1.1pt,dashed, black!15!orange] 
	{2*(1-x)/(1+(1-x)^2)};
	\label{graph2tf},
	%%%%%%% FAKE PLOTS BEYOND RANGE (for the legend)
	\addplot
	[line width=1.1,draw=teal][domain=16:18] 
	{2*(1+x)/(1+(1+x)^2)}; 
	\addplot
	[line width=1.1,draw=red][domain=16:18] 
	{2*(1-x)/(1+(1-x)^2)}; 
	%%%%% rhocmax' and rhocmin' text labels
	\node[inner sep=1.2pt,label={[black!30!green,xshift=0pt, yshift=0pt, text width=1.5cm, align=center,rotate=-3]left:		
		$\Psi(\theta\cdot\mathtt{x'})$
		%		$\rho_{{c}_{max'}}(\theta)$
	}	] at (axis cs:8.0,0.3) {};
	\node[inner sep=1.2pt,label={[text width=3cm,black!15!orange, align=center,rotate=6]right:
		$\psi(\theta\cdot\mathtt{x'})
		%	=\rho_{{c}_{min'}}(\theta)
		$
	}] at (axis cs:4.52,-0.55) {};
	%%%%
	%%%%% MSE-based Span Colouring and bordering
	\addplot[shade, 
	%	top color=green!20, bottom color=orange!40
	top color=teal!15, bottom color=red!10
	] fill between[of=rhomax and rhomin,,soft clip={domain=1.5:8}];
	%%%%% y=O line 
	\addplot [black, solid, line width=0.5] 
	{0};
	%%%%% MSE min max boundary annotations
	\node[inner sep=1.2pt,label={[blue,xshift=0pt, yshift=0pt, text width=1.5cm, align=center]left:		$\sqrt{\frac{MSE_{min}}{\sigma_{G}^2}}$ \\[0mm]($\theta=1$)	}	] at (axis cs:1.6,0) {};
	\node[inner sep=1.2pt,label={[text width=1.55cm,blue, align=center]right:	$\sqrt{\frac{MSE_{max}}{\sigma_{G}^2}}$ \\[0mm]($\theta=\theta_{max}$) }] at (axis cs:7.9,0) {};
	%%%%
	%	\node[fill,cross,inner sep=1.2pt,label={[
	%%		xshift=-10.5pt, 
	%		yshift=-3pt
	%		]:$\mathbf{\rho_{c_{min'}}}$}] at 
	%	%	(80,8) 
	%	(80,9.5)
	%	{};
	%	\node[inner sep=1.2pt,label={[
	%		xshift=22pt, yshift=-3pt
	%		]:$\mathbf{MSE_{max}}$}] at (80,10.95) {};
	%%%%%%%% ABCD STARTED %%%%%%%%%%	
	\coordinate (coA) at (axis cs: 1.5,0.6897);
	\coordinate (coB) at (axis cs: 1.5,-0.8);
	\coordinate (coC) at (axis cs: 8,-0.28);
	\coordinate (coD) at (axis cs: 2,-1);
	\coordinate (coE) at (axis cs: 8,0.2195);
	\draw [line width=0.9,blue,dotted] (coA) -- (coB);
	\draw [line width=0.9,blue,dotted] (coC) -- (coE);
	\node[fill,cross=2pt,teal, inner sep=1.2pt,label={[xshift=0pt, yshift=-3pt,align=center,teal,text width=1.45cm,align=right]left:A$\rvert_{\theta=1}$\\[0mm]
		%		$\left(\mathbf{\rho_{c_{max'}}}(1)\right)$
		$\mathbf{\rho_c=\Psi(\mathtt{x'})}$
		\\[0mm] $\mathbf{=\rho_{{c}_{max'}}
		}$
	}] at (coA) {};
	\node[fill,cross,red, inner sep=1.2pt,label={[xshift=0pt, yshift=0pt,red,text width=1.45cm,align=right]left:B$\rvert_{\theta=1}$
		\\[0mm]
		$\mathbf{\rho_c=\psi(\mathtt{x'})}$
		%		$\left(\mathbf{\rho_{c_{min'}}}(1)\right)$
	}] at (coB) {};
	\node[fill,cross,violet, inner sep=1.2pt,label={[xshift=0pt, yshift=0pt,align=center, violet,text width=1.45cm,align=right]below:C$\rvert_{\theta=\theta_{max}}$\\[0mm]
		$	\mathbf{
			\rho_c=
			\psi(\theta_{max} \cdot \mathtt{x'})
		}
		$
		%		$\left(\mathbf{\rho_{c_{min'}}}(\theta_{max})\right)$
	}] at (coC) {};
	\node[fill,cross,inner sep=1.2pt,purple,label={[xshift=18pt, yshift=4pt,purple]below:D$\rvert_{\theta=\theta_0},$\,		
		$\mathbf{
			\rho_c=
			\psi(\theta_{0}\cdot \mathtt{x'})
			=\rho_{{c}_{min'}}=-1
		}
		$
		%		$\left(\mathbf{\rho_{c_{min'}}}(\theta_{0})\right)$
	}] at (coD) {};
	\node[fill,circle,inner sep=1.2pt,label={[xshift=0pt, yshift=0pt]right:$(0,1)$}] at (axis cs: 0,1) {};
	%%%%%
	\draw[teal] 	(coA) circle (1.5pt);
	\draw[red] 		(coB) circle (1.5pt);
	\draw[violet] 	(coC) circle (1.5pt);
	\draw[purple] 	(coD) circle (1.5pt);
	%%%%%%%% ABCD ENDED %%%%%%%%%%	
	%%%%%%%% REAL SPAN %%%%%%%%%%%%%%%
	%%%% rhospan 1
	\draw [name path=rhomaxreal1,dotted, black]
	(axis cs: 1.5, 0.26) 
	to [bend left=20]  
	(axis cs: 2.5, 0.39); 
	\draw [name path=rhominreal1,dotted, black]
	(axis cs: 1.5, 0.26) 
	to [bend right=10]  
	(axis cs: 2.0, 0.25) 
	to [bend right=10]  
	(axis cs: 1.8, 0.1) 
	to [bend right=10]  
	(axis cs: 2.5, 0.39) ;
	%%%% rhospan 2
	\draw [name path=rhomaxreal2,dotted, black!50!cyan]
	(axis cs: 1.5, 0.1) 
	to [bend right=20]  
	(axis cs: 2.0, 0.39)
	to [bend  left=40]  
	(axis cs: 3.0, 0.39)
	to [bend right=110]  
	(axis cs: 4.25, -0.26)
	to [bend right=10]  
	(axis cs: 5.5, 0.13)
	to [bend left=30]  
	(axis cs: 6.5, 0.13)
	to [bend right=20]  
	(axis cs: 7.25, 0.13)
	to [bend left=20]  
	(axis cs: 8.0, 0.13); 
	\draw [name path=rhominreal2,dotted, black!50!cyan]
	(axis cs: 1.5, 0.1) 
	to [bend right=20]  
	(axis cs: 2.25, -0.26)
	to [bend  right=10]  
	(axis cs: 3.0, -0.13)
	to [bend left=10]  
	(axis cs: 3.75, -0.39)
	to [bend right=10]  
	(axis cs: 4.25, -0.52)
	to [bend right=30]  
	(axis cs: 5.75, 0.0)
	to [bend left=120]  
	(axis cs: 6.5, -0.13)
	to [bend right=120]  
	(axis cs: 7.0, -0.13)
	to [bend left=120]  
	(axis cs: 7.5, 0)
	to [bend right=10]  
	(axis cs: 8.0, 0.13); 
	%%%% rhospan 3
	\draw [name path=rhomaxreal3,dotted, brown]
	(axis cs: 1.5, -0.52) 
	to [bend left=30]  
	(axis cs: 4.0, 0.35)
	to [bend  left=10]  
	(axis cs: 5.0, 0.13)
	to [bend right=20]  
	(axis cs: 7.0, 0.13)
	to [bend left=20]  
	(axis cs: 8.0, -0.13);
	\draw [name path=rhominreal3,dotted, brown]
	(axis cs: 1.5, -0.52) 
	to [bend right=30]  
	(axis cs: 2.25, -0.91)
	to [bend  right=50]  
	(axis cs: 3.25, -0.39)
	to [bend left=90]  
	(axis cs: 4.0, -0.13)
	to [bend right=120]  
	(axis cs: 5.5, -0.26)
	to [bend left=120]  
	(axis cs: 6.25, -0.13)
	to [bend right=120]  
	(axis cs: 6.75, 0.0)
	to [bend left=120]  
	(axis cs: 7.5, -0.13)
	to [bend right=30]  
	(axis cs: 8.0, -0.13); 
	%%%% teal!25 red!15
	\addplot[shade, top color=yellow!15, bottom color=yellow!15] fill between[of=rhomaxreal1 and rhominreal1];
	\addplot[shade, top color=yellow!15, bottom color=yellow!15] fill between[of=rhomaxreal2 and rhominreal2];
	\addplot[shade, top color=yellow!15, bottom color=yellow!15] fill between[of=rhomaxreal3 and rhominreal3];
	%% Overall rhomax
	\draw [name path=rhomaxrealO,solid,line width=1.1, teal]
	(axis cs: 1.5, 0.26) 
	to [bend left=10]  
	(axis cs: 2.0, 0.385)
	to [bend  left=40]  
	(axis cs: 3.0, 0.385)
	to [bend right=40]  
	(axis cs: 2.88, 0.22)
	to [bend  left=10]  
	(axis cs: 4.0, 0.35)
	to [bend left=10] 
	(axis cs: 5.0, 0.13) 
	to [bend  left=0]  
	(axis cs: 5.37, 0.077)
	to [bend  left=33]  
	(axis cs: 6.55, 0.12)
	to [bend  right=20]  
	(axis cs: 6.9, 0.12)
	to [bend  left=15]  
	(axis cs: 7.17, 0.123)
	to [bend left=20]  
	(axis cs: 8.0, 0.13); 
	%% Overall rhomin
	\draw [name path=rhominrealO,solid,line width=1.1,red]
	(axis cs: 1.5, -0.52) 
	to [bend right=30]  
	(axis cs: 2.25, -0.91)
	to [bend  right=50]  
	(axis cs: 3.25, -0.39)
	to [bend left=30]  
	(axis cs: 3.34, -0.215)
	to [bend left=0]  
	(axis cs: 3.75, -0.39)
	to [bend right=10]  
	(axis cs: 4.25, -0.52)
	to [bend right=10]  
	(axis cs: 5.04, -0.415)
	to [bend right=12]  
	(axis cs: 5.4, -0.39)
	to [bend right=30]  
	(axis cs: 5.5, -0.265)
	to [bend left=120]  
	(axis cs: 6.25, -0.13)
	to [bend right=55]  
	(axis cs: 6.48, -0.155)
	to [bend right=37]  
	(axis cs: 7.0, -0.155)
	to [bend right=40]  
	(axis cs: 6.98, -0.115)
	to [bend left=40]  
	(axis cs: 7.25, 0.01)
	to [bend left=30]  
	(axis cs: 7.52, -0.08)
	to [bend left=30]  
	(axis cs: 7.5, -0.12)
	to [bend right=30]  
	(axis cs: 8.0, -0.13);
	%%%%%%%% REAL SPAN ENDED %%%%%%%%%%%%%%%
	%%%%%%% Annotated real rho_c optimised
	\node[fill,cross,inner sep=1.2pt,label={ [xshift=0pt, yshift=0pt] above:$\mathbf{\rho_{c_{min}}}$}] at (axis cs: 2.25,-0.91) {};
	\node[fill,cross,inner sep=1.2pt,label={ [xshift=0pt, yshift=0pt] below:$\mathbf{\rho_{c_{max}}}$}] at (axis cs: 2.5,0.48) {};
	%%%%%%% Span annotation addition
	\node[label=
	{[text width=7cm,align=center]
		\textbf{
			\textit{
				Operating region $\mathcal{S}^*$ 
				%				| Given $\mathbf{L_p=c_0,p\neq2}$
				\\[0mm]
			}
		}
		that was
		implicitly assumed\\[0mm] while deriving 
		\Cref{eqn_generalmaxmin}
%		\Cref{eqnrhocmaxae,eqnrhocminae,eqnrhocmaxmink,eqnrhocmaxmink2}
		%	\\[0mm]
		%	(disregarding the gold standard)
	}
	] 
	at (axis cs: 7.,-1.2)
	{};
	\draw[-{latex[scale=3.0]}] 
	(axis cs: 5.8,-0.72) -- 	(axis cs: 3.5,-0.5); 
	%	(48.5,3)--(35,5.25);
	%
	\node[label=
	{
		[text width=7cm, align=center]
		\textbf{
			\textit{
				True operating region $\mathcal{S}\hspace{6pt}|$Given $\mathbf{L_p=c_0,p\neq2}$\\[0mm]
			}
		}
		(taking into account the gold standard)
	}
	] 
	at (axis cs: 4.5,0.55) 
	{};
	\draw[-{latex[scale=3.0]}] 
	(axis cs: 4.5,0.65) -- 	(axis cs: 4,0.2); 
	%	(48,17)--(39,12);
	%%%%%%%%%%
	\node[label=
	{
		%	[text width=6cm]
		%	\textbf{
		%		\textit{
		$L_p\neq c_0$ 
		%		}
		%	}
	}
	] 
	at (axis cs: 5.5,0.35) 
	{};
	\draw[-{latex[scale=3.0]}] (axis cs: 5.2,0.45)--(axis cs: 5,0.2);
	\legend{$\rho_{c_{max}}(MSE)$ ,$\rho_{c_{min}}(MSE)$,
		$\rho_{c_{max}}(MSE)\rvert_{L_p=c_0}$, $\rho_{c_{min}}(MSE)\rvert_{L_p=c_0}$ 
	}
	\end{axis}
	\end{tikzpicture}
	%	\vspace{-1.5cm}
	\caption{Comparison between the true $\rho_{c_{min}}$ and $\rho_{c_{max}}$ and those obtained with $MSE$ optimisation ($\rho_{c_{max'}}$ and $\rho_{c_{min'}}$ using \Cref{eqn_generalmaxmin}) for a given $L_p=c_0$.  For $k>2$, $MSE_{min}^{\frac{1}{2}}= N^{\frac{2-k}{2k}}\cdot MkE^\frac{1}{k}$ and $MSE_{max}^{\frac{1}{2}}=MkE^\frac{1}{k}$. 
		For $0<k<2$, the equations for  $MSE_{min}$ and $MSE_{max}$ merely get exchanged, but the illustration remains the same essentially.
		Because of the many-to-many mapping between $L_2$ and $L_p$, the points in true operating region $\mathcal{S}$ map to other $L_p$-norm values as well, and not just $L_p=c_0$. The $\neg \mathcal{S}\cap \mathcal{S^*}$ region maps to  $L_p\neq c_0$. The points $A,B,C,D$ refer to the points $A,B,C,D$ from \Cref{fig_maxminrhoLp}, except that here:  $\frac{{MSE_{min}}^{0.5}}{\sigma_G}=\frac{L_k}{\sqrt{N}\sigma_G}=1.5$ and $\theta_{max}=\frac{8}{1.5}=5.\overline{3}$ (instead of 0.75 and 8 respectively).
		%	Notice that always $\rho_{c_{max}}\geq\rho_{c_{max'}}$, irrespective of value of $MSE_{min}$. $\rho_{c_{min}}\text{ can be }\leq\rho_{c_{min'}}$ depending on the span of the true operating region
		\label{fig_mke_mse}}
\end{figure}
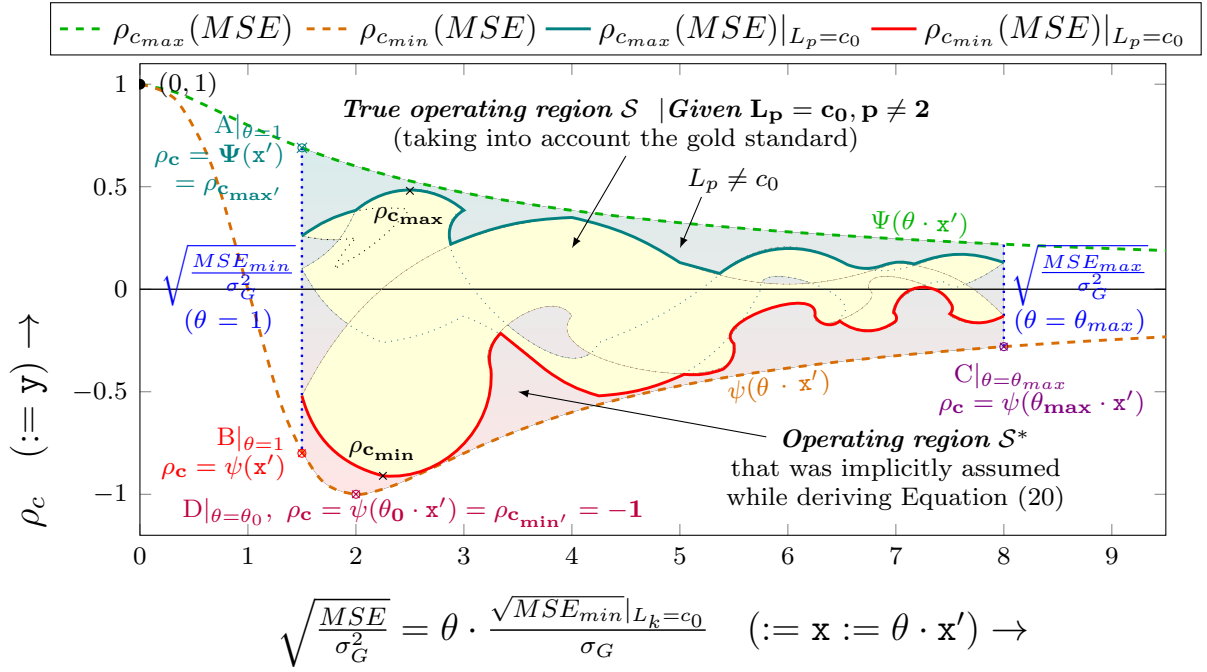
%%% ENDS NEW MODIFY TRUE REGION for 1.5 to 8 %%%%%

	Thus, for a two dimensional space $\mathbb{R}^2:=(\mathtt{X},\mathtt{Y})$, where $\mathtt{x}=\left|\frac{MSE}{\sigma_G}\right|$, $\mathtt{y}=\rho_c$ (cf. \Cref{fig_mke_mse}), 

\begin{itemize}
\item
There can be at the most  $2N$ and $2^N$ $\{MSE, \rho_c\}$ pairs mapped $\in \mathcal{S}$ at the two extremums (cf. the two blue lines in \Cref{fig_mke_mse}). The example \emph{true} operating region in \Cref{fig_mke_mse}, therefore, does not converge to a unique point at either $MSE$ extremum, nor can all of the infinitely many points lying on the two blue lines be part of $\mathcal{S}$ (cf. \Cref{fig_maxminrhoLp}). 
\item
The span rather starts and ends with at the most $2N$ and $2^N$ points at $MSE_{min}$ and $MSE_{max}$ extremum respectively for $k>2$ (while $2^N$ and $2N$ number of points respectively for $k<2$). 
%Notice that for $k=2$, we do not need to rely on such boundary conditions, since $MkE=MSE$ is true always for $k=2$.

%In \Cref{fig_mke_mse} for a given $L_p=c_0$, we denote the possible trajectories with 3 and 2 $MSE$ minimisation and maximisation conditions respectively (instead of $2N$ and $2^N$ respectively for the sake of clarity), and thus the example true $\mathcal{S}$ in the $(L_2, \rho_c)$ space.  We note here that the so-called true span $\mathcal{S}$ shown in \Cref{fig_mke_mse} is rather ad-hoc, and is for illustration purpose only.
\end{itemize}
\end{remark}
Evaluation of the properties (such as the parametric equation, shape, existence of  voids and discontinuities) of $\mathcal{S}$ in the $\{L_2,\rho_c\}$ space likewise as a function of the given gold standard for a fixed $L_p$ norm $(p\neq 2)$  (cf. \Cref{fig_mke_mse}) is a possible future research direction.

In summary, the operating regions in \Cref{fig_maxminrhoLp} are constrained by the  plots $\rho_{c_{max'}}$ and $\rho_{c_{min'}}\Bigr\rvert_{ \theta\in\left[1,\theta_{max}(N)\right]   }$ as per the  \Cref{eqn_generalmaxmin}.
While it is true that a given $MkE$ (effectively, a given $L_p$ norm) translates to a restricted range for $MSE$ (cf. \Cref{eqnl2rel,eqnl2rellek}), these $MSE$ extremities do not necessarily translate to $\rho_c$ optimisation as dictated by \Cref{eqn_generalmaxmin} (cf. \Cref{fig_mke_mse}). 
%thus also directly contradicting the popular notion of using $MSE$ or $L_p$ norm as the loss function. 
Rather, it is $\frac{MSE}{\sigma_{XY}}$ metric that needs to be optimised to find the optimised error distribution and sequence $D_{optimised}$ optimising $\rho_c$, when given \emph{any} constraint on $D:=(d_i)_1^N$ -- \eg a constant $L_p$ norm. Investigation of the shape and properties of the span of the valid $(L_p,\rho_c)$ pairs in the two dimensional $\{L_p,\rho_c\}$ space (similar to \Cref{fig_maxminrho}, but for $p \neq 2$) is yet another possible future research direction. 

We evaluate the conditions and formulation for these in the section next.
by optimising the function $\frac{MSE}{\sigma_{XY}}$ (cf. \Cref{eqn_cccmsemap}) and not just $MSE$, subject to the constraint: given $L_p=\left(\sum_{i=1}^{N} {|x_i-y_i|^p}\right)^\frac{1}{p}$.

%%%%%%% Lp condition finding STARTS HERE for sigma/MSE
%We evaluate conditions and formulation for the same in the section next.
%\newpage
\section{$\rho_c$ optimisation, given $L_p$ $(p=2m,\forall m\in \mathbb{N})$}
% \begin{alignat}{2}
\label{sec_m2mLp}
%Inspired by the discovery that the predictions with identical $MSE$ map to different $\rho_c$ values, the maximum and minimum $\rho_c$ for a constant $MSE$ are found next. 
%\vp{
%\Cref{sec_mae,sec_lp,sec_truespan} establish that the optimisation of $\rho_c$ for a given $L_p$ norm (and in general, given any constraint on the prediction and/or error sequence) requires optimisation  of  $\frac{MSE}{\sigma_{XY}}$ and not just $MSE$ (cf. \Cref{eqn_cccmsemap}).
We generalise the \Cref{sec_m2mm} to any given $L_p$ norm for every even $p$, beyond $p=2$. 
%where $p$ is an even natural number.
% 
The problem statement is, thus: 

\textit{
	Given (1) a gold standard sequence, $G:=(g_i)_1^N$, and (2) a fixed $L_p,p=2m,\forall m\in\mathbb{N}$, find 
	the set(/s) of error values $E:=(e_i)_1^N$ that achieve maximisation and minimsaition of $\rho_c$.
}

For the prediction $X:=(x_i)_1^N$  and the gold standard sequence 
$Y:=(y_i)_1^N$ or $G:=(g_i)_1^N $, 
%	from \Cref{eqncovmax,eqn_cccmsemap},
%
%	\begin{alignat}{2}
%	N\cdot\sigma_{XY}
%	&=
%	\sum_{i=1}^{N}{y_{z_i}}^2
%	+\sum_{i=1}^{N}d_i{y_{z_i}}
%%	\quad
%	\text{ and }
%%	\quad
%	\mathbf{\rho_c}
%%	\quad
%	=
%	%	&&
%	\left(1+\frac{MSE}{2\sigma_{XY}}\right)^{-1} \text{ ($\because$ \cref{eqncovmax,eqn_cccmsemap})}
%	\label{eqnkstart}
%	\end{alignat}
\begin{align}
%		\span\span\span
%	\text{For } k=2m,\hspace{6pt} \exists m\in \mathbb{N}, \hspace{6pt} \sum_{i=1}^{N}|d_i|^k-N\cdot MkE=\sum_{i=1}^{N}d_i^k-N\cdot MkE.
%	\label{eqnmkediffable}
%	\\
%	\end{align}
%	
%	Thus, to maximise $\rho_c$, we need to maximise $\frac{\sigma_{XY}}{MSE}$ as given by \Cref{eqncovmax} by tuning the error sequence $D:=(d_i)_1^N$, while satisfying the condition $\sum_{i=1}^{N}d_i^k-N\cdot MkE=0$; \ie 
%	($\because$ \Cref{eqnmkediffable}). 
%	
%	That is, we need to
%	
%	\noindent Formally speaking, we need to 
%	\begin{align}
%
%\because k=2m,\hspace{6pt}  m\in \mathbb{N}
%\span\span\span
%\\
%\nonumber
\text{
	%		I.\,e.,\ 
	Maximise: }&f(\left\{d_i\right\}_{1}^{N})\hspace{-12pt}&&= \frac{N\cdot \sigma_{XY}}{N\cdot MSE}
=
\frac{
	\sum_{i=1}^{N}{y_{z_i}}^2
	+\sum_{i=1}^{N}d_i{y_{z_i}}}{\sum_{i=1}^{N}d_i^2},
\nonumber%%
\\
\text{subject to: } &g(\left\{d_i\right\}_{1}^{N})\hspace{-12pt}&&= 
\sum_{i=1}^{N}|d_i|^k-N\cdot MkE=
\sum_{i=1}^{N}d_i^k-N\cdot MkE=0
\hspace{3pt}
(\because k\hspace{-2pt}=\hspace{-2pt}2m, m\hspace{-2pt}\in\hspace{-2pt}\mathbb{N})
.
%	\hspace{6pt} (\because k=2m,\hspace{6pt} \exists m\in \mathbb{N}).
\nonumber
\end{align}
With $f$ and $g$ continuous and differentiable $\forall d_i\in\mathbb{R}$,
%	 ($\because$ \Cref{eqnkstart,eqnmkediffable}), 
the auxiliary Lagrange expression $\mathcal {L}$ is:
\begin{align}
{\mathcal {L}}(d_1,d_2,\cdots,d_N,\lambda )=&f(d_1,d_2,\cdots,d_N)-\lambda \cdot g(d_1,d_2,\cdots,d_N),
\nonumber%%
\\
\label{eqnlagrange_funcs}
{\mathcal {L}}
=&\frac{\sigma_{XY}}{MSE}
-\lambda \left(\sum_{i=1}^{N}d_i^k-N\cdot MkE\right).\\
%	&=	\frac{
%		\sum_{i=1}^{N}{y_{z_i}}^2
%		+\sum_{i=1}^{N}d_i{y_{z_i}}}{\sum_{i=1}^{N}d_i^2}
%	-\lambda \left(\sum_{i=1}^{N}d_i^k-N\cdot MkE\right)\\
%	\because
\therefore
\nabla_{d_1,d_2,\cdots,d_N,\lambda}
{\mathcal {L}}
=&0
\Leftrightarrow 
{\begin{cases}
	\frac{\partial }{\partial d_i}\mathcal {L} = 0
	\therefore
	\frac{\partial}{\partial d_i} 
	\left(\frac{\sigma_{XY}}{MSE}\right)
	-\lambda k d_i^{k-1}
	=0
	%		 \quad 
	\forall  i \in \mathbb{N} : i \in [1,N].
	\\
	\frac{\partial}{\partial \lambda}\mathcal {L} = 0		
	\therefore
	\sum_{i=1}^{N}d_i^k-N\cdot MkE=0.
	\end{cases}}
\label{eqnlag2m}
%\end{align}
\\
%    Change in $d_i$ does not affect $\mathtt{s,p,m,c}$. 
%\begin{align}
%	\because
%\frac{d}{dx}\left(\frac{u}{v}\right)&=
%\frac{v\frac{du}{dx}-u\frac{dv}{dx}}{v^2},
%\hspace{3pt}
%%\quad
%	\therefore	
\frac{\partial}{\partial d_i} 
\left(\frac{\sigma_{XY}}{MSE}\right)
=&\frac{MSE\cdot
	\frac{\partial}{\partial d_i}
	\sigma_{XY}
	-\sigma_{XY}\cdot
	\frac{\partial}{\partial d_i}
	MSE}{MSE^2}.
\nonumber%%
\\
\therefore 0=
\frac{\partial}{\partial d_i} \mathcal {L}
=& \frac{MSE\cdot
	\frac{\partial}{\partial d_i}
	\sigma_{XY}
	-\sigma_{XY}\cdot
	\frac{\partial}{\partial d_i}
	MSE}{MSE^2}
- \lambda k d_i^{k-1}. 
\nonumber%%
\\
\therefore
\lambda k d_i^{k-1}
=&
\frac{
	MSE\cdot
	\frac{\partial}{\partial d_i}
	\sigma_{XY}
	-\sigma_{XY}\cdot
	\frac{\partial}{\partial d_i}
	MSE
}{MSE^2},
\nonumber%%	
\\
=&
\frac{
	MSE\cdot
	\frac{\partial}{\partial d_i}
	\sum_{i=1}^{N}y_{z_i}d_i
	-\sigma_{XY}\cdot
	\frac{\partial}{\partial d_i}
	\sum_{i=1}^{N} d_j^2
}{N\cdot MSE^2}	.
%=
%\frac{
%MSE\cdot
%y_{z_i}
%-\sigma_{XY}\cdot
%2 d_i
%}{MSE^2}	
\nonumber%%
\\
\therefore
d_i^k
=&\frac{MSE\cdot y_{z_i}d_i 
	-2 d_i^2 (\sigma_{G}^2+ \sigma_{GD})}{\lambda k  N MSE^2} \quad (\because \text{ \Cref{eqncovmax}}),
\label{eqndik}
%: \sigma_{XY}= \sigma_{G}^2+\sigma_{GD})
%\\
%\nonumber
%\nonumber%%
\\
\nonumber
\text{where: }
\sigma_{GD}=&\frac{\sum_{i=1}^{N} y_{z_i}d_i}{N}:\text{ Covariance between $G$ and $D$}.\\
%\nonumber
%\sigma_{G}^2=&\frac{\sum_{i=1}^{N} y_{z_i}^2}{N} :\text{ Standard deviation of }G.\\
\therefore
\sum_{i=1}^{N}\frac{d_i^k}{N}
=MkE=&
\sum_{i=1}^{N}
\frac{MSE\cdot \frac{y_{z_i}d_i}{N} 
	-2 \frac{d_i^2}{N} (\sigma_{G}^2+ \sigma_{GD})}{\lambda \cdot k \cdot N  \cdot MSE^2}.
\nonumber%%
\\
\therefore 
MkE
=&
\frac{MSE\cdot \sigma_{GD}
	-2 \cdot MSE\cdot (\sigma_{G}^2+ \sigma_{GD})}{\lambda \cdot k \cdot N \cdot MSE^2},
\nonumber%%
%\\
%=&
=
\frac{
	-(2 \sigma_{G}^2+ \sigma_{GD}) } {\lambda \cdot k \cdot N \cdot MSE}.
\nonumber%%
\\
\therefore
\lambda \cdot k \cdot N \cdot MSE^2
=&
\frac{
	-(2 \sigma_{G}^2+ \sigma_{GD})\cdot MSE } {MkE}
\text{ ( = denominator in \Cref{eqndik})}.
\nonumber%%
\\
\therefore d_i^k
=&-MkE\frac{MSE\cdot y_{z_i}d_i 
	-2 d_i^2 (\sigma_{G}^2+ \sigma_{GD})}{(2 \sigma_{G}^2+ \sigma_{GD})\cdot MSE}.	
\nonumber%%
%\end{align}
\\
%\begin{align}
%\span
%\hspace{-24pt}
%&
\therefore
%&
\quad d_i^k\cdot( 2 \sigma_{G}^2+\sigma_{GD})\cdot MSE
=&-MkE \cdot (MSE \cdot y_{z_i}d_i-2d_i^2\sigma_{G}^2- 2d_i^2\sigma_{GD})
%\nonumber%%
\label{eqnsimpler}
\\
%\nonumber
%\span
%\text{Dividing by $MSE\cdot MkE$}
%\\
%&
\therefore
%\quad
%&
%\Aboxed{
	%0=
	2 \sigma_{G}^2\cdot\left(\frac{d_i^k}{MkE}-\frac{d_i^2}{MSE}\right)
	&+\sigma_{GD}\cdot\left(\frac{d_i^k}{MkE}-2\frac{d_i^2}{MSE}\right)+y_{z_i}{d_i}
	=0.
%}
%\Aboxed{
%	0&=\mathbf{2 \sigma_{G}^2}\cdot\left(\frac{d_i^k}{\mathbf{MkE}}-\frac{d_i^2}{MSE}\right)
%	+\sigma_{\mathbf{G}D}\cdot\left(\frac{d_i^k}{\mathbf{MkE}}-\mathbf{2}\frac{d_i^2}{MSE}\right)+\mathbf{y_{z_i}}{d_i}
%}
%=0
\label{eqnsolvethisk}
%\\
%&
\end{align}
\begin{align}
\nonumber
%\text{Denoting the constant terms in bold, we get a $k$-degree polynomial in $d_i$: }
\text{Dividing \Cref{eqnsimpler} by $d_i$, and expressing $\sigma_{GD}$ and $MSE$ in terms of $d_i$, we get a polynomial in $d_i$: }
\span
\\
\label{eqnsolvethispoly}
0\quad&=\quad 2 \mathbf{d_i^{k+1}} \Big(\sum_{\forall j} y_{z_j}^2\Big)
+ 2 \mathbf{d_i^{k-1}} \Big(\sum_{\forall j} y_{z_j}^2\Big)\Big(\sum_{j\neq i} d_j^2\Big)
%\nonumber
\\
\nonumber
&\quad + \mathbf{d_i^{k+2}} y_i 
+ \mathbf{d_i^{k} } y_i \Big(\sum_{j\neq i} d_j^2\Big)
+ \mathbf{d_i^{k+1} }  \Big(\sum_{j\neq i} y_{z_j} d_j\Big)
+ \mathbf{d_i^{k-1}}  \Big(\sum_{j\neq i} y_{z_j} d_j\Big) \Big(\sum_{j\neq i} d_j^2\Big)
%\span
\\
&\quad+N\cdot MkE \cdot 
\Big( y_{z_i} \mathbf{d_i^2}
+  y_{z_i} \sum_{j\neq i} d_j^2
-2 \mathbf{d_i} \sum_{\forall j} y_{z_j}^2
-2 \mathbf{d_i^2} y_{z_i} 
-2 \mathbf{d_i} \sum_{j\neq i} y_{z_j} d_j \Big)
\nonumber
\end{align}
Solving \Cref{eqnsolvethispoly} for $d_i$ in terms of $y_{z_i}$, \ie $d_i=\zeta(y_{z_i})$, one finds an optimal  
%$D=(d_i)_{1}^{N}$ 
$\{d_i\}_{1}^{N}$ 
maximising $\rho_c$.
For example, substituting $k=2$ results in \Cref{eqn_rhocmax,eqn_rhocmin} derived in \Cref{sec_m2mm} (cf. \Cref{sec_app6_solvethisk}). %\footnote{\EqnValidityMSE}

Rewriting \Cref{eqnsolvethisk} in terms of $y_{z_i}$, we get:
\begin{alignat}{3}
%\nonumber
%0\quad=&\quad	2 \sigma_{G}^2\cdot\left(\frac{d_i^{k-1}}{MkE}-\frac{d_i}{MSE}\right)
%&+&&\sigma_{GD}\cdot&\left(\frac{d_i^{k-1}}{MkE}-2\frac{d_i}{MSE}\right)
%+y_{z_i}
%\\
\nonumber
%\text{\ie }
%\implies
0\quad=&\quad\sum_{j=1}^{N} y_{z_j}^2 \left(\frac{d_i^{k-1}}{MkE}-\frac{d_i}{MSE}\right)
&+&&\quad\sum_{j=1}^{N} y_{z_j}\cdot d_j\cdot&\left(\frac{d_i^{k-1}}{MkE}-2\frac{d_i}{MSE}\right)
+y_{z_i}\\
\nonumber
%\text{\ie }
%\implies
=&\quad\sum_{j=1}^{N} y_{z_j}^2 \left({d_i^{k-1}}\cdot{MSE} - {d_i}\cdot{MkE}\right)
+\sum_{j=1}^{N} y_{z_j}\cdot d_j\cdot\left({d_i^{k-1}}\cdot{MSE} -2{d_i}\cdot{MkE} \right)
+y_{z_i}\cdot{MSE} \cdot{MkE}
\span\span\span\span
\\
\nonumber
%\text{\ie }
%\implies
=&
\phantom{+\quad \sum_{j=1}^{N}}
y_{z_i}^2 \left({d_i^{k-1}}\cdot{MSE} - {d_i}\cdot{MkE}\right)
%\quad
+
%\quad
\phantom{\sum_{j=1}^{N}}
y_{z_i}\cdot d_i\cdot\left({d_i^{k-1}}\cdot{MSE} -2{d_i}\cdot{MkE} \right)
+y_{z_i}\cdot{MSE} \cdot{MkE} 
\span\span\span\span
\\
\nonumber
&+\quad
\sum_{j\neq i}^{} y_{z_j}^2 \left({d_i^{k-1}}\cdot{MSE} - {d_i}\cdot{MkE}\right)
+\sum_{j\neq i}^{} y_{z_j}\cdot d_j\cdot\left({d_i^{k-1}}\cdot{MSE} -2{d_i}\cdot{MkE} \right)
\span\span\span\span
\\
=&\quad \mathbf{y_{z_i}^2 }
+ \mathbf{y_{z_i}} \left(d_i + \frac{ {MSE} \cdot{MkE}-d_i^2\cdot MkE}{d_i^{k-1}\cdot{MSE} -{d_i}\cdot{MkE} } \right)
&+& \Bigg[\sum_{j\neq i}^{}  y_{z_j}\Bigg(y_{z_j} + d_j\frac{ {d_i^{k-1}}\cdot{MSE} -2{d_i}\cdot{MkE}}{d_i^{k-1}\cdot{MSE} -{d_i}\cdot{MkE} } \Bigg)\Bigg]\span\span
\label{eqnyzi2m}
\end{alignat}
Solving quadratic \Cref{eqnyzi2m} would result in expression of $y_{z_i}$ in terms of $d_i$, possibly free of $y_{z_j}, \text{for }j\neq i$. 
%\Cref{eqnyzi2m} is a quadratic equation in terms of $y_{z_i}$, solving which would result in expression of $y_{z_i}$ in terms of $d_i$, possibly eliminating $y_{z_j}, \text{for }j\neq i$. 
%Because $\sigma_{{XY}}=\sum_{i=1}^{N}\frac{ y_{z_i}\cdot(y_{z_i}+d_i)}{N}$, $MSE$=$\sum_{i=1}^{N}\frac{ d_i^2}{N}$ and $\rho_c=\left(1+\frac{MSE}{\sigma_{XY}}\right)^{-1}$ according to \Cref{eqn_cccmsemap,eqncovmax,yzdef}, we can obtain the expression for $\rho_c$ in terms of $d_i$ for a given $L_k$-norm or for a given $MkE$, for the given gold standard sequence $Y$.
%
We note that this true optimisation formulation (cf. \Cref{eqnsolvethispoly,eqnyzi2m}) is different than that we obtained in \Cref{sec_lp} (cf. \Cref{eqn_generalmaxmin}), where only $MSE$ was optimised 
%had focused on optimising  $MSE$ only, 
instead of $\frac{\sigma_{XY}}{MSE}$.
%%%%%%% Lp condition finding ENDS HERE for sigma/MSE

%\newpage

\section{$\rho_c$ optimisation, 
	given the error-set
	%	Given Fixed Set of Prediction Errors
}
\label{sec_prob}
Different conditions on the error-set  for $\rho_c$ minimisation and maximisation  in \Cref{sec_m2mm} establish the relevance of error values for $\rho_c$ evaluation, for a given $MSE$.
%
%Having found the conditions for $\rho_c$ minimisation and maximisation for a given $MSE$ in \Cref{sec_m2mm}, we have established that the value-distribution of errors that make $MSE$ plays an important role in the $\rho_c$  metric evaluation.
In this section, we prove that not only the values of the individual errors, but also their order directly impacts the $\rho_c$ metric evaluation.
%Inspired by the discovery that the predictions with identical $MSE$ map to different $\rho_c$ values
To this end, we attempt to decouple the components of $\rho_c$ that are dependent on merely the values of errors, from those directly impacted by the `sequence'/ordering of the errors. 
%-- by considering a fixed set of prediction errors.  
The problem statement is, thus:
%We, thus, formulate our problem as follows. 

\textit{
	Given (1) a gold standard time series, $G:=(g_i)_1^N$, and (2) a fixed set of error values, $E:=(e_i)_1^N$ (thus a fixed $MSE$), find the distribution(s) or correspondence(s) of error values with respect to the gold standard that achieve(s) the highest possible $\rho_c$.
}

Let the prediction and the gold standard sequences be 
$X:=(x_i)_1^N$ and $Y:=(y_i)_1^N$, not necessarily in that order. Note that, 
%which variable represents what sequence 
the sequence to variable correspondence
does not affect $\rho_c$ evaluation -- since the $\rho_c$ formulation is symmetric with respect to $X$ and $Y$. 
%
%From \Cref{eqn_cccmsemap}:
%we have:
%\begin{align}
%\mathbf{\rho_c}\quad
%&=
%\Big(1+\frac{MSE}{2\sigma_{XY}}\Big)^{-1} = \Big(1+\frac{N\cdot MSE}{2\sum_{i=1}^{N}{(x_i-\mu_X)(y_i-\mu_Y)}}\Big)^{-1}
%(\because \text{\Cref{eqn_cccmsemap}}).
%\nonumber%%
%%\\
%%\nonumber%%
%%\text{If } d_i&:=x_i-y_i,  %\\ 
%%%\quad\quad 
%%\text{ and } 
%%% \quad 
%%\mu_{D}:=\frac{\sum_{i=1}^{N}d_i}{N}
%%%  \\
%%%\therefore 
%%%\quad 
%%\implies
%%\mu_{D} 
%%%&
%%=\mu_X-\mu_Y 
%%%\quad
%%\hspace{6pt}
%%(\because \text{\Cref{yzdef}).}
%%%\\
%\end{align}

\subsection{Formulation 1: Replacing $(x_i)$ with $(y_i+d_i)$}
\label{sec_form1}
\begin{align}
\nonumber%%
%\therefore 
\text{From \Cref{eqn_cccmsemap}, we have: }
\mathbf{\rho_c}
&= \left(1+\frac{N\cdot MSE}{2\sum_{i=1}^{N}{(y_i-\mu_Y+d_i-\mu_D)(y_i-\mu_Y)}}\right)^{-1}.\\
\nonumber\span
\text{Note that } \sum_{i=1}^{N}\mu_D (y_{i}-\mu_Y)=\mu_D\sum_{i=1}^{N} (y_{i}-\mu_Y)=\mu_D(N\mu_Y-N\mu_Y)=0.
\\
\nonumber%%
\therefore \mathbf{\rho_c}
&= \left(1+\frac{N\cdot MSE}{2\sum_{i=1}^{N}(y_i-\mu_Y)^2
	+ 2\sum_{i=1}^{N}y_id_i
	-2\sum_{i=1}^{N}\mu_Yd_i}\right)^{-1},\\
\nonumber%%
&= \left(1+\frac{N\cdot MSE}{2N{\sigma_Y}^2
	+ 2\sum_{i=1}^{N}\mathbf{y_id_i}
	-2N\mu_Y\mu_D}\right)^{-1}.\\
%\nonumber
%\span 
\because \left(1+\frac{a}{b}\right)^{-1}
%&
=\left(1-\frac{a}{a+b}\right)
%\\
\implies
%\therefore 
\quad
%\Aboxed{
	\mathbf{\rho_c}
	&= \Bigg(1-\frac{N\cdot MSE}{2N({\sigma_Y}^2-\mu_Y\mu_D)
		+ 2\sum_{i=1}^{N}\mathbf{y_id_i}
		+N\cdot MSE}\Bigg).
%}
\label{eqnYdrho}
\end{align}
Given a gold standard $Y$ and \{${d_i}$\}, $\rho_c$ maximisation necessitates $\displaystyle{\sum _{i=1}^{N}\mathbf{(y_{i})d_i}}$ 
\textbf{
	maximisation
}.
\subsection{Formulation 2: Replacing $(y_i)$ with $(x_i-d_i)$}
\label{sec_form2}
%Likewise (Cf. \Cref{sec_app1}),
\begin{flalign}
%\hspace{1cm}
&\text{Likewise (Cf. \Cref{sec_app1}), }
%\Aboxed{
	\mathbf{\rho_c}
	= \Bigg(1-\frac{N\cdot MSE}{2N({\sigma_X}^2+\mu_X\mu_D)
		-2\sum_{i=1}^{N}\mathbf{x_id_i}
		+N\cdot MSE}\Bigg).&
%}
\label{eqnXdrho}
\end{flalign}
Given a gold standard $X$ and \{${d_i}$\}, $\rho_c$ maximisation necessitates $\displaystyle{\sum _{i=1}^{N}\mathbf{(x_{i})d_i}}$ 
\textbf{
	minimisation
}.

% \newpage
\subsection{Paradoxical nature of the conditions on the error-set}
\label{sec_paradox}

The concluding remarks of \Cref{sec_form1} and \Cref{sec_form2} imply that we arrive at mutually contradictory requirements in terms of the rearrangement of values in the error-set. This is only an apparent paradox, since given a set of error values ($E$) and a gold standard sequence ($G$), two different candidate prediction sequences ($P=G+E$ or $P=G-E$) may be generated yielding a high $\rho_c$,.
%, for a given $MSE$. 
%
% This -- though in a certain limited sense -- is only an `apparent' contradiction. 
%However, 
When devising a loss function in terms of the predicted sequence itself, however, the requirements implied by both formulations converge to the same conditions. We discuss next this rediscovery of consistency, arising out of the prima facie mutually contradictory insights interestingly. 
%To this end, 
% Because we have used $X$ and $Y$ interchangeably, although in different sections, 
%we formally (and this time unambiguously) redefine the symbols for prediction sequence, gold standard sequence and the error sequence -- instead of using $X$ and $Y$ interchangeably, albeit in different sections thus far (\Cref{sec_form1,sec_form2}). 
% to be $G$,$P$,$E$, replacing $X$, $Y$ and $D$ 
% for the sake of clarity.
% \newpage
%\vspace{6pt}
%\noindent
%Let 
%%$\displaystyle G:=(g_i)_1^N$, $\displaystyle E:=(e_i)_1^N$, and
% $\displaystyle P:=(p_i)_1^N$ be 
%% the gold standard sequence, the error sequence and 
% the prediction sequence.
% and $\mu_A$ be the arithmetic mean of sequence $A$, $A\in {G,E,P}$.
% respectively. Let 
% $\mu_G$, $\mu_E$, and $\mu_P$ be the arithmetic means of the sequences $G$, $E$, and $P$ respectively. 
%\newpage
\subsubsection{Contradictory requirements in terms of the $ 
	%	\smash
	{\left(\sum _{i=1}^{N}\mathbf{g_{i}e_i}\right)}$
	summation}
% \vspace{0.3cm}
%The \Cref{sec_form1,sec_form2}  imply contradictory requirements  for the product summation $ {\left(\sum _{i=1}^{N}\mathbf{g_{i}e_i}\right)}$.
%
%\begin{tikzpicture}[remember picture,overlay]
%\draw[line width=1,black]
%([xshift=0ex,yshift=12pt]current page text area.west|-{pic cs:beg-ge1})
%rectangle
%([xshift=12pt,yshift=-6pt]current page text area.east|-{pic cs:end-ge1});
%%
%\draw[line width=1,black]
%([xshift=0ex,yshift=12pt]current page text area.west|-{pic cs:beg-ge2a})
%rectangle
%([xshift=12pt,yshift=-3pt]current page text area.east|-{pic cs:end-ge2a});
%%
%\draw[line width=1,black]
%([xshift=0ex,yshift=12pt]current page text area.west|-{pic cs:beg-ge2b})
%rectangle
%([xshift=12pt,yshift=-3pt]current page text area.east|-{pic cs:end-ge2b});
%%
%\draw[line width=1,black]
%([xshift=0ex,yshift=12pt]current page text area.west|-{pic cs:beg-gpa})
%rectangle
%([xshift=12pt,yshift=-6pt]current page text area.east|-{pic cs:end-gpa});
%%
%\draw[line width=1,black]
%([xshift=0ex,yshift=12pt]current page text area.west|-{pic cs:beg-gpb})
%rectangle
%([xshift=12pt,yshift=-6pt]current page text area.east|-{pic cs:end-gpb});
%\end{tikzpicture}
Specifically, to maximise $\rho_c$:
\begin{itemize}
	\item 
	\tikzmark{beg-ge1}
	The formulation in the \Cref{sec_form1} requires \emph{maximisation} of the  $ 
	%	\smash
	{\left(\sum _{i=1}^{N}\mathbf{g_{i}e_i}\right)}$ quantity. 
	\item The formulation in the \Cref{sec_form2} requires \emph{minimisation} of the  $ 
	%	\smash
	{\left(\sum _{i=1}^{N}\mathbf{g_{i}e_i}\right)}$ quantity. \tikzmark{end-ge1}
\end{itemize}

\subsubsection{Contradictory requirements in terms of the error-set permutation}
% \vspace{0.3cm}
\begin{itemize}
	\item 
	%	The  requirement posed by the formulation in the
	%	 \Cref{sec_form1}, \ie the 
	\tikzmark{beg-ge2a}
	Maximisation of $ 
	%	\smash
	{\left(\sum _{i=1}^{N}\mathbf{g_{i}e_i}\right)}$ (as per \Cref{sec_form1}) necessitates that the error values are in the \emph{same} sorted order as of the gold standard values 
	(cf. \Cref{sec_app2}). 
	\tikzmark{end-ge2a}
	\item That is, a bigger $e_i$ 
	%	needs to get multiplied with the 
	corresponds to a 
	bigger $g_i$.
	\item 
	%	The  requirement posed by the formulation in \Cref{sec_form2}, \ie the 
	\tikzmark{beg-ge2b}
	Minimisation of $ 
	%	\smash
	{\left(\sum _{i=1}^{N}\mathbf{g_{i}e_i}\right)}$ (as per \Cref{sec_form2}) necessitates that the error values are in the \emph{opposite} order as of the gold standard values 
	(cf. \Cref{sec_app2}).
	\tikzmark{end-ge2b}
	\item That is, a smaller $e_i$ 
	%	needs to get multiplied with the
	corresponds to a
	bigger $g_i$.
\end{itemize}

In other words, taking into account bot the magnitudes and the signs, the error values need to be sorted in the same order as of the elements of the time series, as dictated by the first (\Cref{sec_form1}) $\rho_c$ formulation. The second formulation (\Cref{sec_form2}) dictates to the contrary; the errors need to be sorted in exactly the opposite order as of the sorted elements of the gold-standard time series.  Thus, there exist two prediction sequences corresponding to an identical set of errors (with respect to the gold standard sequence), that maximise $\rho_c$.

\subsubsection{Consistent requirements in terms of the $
	%	\smash
	{\left(\sum _{i=1}^{N}\mathbf{g_{i}p_i}\right)}$ summation}
\label{sec_consistent}
% \vspace{0.3cm}
While there exist two distinct prediction sequences with an identical error-set (and thus an identical $MSE$), and despite the contradictory requirements in terms their permutations, we establish next that the conditions for the $\rho_c$ maximisation in terms of the product summation $\smash{\sum _{i=1}^{N}\mathbf{g_{i}p_i}}$ are consistent in both these formulations. 
%In summary, the formulations 1 and 2 end up in identical requirements in terms of the product summation $ \smash{\left(\sum _{i=1}^{N}\mathbf{g_{i}p_i}\right)}$. This consistency in the requirements can be proven as follows.

According to the first formulation (\Cref{sec_form1}),
\begin{itemize}
	\item  $E=P-G$, \ie $(e_i)_1^N=(p_i)_1^N-(g_i)_1^N$, and $\smash{\left(\sum _{i=1}^{N}\mathbf{g_{i}e_i}\right)}$ needs to be maximised.
	\item 
	\tikzmark{beg-gpa}
	That is, $
	%	\smash
	{\left(\sum _{i=1}^{N}\mathbf{g_{i}(p_i-g_i})\right)}$ needs to be maximised,
	%	\item 
	implying that $
	%	\smash
	{\left(\sum _{i=1}^{N}\mathbf{g_{i}(p_i)}\right)}$  needs to be \emph{maximised} -- since $
	%	\smash
	{\left(\sum _{i=1}^{N}\mathbf{g_{i}^2}\right)}$ is constant for a given gold standard.
	\tikzmark{end-gpa}
\end{itemize} 

According to the second formulation (\Cref{sec_form2}),
\begin{itemize}
	\item  $E=G-P$, \ie $(e_i)_1^N=(g_i)_1^N-(p_i)_1^N$, and $\smash{\left(\sum _{i=1}^{N}\mathbf{g_{i}e_i}\right)}$ needs to be minimised.
	\item 
	\tikzmark{beg-gpb}
	That is,   $
	%	\smash
	{\left(\sum _{i=1}^{N}\mathbf{g_{i}(g_i-p_i})\right)}$ needs to be minimised,
	%	\item 
	implying  that $
	%	\smash
	{\left(\sum _{i=1}^{N}\mathbf{g_{i}(p_i)}\right)}$ needs to be \emph{maximised} -- since $
	%	\smash
	{\left(\sum _{i=1}^{N}\mathbf{g_{i}^2}\right)}$ is constant for a given gold standard.
	\tikzmark{end-gpb}
\end{itemize} 
% \newpage
% \newpage
%%%%%%%%%%%%%%%%% HIGHEST ACHIEVABLE STARTS ######
\subsection{Optimal  $\rho_c$ formulation, given the error-set}
\label{sec_maxmin12}

%\begin{tikzpicture}[remember picture,overlay]
%\draw[line width=1,black]
%([xshift=0pt,yshift=3.8ex]current page text area.west|-{pic cs:beg-max12})
%rectangle
%([xshift=-108pt,yshift=-3.8ex]current page text area.east|-{pic cs:end-max12});
%%
%\draw[line width=1,black]
%([xshift=0pt,yshift=3.8ex]current page text area.west|-{pic cs:beg-min12})
%rectangle
%([xshift=-108pt,yshift=-3.8ex]current page text area.east|-{pic cs:end-min12});
%\end{tikzpicture}

%Having obtained the two prediction sequences maximising the $\rho_c$ with the exact same set of error values, we compare the two corresponding $\rho_c$ metrics. 
%Our attempt here is 
%We compare $\rho_c$ of the two prediction sequences with an identical error set, that are governed by \Cref{eqnrhoc1def,eqnrhoc2def}. 
%If it is impossible to determine this conclusively, we attempt to, at the very least, establish the conditions under which one $\rho_c$ is conclusively larger than the other.
%
%The choice of error sequence is subject to certain conditions, derived as follows.
%We establish the conditions to check which we can use to conclusively choose from one of the two error coefficient redistributions to get highest possible $\rho_c$.
%\begin{proof}
% 
Let  $\displaystyle \overset{\mathit{1}}{E}=(\overset{\mathit{1}}{e}_i)_1^N=$ the reordered $P-G$' error sequence (cf. \Cref{sec_form1}), and $\displaystyle \overset{\mathit{2}}{E}=(\overset{\mathit{2}}{e}_i)_1^N=$ the reordered of `$G-P$' error sequence (cf. \Cref{sec_form2}) denote the two optimal error permutations corresponding to the \emph{sorted} rearrangement of $G:=\displaystyle \bar{G}=(\bar{g}_i)_1^N$ for $\rho_c$ maximisation. From \Cref{sec_app4_chev,sec_consistent}:
%According to Chebyshev's sum inequality (cf. \Cref{sec_app4_chev}) and the closing remarks of the previous subsection (cf. \Cref{sec_consistent}), 
%the indexing of $\overset{\mathit{1}}{E}$ and $\overset{\mathit{2}}{E}$ is consistent with 
%the sorted rearrangement of $G$, \ie
% $\displaystyle \bar{G}:=(\bar{g}_i)_1^N$.
\begin{align}%{6}
\nonumber%%
%\label{orderG}
%\text{That is, }
&\text{if for }\quad
&&\bar{G},\quad\quad
&&\quad 
\bar{g}_1 
&&\quad\geq\quad
\bar{g}_2 &&\quad\cdots
&&\quad\geq\quad
\bar{g}_N,\\
\label{orderErr1}
&\text{then for } \quad
&&\overset{\mathit{1}}{E}:\quad\quad
&&\quad 
\overset{\mathit{1}}{e}_1 
&&\quad\geq\quad
\overset{\mathit{1}}{e}_2 &&\quad\cdots
&&\quad\geq\quad
\overset{\mathit{1}}{e}_N,\\
\label{orderErr2}
&\text{and for }\quad
&&\overset{\mathit{2}}{E}:\quad\quad
&&\quad 
\overset{\mathit{2}}{e}_1 
&&\quad\leq\quad
\overset{\mathit{2}}{e}_2 &&\quad\cdots
&&\quad\leq\quad
\overset{\mathit{2}}{e}_N,\\
\nonumber%%
&\text{\ie }
&&\overset{\mathit{2}}{E}:\quad\quad
&&\quad 
\overset{\mathit{2}}{e}_{N} 
&&\quad\geq\quad
\overset{\mathit{2}}{e}_{N-1} &&\quad\cdots
&&\quad\geq\quad
\overset{\mathit{2}}{e}_1,\\
% &&\overset{\mathit{1}}{E},\quad\quad
% &&\quad 
% \overset{\mathit{1}}{e}_{N} 
% &&\quad\leq
% \overset{\mathit{1}}{e}_{N-1}, \cdots
% &&\quad\leq
% \overset{\mathit{1}}{e}_1, \\
&&&\therefore &&\quad \overset{\mathit{2}}{e}_{j} 
&&\quad=\quad \overset{\mathit{1}}{e}_{N+1-j} &&\quad&&\forall\quad j \in \mathbb{N} : j \in [1,N].\label{eqn_symmetric_coef}
\end{align}
% 
%	From \Cref{eqnYdrho,eqnXdrho},
\begin{align}
%\tikzmark{beg-max12}
\label{eqnrhoc1def}
\text{Thus, }
%\therefore
%\Aboxed{
\rho_{c_{max_1}} &=\bigg(1-\frac
{
	% \displaystyle
	N(MSE)}
{
	% \displaystyle
	2N\sigma_G^2+2\sum _{i=1}^{N}{\bar{g_{i}}\overset{\mathit{1}}{e}_i}-2N\mu_G\mu_E+N\cdot MSE}\bigg)
%}
\hspace{4pt} 
\because \text{ \Cref{eqnYdrho,orderErr1}},
\\
\label{eqnrhoc2def}
%\therefore
%\Aboxed{
\rho_{c_{max_2}} &=\bigg(1-\frac
{
	% \displaystyle
	N(MSE)}
{
	% \displaystyle
	2N\sigma_G^2-2\sum _{i=1}^{N}{\bar{g_{i}}\overset{\mathit{2}}{e}_i}+2N\mu_G\mu_E+N\cdot MSE}\bigg)
%}
%\tikzmark{end-max12}
\hspace{4pt}
\because \text{ \Cref{eqnXdrho,orderErr1}},
%\end{align}
\\
%\begin{align}
%\tikzmark{beg-min12}
\label{eqnrhoc1defmin}
%\text{Likewise, }
%\Aboxed{
\rho_{c_{min_1}} &=\bigg(1-\frac
{
	% \displaystyle
	N(MSE)}
{
	% \displaystyle
	2N\sigma_G^2+2\sum _{i=1}^{N}{\bar{g_{i}}\overset{\mathit{2}}{e}_i}-2N\mu_G\mu_E+N\cdot MSE}\bigg)
%}
\hspace{5pt} 
\because \text{ \Cref{eqnYdrho,orderErr2}},
\\
\label{eqnrhoc2defmin}
%\Aboxed{
\rho_{c_{min_2}} &=\bigg(1-\frac
{
	% \displaystyle
	N(MSE)}
{
	% \displaystyle
	2N\sigma_G^2-2\sum _{i=1}^{N}{\bar{g_{i}}\overset{\mathit{1}}{e}_i}+2N\mu_G\mu_E+N\cdot MSE}\bigg)
%}
%\tikzmark{end-min12}
\hspace{5pt}
\because \text{ \Cref{eqnXdrho,orderErr2}}.
\end{align}
\begin{remark}
From the equations above, following observations can be made:
	\begin{itemize}
\item
The denominators in \Cref{eqnrhoc1def,eqnrhoc2def} are strictly non-negative (cf. Chebyshev's sum inequality in \Cref{sec_app4_chev}), and are strictly $\leq N\cdot MSE$. Which implies:
%	 $0\leq\rho_{c_1},\rho_{c_2}\leq1$
%\begin{align}
%\nonumber%%
$0\leq\rho_{c_{max_1}},\rho_{c_{max_2}}\leq1$.
%\end{align}
%
%That is, both the error permutations (\Cref{orderErr1,orderErr1}) result in the sequences that are non-negatively correlated.
\item
In other words, there exists a permutation of every possible error-set that results in a non-negatively correlated prediction sequence, irrespective of the definition or formulation of the error sequence used (\ie whether $P-G$ or $G-P$), and irrespective of the $MSE$ the error sequence amounts to.
%These observations are consistent with \Cref{fig_maxminrho}.  
\item
$\rho_{c_{max_1}}, \rho_{c_{max_2}}$  $\to 0$ only when $MSE\to\infty$. Both these inferences are consistent with \Cref{fig_maxminrho}.
\item
As per Chebyshev's sum inequality (cf. \Cref{sec_app4_chev}) and from \Cref{orderErr1,orderErr2}: 
%\Cref{eqnChev1,eqnChev2}):
\begin{align}
\nonumber%%
%\text{Chebyshev's sum inequality (cf. \Cref{sec_app4_chev}) }\implies
2\sum _{i=1}^{N}{\bar{g_{i}}\overset{\mathit{2}}{e}_i}-2N\mu_G\mu_E \leq 0 \hspace{3pt}
\text{ and }
\hspace{3pt}
-2\sum _{i=1}^{N}{\bar{g_{i}}\overset{\mathit{1}}{e}_i}+2N\mu_G\mu_E
\leq 0
\\
\nonumber%%
\implies
\rho_{{c}_{min_1}}, \rho_{{c}_{min_2}} 
\leq 0 \quad \text{ for a smaller } \frac{\sigma_{G}^2}{MSE}
\quad
\because \text{ \Cref{eqnrhoc1defmin,eqnrhoc2defmin}}.
%\text{Thus,} $\rho_{{c}_{min_1}}$ and $\rho_{{c}_{min_2}}$ are negative for the smaller values of $\frac{\sigma_{G}^2}{MSE}$.
\end{align}
%$2\sum _{i=1}^{N}{\bar{g_{i}}\overset{\mathit{2}}{e}_i}-2N\mu_G\mu_E$,
%$-2\sum _{i=1}^{N}{\bar{g_{i}}\overset{\mathit{1}}{e}_i}+2N\mu_G\mu_E
%\leq 0$. 
%Thus, $\rho_{{c}_{min_1}}$ and $\rho_{{c}_{min_2}}$ are negative for the smaller values of $\frac{\sigma_{G}^2}{MSE}$.
This observation is consistent with \Cref{fig_maxminrho}, where:
%\begin{align}
%\nonumber%%
$-1\leq\rho_{c_{min}}\leq 0 \quad \text{ for } \quad \frac{MSE}{\sigma_{G}^2}\geq 1.$
%\end{align}
\item
A generalisation, or determining conclusively $\rho_{c_{max_1}}
%\stackrel{?}{\geq}
\gtreqqless
\rho_{c_{max_2}}$ is \emph{not} possible (cf. \Cref{sec_app5_comparerho12}). 
\end{itemize}
\end{remark}

%%%%%%%%%%%%%%%%%% HIGHEST ACHIEVABLE ENDS ######

\subsection{Dataset example with illustrations}
\label{sec_sewa}

\begin{figure}[!tb]
	%\begin{center}
	\centerline{
		\includegraphics[width=13cm]{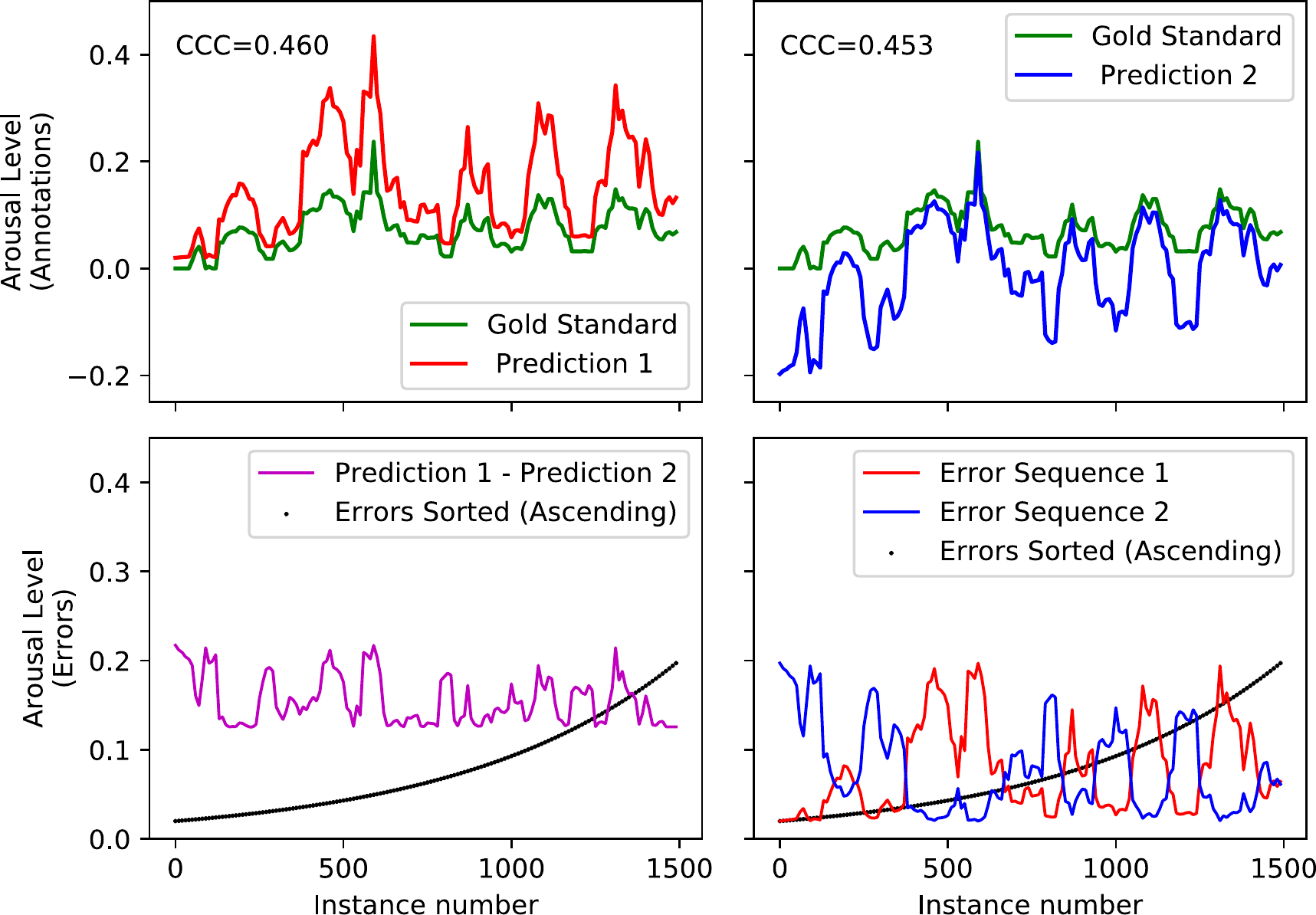}
	}
	%\end{center}
	\caption{A reproducible illustration
		of an identical error-set resulting in different performance in terms of $\rho_c$ through reordering. 
		The gold standard sequence used (in green) is
		the \textit{Train\_04} arousal annotation from the \emph{SEWA/AVEC'17}  database \citep{Kossaifi19-SDA,Ringeval17-POT}.
		The two error sequences, the sorted error sequence,
		and the difference between the two predictions are in the bottom row of the plots.
		The errors are ordered such that they maximise $\rho_c$ (as per  \Cref{orderErr1,orderErr2}), shown in red and blue respectively in the top row.  
		The resulting two $\rho_c$ are not drastically different in the example above. Because $MSE$ is large enough in this particular case,
		with a different permutation, a close-to-zero or a negative $\rho_c$  can be obtained as well (cf.
		\urlstyle{rm} \url{https://github.com/vedhasua/mse_ccc_corollary/}). 
		\label{figsewa}
	}
\end{figure}

%\Cref{figsewa} illustrates  what it means to `permute the errors', and kinds of prediction sequences this process translates to when we use the two different $\rho_c$ reformulations from \Cref{eqnrhoc1def,eqnrhoc2def}. 
\Cref{figsewa} illustrates  how the two different optimal permutations of the same error-set, as per the $\rho_c$ formulations from \Cref{eqnrhoc1def,eqnrhoc2def} translate to two different prediction sequences (\ie $P=G+E$ and $P=G-E$), and thus two different $\rho_{c_max}$. 
The database we used to generate 
\Cref{figsewa}
is the gold standard sequence of \textit{Train\_04}  arousal levels from the \emph{SEWA/AVEC'17} database \citep{Kossaifi19-SDA,Ringeval17-POT}. The database has recently gained a lot of popularity in the affective computing community, as it was used in all of the recent Audio/Visual Emotion Challenges (AVEC) \citep{Ringeval17-POT, Ringeval18-A2W, Ringeval19-A2W}, where $\rho_c$ was the decision criteria for recognising the winning submission. 
Note that the results presented are fully reproducible with an easy-to-use scripts we have made available at \urlstyle{rm}\url{https://github.com/vedhasua/mse_ccc_corollary/} that use a publicly available \emph{SEWA/AVEC'17} dataset, or by using \emph{any} ordinal data that one can own or can generate.
%The presented results are fully reproducible through easy-to-use scripts we have made available at \urlstyle{rm}\url{https://github.com/vedhasua/mse_ccc_corollary/} utilising a publicly available dataset, or using any ordinal, randomly generated data.

%The error-set we generated for the illustration purposes is plotted in the subplot at the bottom right corner of \Cref{figsewa} in black, upon sorting. 
%In each of these subplots, the green line-plot represents the gold standard.
The green line-plot represents the gold standard sequence always.
Different error sequences give rise to blue, red and black line-plots in the bottom right subplot, of which the black one represents the sorted error-set . `Prediction 1' and `Prediction 2' correspond to  $\rho_{c_{max_1}}$ and $\rho_{c_{max_2}}$ (cf. \Cref{eqnrhoc1def,eqnrhoc2def}). For the sake of simplicity and clarity in the illustration, we chose the error values to be strictly positive, where the minimum of the errors is close to zero. Observing the locations of the tiniest of the errors, we note that the first prediction sequence attempts to closely follow the lower values in the gold standard, while the latter closely follows the bigger values, 
consistent to the discussion in \Cref{sec_paradox}.
%due to the error rearrangement discussed in \Cref{sec_paradox}. 

We also note that the two prediction sequences are quite similar in shape to one another, but are far from being identical. The two are non-linearly stretched in the vertical direction, \ie owing to the reordering of different errors. The second row of plots provides a better insight into this vertical stretching. Comparing the magenta and the green-coloured plot, it can be seen that the difference between the two prediction sequences is the highest at the extremities of the gold standard. This is expected, 
%as per the \Cref{orderG,orderErr1,orderErr2},
since the smallest error is associated with either the smallest (\Cref{orderErr1}) or the biggest (\Cref{orderErr2}) sample.
% of the gold standard samples.
% \input{tb_sewa.tex}
% \newpage
\section{Additional loss functions and the interpretations}
\label{sec_cost}

\Cref{eqn_cccmsemap} dictates the need for minimisation the $L_p$ norm of the error values, \ie $MSE=N\cdot \sum _{j=1}^{N}(g_j-p_j)^{2}$,
while simultaneous maximisation of the $\operatorname{DotProduct}(G,P)=DP={\sum _{i=1}^{N}g_{i}p_i}$ (aka Hadamard product) to achieve $\rho_c$ maximisation. 
A family of loss functions, such as the following, can be easily designed.
\begin{alignat}{6}
&%\nonumber%%
\label{eqncost1}
&&f(g_i,p_i)&=&\frac{\sum _{j=1}^{N}(g_j-p_j)^{2}}{\sum _{j=1}^{N}g_{j}p_j}, 
\\
%\nonumber
\label{eqncost2}
&\hspace{0cm}
\text{more generally, } &&f(g_i,p_i)  &=& \left|\frac{\sum _{j=1}^{N}(g_j-p_j)^{2}}{\sum _{j=1}^{N}g_{j}p_j}\right|^{\gamma}&\text{where }&\gamma>0,
\\
%\nonumber
\label{eqncost3}
&\hspace{0cm}
\text{even more generally, } &&f(g_i,p_i)  &=& \left|\frac{\sum _{j=1}^{N}\varepsilon_j(g_j-p_j)^{2}}{\sum _{j=1}^{N}\alpha_i (g_{j}p_j)^{2\beta_j+1}}\right|^{\gamma}&\text{where }&\alpha_j,\beta_j,\varepsilon_j,\gamma>0, \beta_{j}\in\mathbb{N}.
\\
\label{eqncost4}
%\nonumber%%
%\label{eqncost1}
%\hspace{-2cm}
&\text{Or } &&f(g_i,p_i)&=&{\sum _{j=1}^{N}(g_j-p_j)^{2}-\alpha\sum _{j=1}^{N}g_{j}p_j}, &\text{where }&\alpha>0,\\
\label{eqncost5}
&\text{more generally, } &&f(g_i,p_i)&=&\left|{\sum _{j=1}^{N}(g_j-p_j)^{2}-\alpha\sum _{j=1}^{N}(g_{j}p_j)^{2\beta+1}}\right|^\gamma&\text{where }&\alpha,\beta,\gamma>0, \beta\in\mathbb{N},\\
%&\text{or more generally, } &&f(g_i,p_i)&=&{\sum _{j=1}^{N}(g_j-p_j)^{2}-\alpha\sum _{j=1}^{N}(g_{j}p_j)^{2\beta+1}}, &\text{where }&\alpha,  \beta >0, \beta\in\mathbb{N},\\
%&\text{or even more generally, }&&f(g_i,p_i)&=&{\sum _{j=1}^{N}(g_j-p_j)^{2}-\sum _{j=1}^{N}\alpha_i(g_{j}p_j)^{2\beta_{j}+1}}, &\text{where }&\alpha_j,  \beta_j >0,  \beta_j\in\mathbb{N}.\\
\label{eqncost6}
&\text{even more generally, }&&f(g_i,p_i)&=&\left|{\sum _{j=1}^{N}\varepsilon_j(g_j-p_j)^{2}-\sum _{j=1}^{N}\alpha_i(g_{j}p_j)^{2\beta_{j}+1}}\right|^\gamma&\text{where }&\alpha_j,  \beta_j,\varepsilon_j,\gamma >0,  \beta_j\in\mathbb{N}.
\end{alignat}

%\begin{alignat}{6}
%&%\nonumber%%
%%\label{eqncost1}
%&&f(g_i,p_i)&=&\frac{\sum _{j=1}^{N}(g_j-p_j)^{2}}{\sum _{j=1}^{N}g_{j}p_j}, 
%\\
%%\nonumber
%&\hspace{0cm}
%\text{or more generally, } &&f(g_i,p_i)  &=& \left|\frac{\sum _{j=1}^{N}(g_j-p_j)^{2}}{\sum _{j=1}^{N}g_{j}p_j}\right|^{\gamma},&\text{where }&\gamma>0.
%\\
%% \label{eqnCost}
%%\nonumber%%
%%\label{eqncost1}
%%\hspace{-2cm}
%&\text{Or }
%&&f(g_i,p_i)&=&{\sum _{j=1}^{N}(g_j-p_j)^{2}-\alpha\sum _{j=1}^{N}g_{j}p_j}, &\text{where }&\alpha>0,\\
%%\nonumber
%%\hspace{-2cm}
%&\text{or more generally, }
%%\span\span\span\span\span\span\span\span\span
%%\\
%%\label{eqncost2}
%%\nonumber%%
%&&f(g_i,p_i)&=&{\sum _{j=1}^{N}(g_j-p_j)^{2}-\alpha\sum _{j=1}^{N}(g_{j}p_j)^{2\beta+1}}, &\text{where }&\alpha,  \beta >0, \beta\in\mathbb{N},\\
%%\nonumber
%%\hspace{-2cm}
%&\text{or even more generally, }
%%\span\span\span\span\span\span\span\span\span
%%\\
%%\label{eqncost3}
%%\nonumber%%
%&&f(g_i,p_i)&=&{\sum _{j=1}^{N}(g_j-p_j)^{2}-\sum _{j=1}^{N}\alpha_i(g_{j}p_j)^{2\beta_{j}+1}}, &\text{where }&\alpha_j,  \beta_j >0,  \beta_j\in\mathbb{N}.
%\\
%%\nonumber
%%\hspace{-2cm}
%&\text{or even more generally, }
%%\span\span\span\span\span\span\span\span\span
%%\\
%%\label{eqncost3}
%%\nonumber%%
%&&f(g_i,p_i)&=&\left|{\sum _{j=1}^{N}(g_j-p_j)^{2}-\sum _{j=1}^{N}\alpha_i(g_{j}p_j)^{2\beta_{j}+1}}\right|^\gamma, &\text{where }&\alpha_j,  \beta_j,\gamma >0,  \beta_j\in\mathbb{N}.
%\end{alignat}

A loss function, attempting to maximise $\sum_{i=1}^{N}g_ip_i$, \ie the dot product between the predictions and the gold standard, makes sense intuitively as well. We essentially dictate the model to raise the prediction values as large as possible when dealing with large values in the gold standard sequence, and diminish the predictions to as small as possible corresponding to the smaller values in the gold standard sequence. The $MSE$ metric as a sub-component of the derived loss function, too, attempts to achieve this. However, the inner workings differ slightly. The $MSE$ component drives the weights of the model to readjust by the amount that is proportional to the error seen, through back-propagation. The dot product readjusts the weights by an amount that is proportional to the gold standard itself. 
%Another way to look at the dot product term is that, 
From another perspective, we effectively `weigh' the individual errors by the corresponding 
desired outputs.
%gold standards. 
% Note that the cost in the \Cref{eqncost1} can be effectively reduced to $\displaystyle\sum _{i=1}^{N}\left(g_i-\frac{p_i}{\gamma}\right)\left(g_i-{\gamma}{p_i}\right) \exists\gamma>0$, which is zero when $\displaystyle g_i=\frac{p_i}{\gamma}$ or $g_i={\gamma}{p_i} \forall i\in\mathbb{N}:i\in[1,N]$. That is, when prediction is directly proportional to the gold standard.

%This same effect () can be achieved by using
\subsection{Role of the constants $\alpha, \beta, \alpha_j,  \beta_j,\varepsilon_j,\gamma$ in \Cref{eqncost1,eqncost2,eqncost3,eqncost4,eqncost5,eqncost6}}
While it is possible to use the deviation from
the maximally achievable $\rho_c$ (\ie $1-\rho_c$) as the loss function  \citep{Trigeorgis16-AFE,Schmitt19-CER,Pandit20-ISI}, the joint variability of the prediction and the gold standard can also be simultaneously improved using
$\frac{MSE}{\sigma_{XY}}$ as the loss function directly (cf. \Cref{eqn_cccmsemap}), \ie using \Cref{eqncost1} as the loss function. $\rho_c$ maximisation could then be expected through simultaneous $MSE$ minimisation with $\sigma_{XY}$ maximisation. However, one drawback of $\frac{MSE}{\sigma_{XY}}$  as the loss function has been witnessed \citep{Pandit19-IKH}; \ie the neural network may `cheat' by simply making $\sigma_{XY}$ more and more negative, without attempting $MSE$ minimisation. To avoid this, one can alternately use $\left|\frac{MSE}{\sigma_{XY}}\right|^\gamma,\gamma>0$, \ie \Cref{eqncost2} as the loss function. While with $\left|\frac{MSE}{\sigma_{XY}}\right|^\gamma$ as the objective function, $\sigma_{XY}$ itself is still not guaranteed to be non-negative, the network now needs to maximise $|\sigma_{XY}|$ (instead of minimising $\sigma_{XY}$), while simultaneously minimising $MSE$.
In other words, the key difference between the two functions is that, there is no lower bound for $\frac{MSE}{\sigma_{XY}}$. 
Thus, the ability of the network to `cheat' gets rather restricted when using $\left|\frac{MSE}{\sigma_{XY}}\right|^\gamma$, since a highly negative $\sigma_{XY}$ (\ie a highly negative correlation) can only be obtained with high magnitude of errors,
%  . This is because only the second term of \Cref{eqncovmax} can make the sum go negative,
effectively resulting in a high $MSE$, thus a higher loss. This loss function too has been also proven to be successful in practice \citep{Pandit19-IKH}. 

%Minimisation of $MSE$ and maximisation of $\sigma_{XY}$ 
Maximisation of $\rho_c$
may also be encouraged by using a difference function such as the one in \Cref{eqncost4} as the loss function, where $\alpha$ denotes relative importance assigned to maximisation of $\sigma_{XY}$ compared to minimisation of $MSE$. The dot product $g_jp_j$ can be raised to an odd integer exponent $2\beta+1$ to retard or to expedite $\sigma_{XY}$ minimisation process (cf. \Cref{eqncost5}). Finally, the individual sample-level loss components can be emphasised using scalars $\alpha_j,\beta_j,\varepsilon_j$ (cf. \Cref{eqncost3,eqncost6}).

\section{Concluding remarks}
\label{sec_concl}
	We derive the missing, yet fundamentally crucial many-to-many mapping existing between $\rho_c$ and $MSE$  -- both arguably the most popular utility functions in use today.
%
%	\item 
	We discover the conditions necessary achieving the minimum and the maximum possible $\rho_c$ at any given $MSE$, 
%	and the formulations for the same in terms of $MSE$ and $\sigma_{{XY}}$. %(and for any $L_p$ norm, where $p=1$ or $p\geq2, p\in \mathbb{R}$). 
%
%	\item
%	We discover also the formulations for $\rho_c$ when $MSE$ is minimised and maximised, 
	and likewise, for any
	given fixed $L_p$ norm for any real-valued $p$ greater than 0, through $MSE$ optimisation. 
	As the research community has often witnessed;
while the two processes -- namely the error $L_p$ norm minimisation and the $\rho_c$ maximisation -- 
%	are not 
%	necessarily identical, nor necessarily convergent
are elusively directed at the same goal, we conclusively prove that the two are neither necessarily identical, 
%	and can even be divergent
nor necessarily convergent 
%	always direct the model to an exact outcome 
-- except at the very extremity, where the ideal set of zero-valued errors translates to a perfect identity relationship, \ie $\rho_c=1$. 
%	-- often unattainable due to degrees of freedom available to a model or noise in the data. 
%Through the equations derived via $MSE$ optimisation approach, 
In other words, we prove mathematically the reasons for the witnessed inconsistency; the  observation often reported in the literature lately \citep{Trigeorgis16-AFE,Pandit19-IKH,atmaja2020evaluation}.
%	\item
	We generalise the $\rho_c$ optimisation strategy to any given error-$L_p$ norm, for any $p$ that is even and not necessarily $p=2$.

%	\item
	While we establish that not just $L_k$-norm value, but rather the \emph{distribution} of the errors making $L_k$ dictates the $\rho_c$ metric performance, we prove further that 
	%	even upon knowing fully the
	even with a full knowledge of  the
	set of error values (\ie the error distribution), 
	the $\rho_c$  metric still can not be conclusively computed. 
	%	even with full information of the error-set 
	%	based on the available error-set information alone. 
	%	\item
	The order (\ie the correspondence of the error values with respect to the ground truth) also governs $\rho_c$ metric. We establish the conditions for which $\rho_c$ is maximised and minimised when given a fixed set of errors.
	There exist two different notions of what the `error' actually means (\ie whether the prediction minus the gold standard, or the other way around), we derive the $\rho_c$ optimisation formulations for both these cases.
	%	-- for the two different notions as to what the `error' actually means (\ie whether the prediction minus the gold standard, or the other way around), we establish the conditions for which $\rho_c$ maximisation is witnessed in both these cases. 

	Keeping the deep learning models in perspective 
	-- that are capable of mapping even complex and non-linear input-to-output relationships --
	we propose a family of new loss functions, some of which have been tested and have been proven to work. With these loss functions, the models can be aimed to maximise $\rho_c$, with simultaneous minimisation of the errors and maximisation of the joint variability of the  prediction with respect to the desired output. 
	% 	\end{itemize} 
	%
%
%	\item 
	The proposed loss functions consist of two components. First component is the classical loss function $MSE$ that reduces the difference between the prediction and the gold standard by an amount that is proportional to the error itself. Our newly introduced dot product or covariance-based component tunes 
	the parameters of the model 
	%	closer to the corresponding gold standard 
	by an amount that is proportional to the value of the gold standard itself, through backpropagation.  
	%	We derived these properties of the desired loss function by reformulating the $\rho_c$ measure in terms of the error coefficients, \ie the difference between the prediction and the gold standard population, in two different ways. 
%	
%\item 
	Through the rigorous derivation of the formula for the many-to-many mapping that exists between $MSE$ and $\rho_c$, we also propose a rather more elegant loss function, which is simply the positive power of the absolute value of the ratio of $MSE$ to $\sigma_{XY}$, \ie $\left|\frac{MSE}{\sigma_{XY}}\right|^\gamma,\gamma>0$.
	%\item While we have derived the condition for $\rho_c$ optimisation for a given $L_p$ (cf. \Cref{eqnsolvethispoly}), the investigation of properties of span of $\{\rho_c, L_p\}$ in the $(\rho_c,L_p)$ and in the $(\rho_c,L_2)$ space ($p \neq 2$) remains an unsolved problem; a possible future research direction.
%\end{itemize} 

While we have derived the condition for $\rho_c$ optimisation for a given $L_p$ (cf. \Cref{eqnsolvethispoly}), the investigation of properties of span of $\{L_p,\rho_c\}$ in the $(L_p,\rho_c)$ and in the $(L_2,\rho_c)$ space ($p \neq 2$) remains an unsolved problem; a possible future research direction.

% Acknowledgements should go at the end, before appendices and references

%\acks{ The authors thank the referees, the associate editor, and the editor for their valuable time, and for their constructive inputs improving the quality of the manuscript. }


%\newpage

\appendix

\section{Proof of \Cref{thm42}}
\label{sec_app0_rhomin}
\begin{proof}
	\noindent To minimise $\rho_c$,
	according to \Cref{eqn_cccmsemap,eqncovmax,yzdef},
	%formally speaking, 
	we need to 
	\begin{align}%{2}
%	\nonumber%%
	\text{maximise: }
%	\quad 
%	&
%	f(d_1,d_2,\cdots,d_N)
	f(\{d_i\})
%	&&
	= -N\sigma_{XY}
	=
	-\sum_{i=1}^{N}{y_{z_i}}^2
	-\sum_{i=1}^{N}d_i{y_{z_i}},
	%	-\frac{1}{N}\sum_{i=1}^{N}\left(\left(\sum_{j=1}^{N}d_j\right){y_{z_i}}\right)
	%	\nonumber
%	\\
	\quad	\quad
	\text{subject to: }
%	\quad 
%	&
%	g(d_1,d_2,\cdots,d_N)
	g(\{d_i\})
%	&&
	= N\cdot MSE-\sum_{i=1}^{N}d_i^2=0.\nonumber
	\end{align}
	Auxiliary Lagrange function	${\mathcal {L}}(d_1,d_2,\cdots,d_N,\lambda )=f(d_1,d_2,\cdots,d_N)-\lambda \cdot g(d_1,d_2,\cdots,d_N)$ is given by:
	%Applying ordinary Lagrange multiplier method,
	%we introduce auxiliary Lagrange expression, defined by
	\begin{align}%{2}
%	{\mathcal {L}}(d_1,d_2,\cdots,d_N,\lambda )&=f(d_1,d_2,\cdots,d_N)-\lambda \cdot g(d_1,d_2,\cdots,d_N),\\
	%&=-N\sigma_{XY}
	%-\lambda \left(N\cdot MSE-\sum_{i=1}^{N}d_i^2\right)\\
	\nonumber%%
	\mathcal {L}&=-\sum_{i=1}^{N}{y_{z_i}}^2
	-\sum_{i=1}^{N}d_i{y_{z_i}}
	-\lambda \left(N\cdot MSE-\sum_{i=1}^{N}d_i^2\right).\\
	\therefore\nabla_{d_1,d_2,\cdots,d_N,\lambda}
	{\mathcal {L}}
	&=0 
	\Leftrightarrow 
	{\begin{cases}
		{-y_{z_i}}
		+2\lambda d_i=0 \quad \quad \forall  i \in \mathbb{N} : i \in [1,N],
		\\
		N\cdot MSE-\sum_{i=1}^{N}d_i^2=0.
		\end{cases}}\label{lagrangecondsn}
		\\
	%	
%	\end{align}
%	
	%We note that the conditions on $\{\lambda, d_1,d_2\dots d_i,\dots,d_n\}$ from \Cref{lagrangecondsn} are identical to those from \Cref{lagrangecondsp}.
	%We note that  \Cref{lagrangecondsn} is identical to \Cref{lagrangecondsp}. Thus, we arrive at the same solution for distribution of errors in $D=(d_i)_1^N$ as given by \Cref{MSEDistr}. That is,
	%Because \Cref{lagrangecondsn} is identical to \Cref{lagrangecondsp},
%\text{
%Because \Cref{lagrangecondsn,lagrangecondsp} are identical,
%%we arrive at the same $D=(d_i)_1^N$ 
%from \Cref{MSEDistr}:
%}\span\nonumber\\
%	\begin{align}
%		\span\span
%	\hspace{-30pt}
	\therefore
	d_i=\mp\sqrt{\frac{N\cdot MSE}{\sum_{j=1}^{N}{y_{z_i}}^2}}\cdot{y_{z_j}}
	=\mp\sqrt{\frac{MSE}{\sigma_G^2}}\cdot{y_{z_j}}
	\text{ $\because$
	\cref{lagrangecondsn,lagrangecondsp} are identical $\implies$
	%we arrive at the same $D=(d_i)_1^N$ 
	\cref{MSEDistr}.}
	\span\span
	\label{MSEDistrN}\\
%	\end{align}
	%	Clearly, from \Cref{eqncovmax,MSEDistrN},
\text{
	$N\sigma_{XY}$ is minimised when  $d_i$ and ${y_{z_j}}$ have opposite signs, and when
}\span\nonumber\\ 
%	\begin{align}
	\nonumber%%
%	\Aboxed{
		d_i&=-\left|\sqrt{\frac{N\cdot MSE}{\sum_{j=1}^{N}{y_{z_i}}^2}}\right|\cdot{y_{z_j}}
		=-\left|\sqrt{\frac{MSE}{\sigma_G^2}}\right|\cdot{y_{z_j}}
%	} 
\hspace{1cm} \because \text{ \Cref{eqncovmax,MSEDistrN}}.
\\
%	\end{align}
%	Thus, $\rho_c$ is minimised when $MSE$ is composed of the errors (\ie \{$d_i$\}) that are equally negatively proportional to the corresponding deviations of the gold standard from the mean value (\ie $\{y_{z_i}\}:=\{y_i-\mu_{Y}\}$).
%	%	and have signs opposite to those deviations  (\ie signs of $\{y_{z_i}\}$) correspondingly.
%	If the square-root sign denotes a positive square root, 
%	%	we get from \Cref{eqncovmax}:
%	\begin{align}
	\nonumber%%
	\therefore
	\sigma_{{XY}_{min}}
	&=\frac{1}{N}\left(\sum_{i=1}^{N}{y_{z_i}}^2\left(1-\sqrt{\frac{MSE}{\sigma_G^2}}\right)\right)
	%\\
	%&
	%	=\frac{1}{N}\left(N\cdot \sigma_G^2\left(1-\sqrt{\frac{MSE}{\sigma_G^2}}\right)\right)\\
	%	&={\sqrt{\sigma_G^2+MSE}}\cdot{\sqrt{\sigma_G^2}}
	%	&
	=\sigma_G^2-{\sqrt{\sigma_G^2\cdot MSE}}
	\quad
	\text{ $\because$ \Cref{eqncovmax}}.\\
	%	\end{alignat}
	%	Note that, we get the exact same conditions for $\rho_c$ minimisation if we were to express $\sigma_{XY}$ in \Cref{Ncovdef0} purely in terms of $x_i$, instead of $y_i$.
	%	
	%	Correspondingly, 
	%	Thus, from \Cref{eqn_cccmsemap},
	%	\begin{align}%{2}
	%	\therefore
%\nonumber%%
%\therefore
%\rho_{{c}_{min}}
%&=\left(1+\frac{MSE}{2\cdot(\sigma_G^2-{\sqrt{\sigma_G^2\cdot MSE}})}\right)^{-1}
%%\\
%%&
%=\frac{2\cdot(\sigma_G^2-{\sqrt{\sigma_G^2\cdot MSE}})}{MSE+2\cdot(\sigma_G^2-{\sqrt{\sigma_G^2\cdot MSE}})},\\
	\nonumber%%
	\therefore
	\rho_{{c}_{min}}
	&=\frac{2\cdot\left(1-{\sqrt{\frac{MSE}{\sigma_G^2}}}\right)}{\frac{MSE}{\sigma_G^2}+2\left(1-{\sqrt{\frac{MSE}{\sigma_G^2}}}\right)}
	%	\\
	% &=\frac{2-2\cdot{\sqrt{\frac{MSE}{\sigma_G^2}}}}{\frac{MSE}{\sigma_G^2}+2-2{\sqrt{\frac{MSE}{\sigma_G^2}}}}\\
	\implies
	\hspace{0.2cm}
%	\boxed{
		\rho_{{c}_{min}}
		=\frac{2\left(1-{\sqrt{\frac{MSE}{\sigma_G^2}}}\right)}{1+\left(1-{\sqrt{\frac{MSE}{\sigma_G^2}}}\right)^2}.
%	}
	%	\span
	\end{align}
	\vspace{-1cm}
\end{proof}
%\vspace{-2cm}
\vspace{-1cm}
\section{Proof of \Cref{eqnl2rellek}}
\label{sec_app01_kl2}
\begin{proof}
%	From \Cref{sec_app3},
	H\"older's inequality (cf. \Cref{sec_app3}) leads to the following famous identity for $0<r<p$:
	\vspace{-12pt}
	\begin{align}
	\nonumber%%
	&
%	\hspace{60pt}
	L_p && \leq  L_r  && \leq  N^{\frac{p-r}{pr}} L_p .
	\\
		\nonumber%%
		\therefore\hspace{3pt}
		&L_2 &&\leq L_k &&\leq N^{\frac{2-k}{2k}} \cdot L_2, \\
		\nonumber
		\implies\hspace{3pt}
		&MSE^{\frac{1}{2}} &&\leq N^{\left(\frac{1}{k}-\frac{1}{2}\right)} \cdot MkE^{\frac{1}{k}} &&\leq N^{\frac{2-k}{2k}} \cdot MSE^{\frac{1}{2}},\\
		\nonumber
		%\implies\hspace{16pt}
		\text{\ie } \hspace{3pt}
		% \Aboxed{
		&
		MkE^{\frac{1}{k}} 
		&& \leq MSE^{\frac{1}{2}} 
		&& \leq N^{\frac{2-k}{2k}} \cdot MkE^{\frac{1}{k}}. 
	\vspace{-1.5cm}
	\end{align}
\end{proof}
%\newpage

\vspace{-1cm}
\section{\Cref{fig_maxminrhoLp}-equivalent for $k<2$ (\eg for a given $MAE$, \ie for $k=1$)}
\label{sec_app_mae}
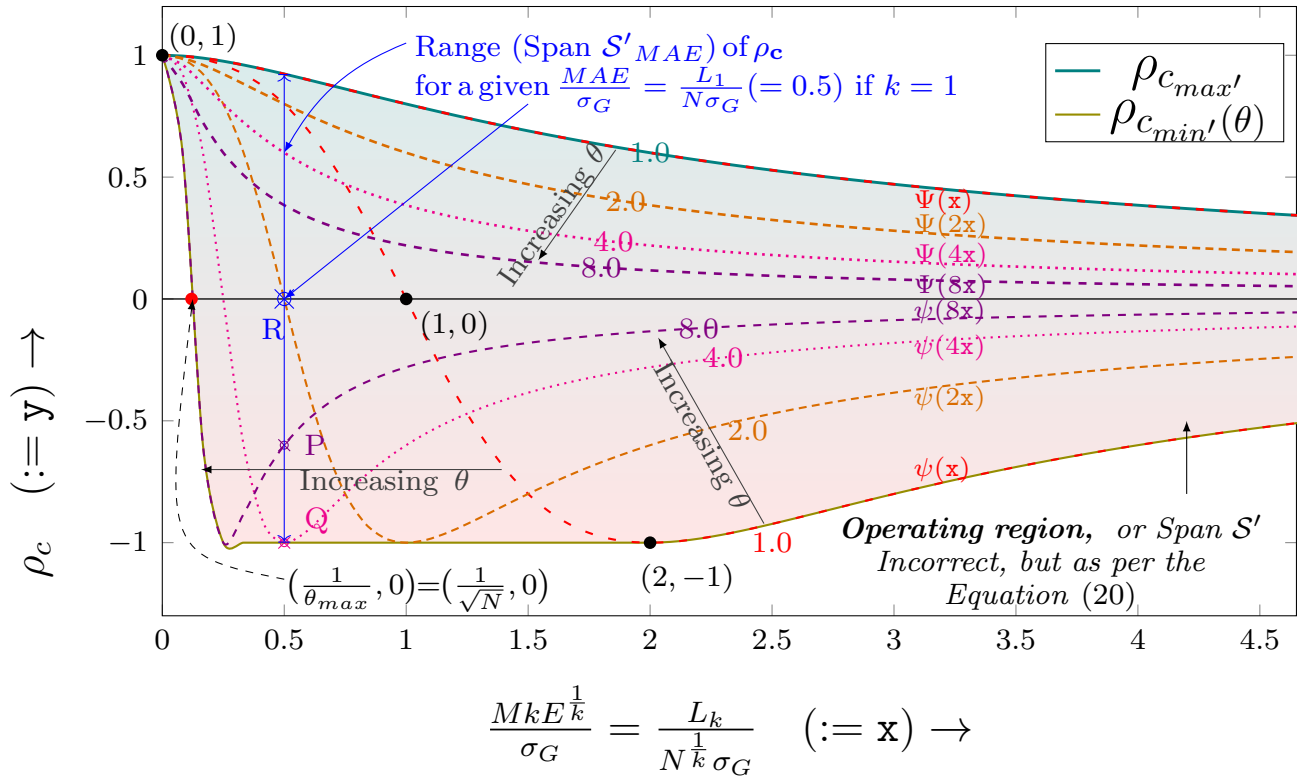
\begin{figure}[!hb]
	%	\vspace{-4cm}
	\centering
	\begin{tikzpicture}
	[every node/.style={scale=1.4},
	declare function={
		func(\x)= (\x <= 0.25) * 2*((1-8*x))/(1+(1-8*x)^2)   +
		and(\x > 0.25, \x < 2) * (-1)     +
		(\x >= 2) * (2*(1-x)/(1+(1-x)^2))
		;
	}]
	\begin{axis}[
	scale=2.2,
	domain=0:50,
	samples=601,
	smooth,
	no markers,
	xlabel={$
		%		\mathtt{x}=
		\frac{MkE^{\frac{1}{k}}}{{\sigma_G}} 
%		\quad
%		\left(
		=\frac{L_k}{N^{\frac{1}{k}}\sigma_G}
%			= \sqrt{\frac{MAE^2}{{\sigma_G}^2}} 
%	%		\quad
%	%		\left(
%			=\frac{L_1}{N{\sigma_G}}
		\quad
		(:= \mathtt{x})
%		\right)
		\rightarrow$   },
	ylabel={$
		%		\mathtt{y}=
		\rho_c 
		\quad
		(:= \mathtt{y})
		\rightarrow$},
	xmin=0,xmax=4.65,
	%extra x ticks={1, 2},
	ymin=-1.3,ymax=1.2,
	unit vector ratio*=0.8 0.8,
	x label style={below,font=\large},
	y label style={left,font=\large},
	legend style={
		%		at={(0.65,0.6)},
		%		at={(axis cs: 5.8,1.0)}, 0 to 6
		at={(axis cs: 4.6,1.05)},
		anchor=north east, font=\fontsize{14}{5}\selectfont}, 
	%  legend image post style={size=20pt}, 
	]
	\addplot 
	+[name path=rhomax, line width=1.1pt,teal] 
	{2*(1+x)/(1+(1+x)^2)}; 
	\label{graphkmaxA}
	\addplot +[name path=rhominall,olive,thick, solid] {func(x)};
	\addplot 
	+[name path=rhomin1, line width=1.1pt, red, thick, loosely dashed] 
	{2*(1-x)/(1+(1-x)^2)};
	\label{graph1kA},
	\addplot 
	+[name path=rhomin2, line width=0.8pt, black!15!orange, densely dashed] 
	{2*(1-2*x)/(1+(1-2*x)^2)};
	\label{graph2kA},
	%	\addplot 
	%	+[name path=rhomin3, line width=0.8pt, violet, dotted] 
	%	{2*(1-3*x)/(1+(1-3*x)^2)};
	%	\label{graph3kA},
	\addplot 
	+[name path=rhomin5, line width=0.8pt, magenta, dotted] 
	{2*(1-4*x)/(1+(1-4*x)^2)};
	\label{graph4kA},
	\addplot 
	+[name path=rhomin8, line width=0.8pt, violet, dashed] 
	{2*(1-8*x)/(1+(1-8*x)^2)};
	\label{graph8kA},	
	\addplot[shade, top color=teal!15, bottom color=red!10] fill between[of=rhomax and rhominall];
	%%%%%%% rhomax curves
	\addplot 
	+[name path=rhomin2, line width=1pt, red, loosely dashed] 
	{2*(1+x)/(1+(1+x)^2)};
	\label{graph1kmaxe},
	\addplot 
	+[name path=rhomin2, line width=1pt, black!15!orange, densely dashed] 
	{2*(1+2*x)/(1+(1+2*x)^2)};
	\label{graph2kmaxe},
	\addplot 
	+[name path=rhomin2, line width=1pt, magenta, dotted] 
	{2*(1+4*x)/(1+(1+4*x)^2)};
	\label{graph5kmaxe},
	\addplot 
	+[name path=rhomin2, line width=1pt, violet, dashed] 
	{2*(1+8*x)/(1+(1+8*x)^2)};
	\label{graph10kmaxe},
	%%%%%% Arrows for positive	
	%		\node (E) at (35, 14.8+4) {}; %0 to 6
	%		\node (F) at (28, 10.2+4) {}; %0 to 6
	%		\node (E) at (axis cs: 2.4, 0.6) {}; %0 to 4.5 (axis cs: 2.5, 0.5283)
	%		\node (F) at (axis cs: 2.0, 0.05) {}; %0 to 4.5 (axis cs: 2.0, 0.1172) 
	\node (E) at (axis cs: 1.9, 0.67) {}; %0 to 4.5 (axis cs: 2.5, 0.5283)
	\node (F) at (axis cs: 1.5, 0.1) {}; %0 to 4.5 (axis cs: 2.0, 0.1172) 
%	arrows
	\path[-{latex[scale=3.0]}, draw=darkgray]
	(E) edge node[sloped, anchor=left, above, xshift=-6pt,yshift=-4pt, darkgray] {Increasing 
		$\theta$
	}  (F);
	\node[
	label={[violet]center
		:8.0}
	] at 
	%			(31.2-2.8,9.8+3+2) {}; %0 to 6
	(axis cs: 1.8,0.1293) {};   %0 to 4.5
	\node[
	label={[magenta]center
		:4.0}
	] at 
	%			(34.2-2.5,10.3+3+2) {};%0 to 6
	(axis cs: 1.85,0.235) {};   %0 to 4.5
	\node[
	label={[black!15!orange]center:2.0
	}
	] at 
	%			(33.9-1.8,11.7+3+2) {};%0 to 6
	(axis cs: 1.9,0.4) {};   %0 to 4.5
	\node[
	label={[teal]center
		:1.0}
	] at 
	%			(35.6-2.8,13.4+3+2) {};%0 to 6
	(axis cs: 2,0.6) {};   %0 to 4.5
	%%%%%%%
	\addplot [black, solid, line width=0.5]
	{0};
	\pgfplotsset{
		after end axis/.code={
			\node[fill,circle,inner sep=1.2pt,label={[xshift=11pt, yshift=-5pt]$(0,1)$}] at (axis cs: 0,1) {};
		}
	}
	\node[fill,circle,inner sep=1.2pt,label={[xshift=10pt]below:$(2,-1)$}] at (axis cs: 2,-1) {};
	\node[fill,circle,inner sep=1.2pt,label={[xshift=13pt, yshift=-17pt]$(1,0)$}] at (axis cs: 1,0) {};
	\node[fill,red, circle,inner sep=1.2pt,label={[xshift=9pt, yshift=-3pt]above:
		%	$\big(\frac{1}{\theta_{max}},0\big)$
		%		$(\sfrac{1}{\sqrt{N}},0)$
	}] at (axis cs: 0.12	,0) {};
	\node[label={
		[xshift=0pt, yshift=0pt,rotate=0,
		align=left, black]
		%$\big(\frac{1}{\theta_{max}},0\big)$
		$\big(\frac{1}{\theta_{max}},0\big)\hspace{-3pt}=\hspace{-3pt}
		\big(\frac{1}{\sqrt{N}},0\big)$
	}] 
	%	at (axis cs: 1.0,-1.35) {}; 
	at (axis cs: 1.05,-1.4) {};
	%		at (axis cs: 0.8,0.85) {};
	\draw[-{latex[scale=5.0]},black,dashed] 	
	%	(axis cs: 0.75,-1.15)	.. 
	(axis cs: 0.5,-1.15)	
	%	to [bend left=75]
	%	(axis cs: 0.125,-1.15)	
	%	to [bend left=15]
	.. controls 	
	%	(axis cs: 0.0,-1.35)	.. 
	(axis cs: 0,-1.05)	.. 
	(axis cs: 0.125,0); 
	%%%%%%%%%%
	\node[fill,cross=2pt,violet,inner sep=1.2pt,label={[xshift=0pt, yshift=0pt,violet]right:P}] at (axis cs: 0.5,-0.6) {};
	\node[fill,cross=2pt,magenta,inner sep=1.2pt,label={[xshift=0pt, yshift=6pt,magenta]right:Q}] at (axis cs: 0.5,-1) {};
	\node[inner sep=1.2pt,label={[xshift=-3pt, yshift=0pt,blue]below:R}] at (axis cs: 0.5,0) {};
	\node[inner sep=1.2pt,label={[xshift=0pt, yshift=0pt,black!15!orange,font=\fontsize{7pt}{6pt}\selectfont,]right:$\Psi(2\mathtt{x})$}] at (axis cs: 3,0.3) {};
	\node[inner sep=1.2pt,label={[xshift=0pt, yshift=0pt,black!15!orange,font=\fontsize{7pt}{6pt}\selectfont, ]right:$\psi(2\mathtt{x})$}] at (axis cs: 3,-0.41) {};
	\node[inner sep=1.2pt,label={[xshift=0pt, yshift=0pt,magenta,font=\fontsize{7pt}{6pt}\selectfont,]right:$\Psi(4\mathtt{x})$}] at (axis cs: 3,+0.18) {};
	\node[inner sep=1.2pt,label={[xshift=0pt, yshift=0pt,magenta,font=\fontsize{7pt}{6pt}\selectfont, ]right:$\psi(4\mathtt{x})$}] at (axis cs: 3,-0.2) {};	
	\node[inner sep=1.2pt,label={[xshift=0pt, yshift=0pt,violet,font=\fontsize{7pt}{6pt}\selectfont,]right:$\Psi(8\mathtt{x})$}] at (axis cs: 3,+0.05) {};
	\node[inner sep=1.2pt,label={[xshift=0pt, yshift=0pt,violet,font=\fontsize{7pt}{6pt}\selectfont,]right:$\psi(8\mathtt{x})$}] at (axis cs: 3,-0.05) {};	
	\node[inner sep=1.2pt,label={[xshift=0pt, yshift=0pt,red,font=\fontsize{7pt}{6pt}\selectfont,]right:$\Psi(\mathtt{x})$}] at (axis cs: 3,0.4) {};
	\node[inner sep=1.2pt,label={[xshift=0pt, yshift=0pt,red,font=\fontsize{7pt}{6pt}\selectfont,]right:$\psi(\mathtt{x})$}] at (axis cs: 3,-0.7) {};
	\draw[violet] 	(axis cs: 0.5,-0.6) circle (1.5pt);
	\draw[magenta] 	(axis cs: 0.5,-1) circle (1.5pt);
	\node (A) at 
	%	(32, 3.0+3) {}; %0 to 6
	(axis cs: 2.5, -0.98) {}; %0 to 4.5 (axis cs: 2.5, -0.923)
	\node (B) at 
	%	(16, 10+3) {}; %0 to 6
	(axis cs: 2.0, -0.1) {}; %0 to 4.5 (axis cs: 2.0, -0.1327)
	% arrows
	\path[-{latex[scale=3.0]}, draw=darkgray]
	(A) edge node[sloped, anchor=left, below, xshift=0pt,yshift=4pt,darkgray] {Increasing 
		$\theta$	}  (B);
	\node (C) at (axis cs: 1.45, -0.7) {};
	\node (D) at (axis cs: 0.1,-0.7) {};
	% arrows
	\path[-{latex[scale=3.0]}, draw=darkgray]
	(C) edge node[sloped, anchor=right, above, xshift=9pt,yshift=-11pt, darkgray] {Increasing\,
		$\theta$	}  (D);
	\node[
	label={[violet]center
		:8.0}
	] at 
	%	(19,8+3+2) {}; %0 to 6
	(axis cs: 2.2,-0.12) {}; %0 to 4.5
	\node[
	label={[magenta]center
		:4.0}
	] at 
	%	(21.6,7.0+3+2) {};
	(axis cs: 2.3,-0.24) {}; %0 to 4.5
	\node[
	label={[black!15!orange]center
		:2.0}
	] at 
	%	(26.4,4.8+3+2) {};
	(axis cs: 2.4,-0.54) {}; %0 to 4.5
	\node[
	label={[red]center
		:1.0}
	] at 
	%	(32.2,1.5+3+2) {};
	(axis cs: 2.5,-1) {}; %0 to 4.5
	\node[label=
	{
		[text width=5cm, align=center]
		\textit{
			\textbf{
				Operating region,
			}
			or Span $\mathcal{S'}$\\[0mm]
			Incorrect, but as per the\\[0mm] \Cref{eqn_generalmaxmin}
		}
	}
	] 
	at 
	%	(47,-1.2) %0 to6
	(axis cs: 3.6,-1.4)
	{};  
	\draw[-{latex[scale=5.0]}] 
	%	(55,5.5+1)--(55,7.5+3+1); % 0 to 6
	(axis cs: 4.2, -0.8)--(axis cs: 4.2,-0.5); % 0 to 4.5
	%
	%%%%%%%% RANGE ADDED BEGINS%	
	\node[label={
		[xshift=0pt, yshift=0pt,rotate=0,
		align=left, blue]
		Range (Span $\mathcal{S'}_{MAE}$)\,of\,$\mathbf{\rho_c}$\,\\[0mm]for\,a\,given\,$\frac{MAE}{\sigma_{G}}=\frac{L_1}{N\sigma_{G}}
		(=0.5)$ if $k=1$ 
	}] 
	%at (20,10.8+3) {};
	%at (axis cs: 2.5,0.5) {}; % 0 to 6
	at (axis cs: 2.15,0.65) {}; % 0 to 4
	%%%%%%%%
	\node[fill,cross=3pt,blue,inner sep=1.2pt] at 
	%(3,14) {};
	(axis cs: 0.5,0) {};
	\draw[blue] (axis cs: 0.5,0) circle (2.5pt);
	%%%%%%%%
	\draw[blue] 
	[|<->|] 
	%(3,0.8+3)--(3,20.8+3); 
	(axis cs: 0.5,-1) -- (axis cs: 0.5,0.923);
	%%%%%%%%%% MSE Pointer2
	\draw[-{latex[scale=3.0]},blue] 	
	%(9.6,14+3).. controls (4,13+3).. (3,11+3);
	(axis cs: 1.5,0.8) to [bend right=0] (axis cs: 0.5,0);
	\draw[-{latex[scale=3.0]},blue] 	
	%(9.6,16.5+3) .. controls (5,16.5+3) .. (3,16.5+3);	
	(axis cs: 1,1.05) to [bend right=20] (axis cs: 0.5,0.6);	
	%%%%%%% RANGE ADDED ENDS	
	\legend{$\rho_{c_{max'}}$ ,$\rho_{c_{min'}(\theta)}$ }
	\end{axis}
	\end{tikzpicture}
	%\ref{graph1} $\rho_{c_{max}}$ \qquad \ref{graph2} $\rho_{c_{min}}$
	%	\vspace{-1.5cm}
	\caption{Range of $\rho_c$ as per \Cref{eqn_generalmaxmin}, for a given $MAE$ with respect to the gold standard consisting of 
		$N$ samples, standard deviation = $\sigma_G^2$,
		%		$\sigma_G^2$ = mean squared deviation of the gold standard, 
		$1\leq\theta\leq \sqrt{N}$.
		Note the increase in the operating region $\mathcal{S'}$ with a negative $\rho_c$ due to the allowed variation in $\rho_{c_{min'}}=\psi(\theta\cdot\mathtt{x})$ with increasing $\theta$. 
		While each point in $\mathcal{S'}$ maps to a unique $\left\{MAE,\rho_c\right\}$ pair, each maps to infinitely many $MSE$ values -- except those on the $\rho_{{c}_{max'}}$ $\rho_{{c}_{min'}}$ curves which map to only one possible $MSE$. Similar to \Cref{fig_maxminrho}, $MAE_1\leq MAE_2$ does not always translate to $\rho_{c_1}\leq \rho_{c_2}$.
		%	, \ie a reduction in $MAE$ or $MSE$ does not always translate to $\rho_c$ improvement. 
		For the sake of completeness, we note further that the true $\rho_{{c}_{max}}$ and consequently, the true span $\mathcal{S}$ is less than the one shown above (cf. \Cref{sec_truespan}). Note that the x-coordinate is $\frac{L_K}{N^{\frac{1}{k}}\sigma_{G}}$ and consequently, the observations are true for any $0<k\leq2$, not necessarily for $k=1$  alone.\vspace{-12pt}
		%		. Note that $1\leq\theta\leq \sqrt{N}$. 	
		%	More the number of samples ($N$), more likely the disorder in the error coefficients. Consequently, more is the likelihood that predictions are not well-correlated with the gold standard, more is the likelihood of getting low value of $\rho_c$, even for an identical $L_k$ error.
		\label{fig_maxminrhomae}}
%	\vspace{-12pt}
\end{figure}
%\vspace{-12pt}

\section{Formulation 2: Replacing $\lowercase{(y_i)}$ with $\lowercase{(x_i-d_i)}$}
\label{sec_app1}
% \subsection{Formulation 2: Replacing $(y_i)$ with $(x_i-d_i)$}
%\begin{align}
%\nonumber%%
%\mathbf{\rho_c}
%&=
%\Big(1+\frac{MSE}{2\sigma_{XY}}\Big)^{-1} = \Big(1+\frac{N\cdot MSE}{2\sum_{i=1}^{N}{(x_i-\mu_X)(y_i-\mu_Y)}}\Big)^{-1}
%\because \text{\Cref{eqn_cccmsemap}}.
%\\
%\nonumber%%
%\text{If } d_i&:=x_i-y_i,  
%\text{ and } 
%\mu_{D}:=\frac{\sum_{i=1}^{N}d_i}{N}
%\implies
%\mu_{D} 
%=\mu_X-\mu_Y 
%\hspace{6pt}
%(\because \text{\Cref{yzdef}).}
%\end{align}
%Replacing $y_i$ with $x_i-d_i$, we get:
\begin{align}
\nonumber%%
%\because \text{\Cref{eqn_cccmsemap}}.
%\implies
\text{ From \Cref{eqn_cccmsemap}, we have: }
%\therefore 
\mathbf{\rho_c}
&= \left(1+\frac{N\cdot MSE}{2\sum_{i=1}^{N}{(x_i-\mu_X)(x_i-\mu_X-d_i+\mu_D)}}\right)^{-1}.\\
\nonumber\text{Note that } \sum_{i=1}^{N}\mu_D (x_{i}-\mu_X)&=\mu_D\sum_{i=1}^{N} (X_{i}-\mu_X)=\mu_D(N\mu_X-N\mu_X)=0.
\\
\nonumber%%
\therefore \mathbf{\rho_c}
&= \left(1+\frac{N\cdot MSE}{2\sum_{i=1}^{N}(x_i-\mu_X)^2
	- 2\sum_{i=1}^{N}x_id_i
	+2\sum_{i=1}^{N}\mu_Xd_i}\right)^{-1},\\
\nonumber%%
&= \left(1+\frac{N\cdot MSE}{2N{\sigma_X}^2
	- 2\sum_{i=1}^{N}\mathbf{x_id_i}
	+2N\mu_X\mu_D}\right)^{-1}.\\
\nonumber
%\span 
\because \left(1+\frac{a}{b}\right)^{-1}=\left(1-\frac{a}{a+b}\right)
%\\
\implies
\nonumber%%
%\therefore 
\quad
%\Aboxed{
	\mathbf{\rho_c}
	&= \Bigg(1-\frac{N\cdot MSE}{2N({\sigma_X}^2+\mu_X\mu_D)
		- 2\sum_{i=1}^{N}\mathbf{x_id_i}
		+N\cdot MSE}\Bigg).
%}
\end{align}

%\newpage
\section{The Rearrangement Inequality}
\label{sec_app2}
% \subsection{The Rearrangement Inequality}
% \vspace{1cm}
%\begin{tikzpicture}[remember picture,overlay]
%\draw[line width=1,black]
%([xshift=64pt,yshift=4ex]current page text area.west|-{pic cs:beg-rearr})
%rectangle
%([xshift=-30pt,yshift=-4ex]current page text area.east|-{pic cs:end-rearr});
%\end{tikzpicture}
The rearrangement inequality is a theorem concerning the rearrangements of two sets, to maximise and minimise the sum of element-wise products \citep{hardy1988inequalities}. 

% \vspace{1cm}
Denoting the two sets $(a)=\{a_1,a_2,\cdots a_n\}$ and $(b)=\{b_1,b_2,\cdots b_n\}$, let $(\bar{a})$ and $(\bar{b})$ be permutations of $(a)$ and $(b)$ sorted respectively, such that
\begin{align}%{4}
\nonumber%%
&\bar{a}_1 &&\leq \bar{a}_2 &&\leq \cdots &&\leq \bar{a}_n,\\
\nonumber%%
\text{and } \quad &\bar{b}_1 &&\leq \bar{b}_2 &&\leq \cdots &&\leq \bar{b}_n.
\end{align}

The rearrangement inequality states that 
\begin{align}%{4}
%\box{
% \displaystyle
\nonumber%%
\tikzmark{beg-rearr}
&\hspace{-12pt}\sum_{j=1}^n \bar{a}_{j}\bar{b}_{n+1-j}
&&\hspace{-12pt}\leq
\sum_{j=1}^n a_{j}b_{j}
&&\hspace{-12pt}\leq
\sum_{j=1}^n \bar{a}_{j}\bar{b}_{j}.
\tikzmark{end-rearr}
%}
% \end{equation}
\\
\text{More explicitly, }\nonumber\\
% \begin{align}
\nonumber%%
&\hspace{-12pt}\bar{a}_{n}\bar{b}_{1}+\cdots +\bar{a}_{1}\bar{b}_{n}
&&\hspace{-12pt}\leq \bar{a}_{\sigma (1)}\bar{b}_{1}+\cdots +\bar{a}_{\sigma (n)}\bar{b}_{n}
&&\hspace{-12pt}\leq \bar{a}_{1}\bar{b}_{1}+\cdots +\bar{a}_{n}\bar{b}_{n},\\
\rlap{for every permutation $\{\bar{a}_{\sigma (1)},\bar{a}_{\sigma (2)},\cdots ,\bar{a}_{\sigma (n)}\}$ of  $\{\bar{a}_1, \cdots,\bar{a}_n\}$.}\nonumber
\end{align}
% 
% \vspace{1cm}
For the unsorted ordered sets $(a)$ and $(b)$, the two sets are said to be `similarly ordered' or `monotonic in the same sense' if $(a_{\mu}-a_{\nu})(a_{\mu}-a_{\nu})\geq0$ for all $\mu,\nu$, and `oppositely ordered' or `monotonic in the opposite sense' if the inequality is always reversed.
With this notion of `similar' and `opposite' ordering, the maximum summation corresponds to the similar ordering of $(a)$ and $(b)$. The minimum corresponds to the opposite ordering of $(a)$ and $(b)$.

\section{Relationship between p-norms}
\label{sec_app3}
According to H\"older's inequality,
\begin{alignat*}{5}
\sum\limits_{i=1}^n |a_i||b_i|\leq
\left(\sum\limits_{i=1}^n|a_i|^t\right)^{\frac{1}{t}}\left(\sum\limits_{i=1}^n|b_i|^{\frac{t}{t-1}}\right)^{1-\frac{1}{t}}.
\end{alignat*}
Let $|a_i|=|x_i|^r$, $|b_i|=1$
and $t=p/r>1$.
%r=q/p
\begin{align*}
%\sum\limits_{i=1}^n |x_i|^r & =
&
&\sum\limits_{i=1}^n |x_i|^r& \leq
\left(\sum\limits_{i=1}^n (|x_i|^r)^{\frac{p}{r}}\right)^{\frac{r}{p}}
\Bigg(\sum\limits_{i=1}^n 1^{\frac{p}{p-r}}\Bigg)^{1-\frac{r}{p}} 
&& =
\Bigg(\sum\limits_{i=1}^n |x_i|^p\Bigg)^{\frac{r}{p}} n^{1-\frac{r}{p}} .\\
%\nonumber
\therefore 
\Vert x\Vert_r & =
&\Bigg(\sum\limits_{i=1}^n |x_i|^r\Bigg)^{\frac{1}{r}} & \leq
\Bigg(\left(\sum\limits_{i=1}^n |x_i|^p\right)^{\frac{r}{p}} n^{1-\frac{r}{p}}\Bigg)^{\frac{1}{r}}  
&&=
\Bigg(\sum\limits_{i=1}^n |x_i|^p\Bigg)^{\frac{1}{p}} n^{\frac{1}{r}-\frac{1}{p}}
%\hspace{-18pt}
%\\
%\nonumber
%&&& &&\hspace{-18pt}
\quad
=
n^{1/r-1/p}\Vert x\Vert_p.
%\\
%\therefore \Vert x\Vert_r & \leq n^{1/r-1/p}\Vert x\Vert_p\span\span\span\span
%%\quad\quad\quad\quad\quad
%\\
%\nonumber
%\text{Also, when $r < p$,}
%\span\span\span\span\span
%\\
%\text{Also, when $r < p$, } \quad
%\Vert x\Vert_p  \leq \Vert x\Vert_r. \span\span\span\span
%\\
%\end{align*}
\\
%\begin{align*}
&
&&\hspace{2cm}\text{Also, when $r < p$, }\quad
\Vert x\Vert_p  &&\leq \quad \Vert x\Vert_r. \\
%\Aboxed{
&	
&&\hspace{2cm}
\therefore 
\hspace{2.6cm}
\quad
	\Vert x\Vert_p  && \leq \quad 
	\Vert x\Vert_r
	\quad  
	\leq \quad 
	n^{1/r-1/p}\Vert x\Vert_p.
%	\span\span\span\span
%}
\end{align*}
%\label{lastpage}
%\restoregeometry

\section{Chebyshev's Sum Inequality}
\label{sec_app4_chev}
%\Chebyshev
%\newcommand{\Chebyshev}{
%\begin{tikzpicture}[remember picture,overlay]
%\draw[line width=1,black]
%([xshift=72pt,yshift=2ex]current page text area.west|-{pic cs:beg-che1})
%rectangle
%([xshift=-48pt,yshift=-3ex]current page text area.east|-{pic cs:end-che1});
%%
%\draw[line width=1,black]
%([xshift=72pt,yshift=2ex]current page text area.west|-{pic cs:beg-che2})
%rectangle
%([xshift=-48pt,yshift=-3ex]current page text area.east|-{pic cs:end-che2});
%\end{tikzpicture}
The Chebyshev's sum inequality \citep{hardy1988inequalities} states that
\begin{alignat*}{5}
\tikzmark{beg-che1}
&&&\nonumber\text{if }  
&& a_{1}\geq a_{2}&&\geq \cdots \geq a_{n} 
&&\quad\text{and }\quad
{
	% \displaystyle 
	b_{1}\geq b_{2}\geq \cdots \geq b_{n},}\\
% 
%\label{eqnChev1}
&&&\text{then }\quad
&& \quad &&{1 \over n}\sum _{k=1}^{n}a_{k}\cdot b_{k}
&&\quad\geq \quad
\left({1 \over n}\sum _{k=1}^{n}a_{k}\right)\left({1 \over n}\sum _{k=1}^{n}b_{k}\right), 
\tikzmark{end-che1}
\\
\tikzmark{beg-che2}
&\text{and }\quad&&\text{if }  
&&a_{1}\leq a_{2}&&\leq \cdots \leq a_{n}
&&\quad\text{and }  \quad
b_{1}\geq b_{2}\geq \cdots \geq b_{n},\nonumber
\\
% 
%\label{eqnChev2}
&&&\text{then }\quad
&&\quad && {1 \over n}\sum _{k=1}^{n}a_{k}b_{k}
&&\quad\leq \quad
\left({1 \over n}\sum _{k=1}^{n}a_{k}\right)\left({1 \over n}\sum _{k=1}^{n}b_{k}\right).
\tikzmark{end-che2}
\end{alignat*}
%}

\section{Comparison between 
	\lowercase{$\rho_{c_{max_1}}$} AND \lowercase{$\rho_{c_{max_2}}$} 
}
\label{sec_app5_comparerho12}
%$
\vspace{-12pt}
\begin{align}%{4}
%	\nonumber
\nonumber%%
\text{Suppose}& 
\hspace{3cm}
\rho_{c_{max_1}}
&&\geq
\rho_{c_{max_2}},
\\
%	\nonumber
\nonumber%%
\span	
\text{true if }
%	&
\hspace{20pt}
2\sum _{i=1}^{N}{\bar{g_{i}}\overset{\mathit{1}}{e}_i}-2N\mu_G\mu_E
&&\geq
-2\sum _{i=1}^{N}{\bar{g_{i}}\overset{\mathit{2}}{e}_i}+2N\mu_G\mu_E,\\
%	\nonumber
\nonumber%%
\span
\text{true if }	
%	&
\hspace{34pt}
\sum _{i=1}^{N}{\bar{g_{i}}\overset{\mathit{1}}{e}_i}+\sum _{i=1}^{N}{\bar{g_{i}}\overset{\mathit{2}}{e}_i}
&&\geq
2N\mu_G\mu_E,\\
%	\nonumber
\nonumber%%
\span
\text{true if}
%	\nonumber
%	&
\hspace{9pt}
\frac{1}{N}\sum _{i=1}^{N}{\bar{g_{i}}\overset{\mathit{1}}{e}_i}+
\frac{1}{N}\sum _{i=1}^{N}{\bar{g_{i}}\overset{\mathit{2}}{e}_i}
&&\geq
\left({1 \over N}\sum _{k=1}^{N}\bar{g_{i}}\right)\left({1 \over N}\sum _{k=1}^{N}
\left(\overset{\mathit{1}}{e}_i+
\overset{\mathit{2}}{e}_i\right)\right),\\
%	\nonumber
\label{eqnNoconcl2}
\span
\text{true if}
%	&
\hspace{12pt}
\frac{1}{N}\sum _{i=1}^{N}{\bar{g_{i}}
	\Big(
	\overset{\mathit{1}}{e}_i+
	\overset{\mathit{1}}{e}_{N-i+1}\Big)}
&&\geq
\Bigg({1 \over N}\sum _{k=1}^{N}\bar{g_{i}}\Bigg)\Bigg({1 \over N}\sum _{k=1}^{N}
\Big(\overset{\mathit{1}}{e}_i+
\overset{\mathit{1}}{e}_{N-i+1}\Big)\Bigg).
\end{align}
%$
%\end{proof}

We note that
$\frac{1}{N}\sum _{i=1}^{N}{\bar{g_{i}}\overset{\mathit{1}}{e}_i}
\geq
% \displaystyle
\left({1 \over N}\sum _{k=1}^{N}\bar{g_{i}}\right)\left({1 \over N}\sum _{k=1}^{N}
\overset{\mathit{1}}{e}_i\right)$, while $\left({1 \over N}\sum _{k=1}^{N}\bar{g_{i}}\right) \cdot \left({1 \over N}\sum _{k=1}^{N}
\overset{\mathit{2}}{e}_i\right)\geq \frac{1}{N}\sum _{i=1}^{N}{\bar{g_{i}}\overset{\mathit{2}}{e}_i}$ according to the Chebyshev's sum inequality (cf. \Cref{sec_app4_chev}). Also, $\left(\overset{\mathit{1}}{e}_i+\overset{\mathit{1}}{e}_{N-i+1}\right)$ 
is symmetric with respect to $i=\frac{N+1}{2}$ (cf.  \Cref{eqn_symmetric_coef}).

Because the sequence $\smash{(\overset{\mathit{1}}{e}_i+
	\overset{\mathit{1}}{e}_{N-i+1})}$ features both the similarly ordered and the oppositely ordered error components, no guarantees can be made in terms of veracity of \Cref{eqnNoconcl2} using any of the two equalities we have used so far, \ie the Chebyshev's sum inequality (\Cref{sec_app4_chev}) and the rearrangement inequality (\Cref{sec_app2}). 
One needs to compute both the sides of the \Cref{eqnNoconcl2} to determine which permutation of the errors results in the prediction with the higher $\rho_c$.

\section{Consistency checks / validations of the formulations}
\label{sec_app6_solvethisk}

In this section, we cross-check the derived formulations, whether they are consistent with the formulations derived elsewhere independently for $k=2$.

\subsection*{\Cref{eqnsolvethisk}:}
%\EqnValidityMSE
%For example, when 
%$k=2 \implies MkE=MSE$, 
%and \Cref{eqnsolvethisk} becomes: 
%\begin{align}
%\nonumber
%\span\text{Solving for $d_i$ in terms of $y_{z_i}$, \ie $d_i=\psi(y_{z_i})$ maximises $\rho_c$}\\
%\nonumber
%\span\text{For example, for $p=2, MkE=MSE$. \Cref{solvethisk} changes to: }\\
%\end{align}
\vspace{-12pt}
\begin{align}
\nonumber%%
\therefore
0&=\mathbf{2 \sigma_{G}^2}\cdot
\left(0\right)
%\left(\frac{d_i^2}{\mathbf{N\cdot MSE}}-\frac{d_i^2}{N\cdot MSE}\right)
+\frac{\sum_{j=1}^{N}{ \mathbf{y_{z_j}}d_j}}{N}\cdot
\left(\frac{-d_i^2}{\mathbf{MSE}}\right)
%\left(\frac{d_i^2}{\mathbf{N\cdot MSE}}-2\frac{d_i^2}{N\cdot MSE}\right)
+\mathbf{y_{z_i}}{d_i}
\hspace{6pt}
\because \text{ \Cref{eqnsolvethisk} and $MkE=MSE$ if $k=2$} . 
\\
\therefore
\mathbf{y_{z_i}}{d_i}&=\frac{d_i^2}{N\cdot \mathbf{MSE}}\sum_{j=1}^{N}{\mathbf{y_{z_j}}d_j}
\implies
\mathbf{y_{z_i}}= d_i  \frac{\sigma_{GD}}{\mathbf{MSE}} 
\hspace{6pt}
\forall i\in [1, N].
%\span\span
\label{eqnyidi}\\
%\nonumber%%
%\nonumber
%\text{Multitplying both sides by $y_{z_i}$ and summing over all $i$ gives }\span\span\\
\therefore
\sum_{i=1}^{N}\mathbf{y_{z_i}^2}&=\sum_{i=1}^{N}\mathbf{y_{z_i}}d_i\Bigg(\frac{\sigma_{GD}}{\mathbf{MSE}}\Bigg) \implies N\mathbf{\sigma_{G}^2} = N\frac{\sigma_{GD}^2}{\mathbf{MSE} }
\hspace{3pt}
\therefore
\sigma_{GD}=\pm\sqrt\frac{MSE}{\sigma_{G}^2}.
\label{eqnvalidk2}
\\
\nonumber
\span\text{Substituting $\sigma_{GD}$ in \Cref{eqnyidi} for a positive covariance: }\\
\nonumber%%
\therefore d_i&=\left|\sqrt\frac{MSE}{\sigma_{G}^2}\right|\cdot y_{z_i} 
\hspace{3pt}
\forall i\in [1, N],
\hspace{3pt} \text{ (identical to \Cref{eqn_rhocmax} from \Cref{sec_m2mm}).}
%\span
\end{align}
%This result is identical to \Cref{eqn_rhocmax} derived in \Cref{sec_m2mm}.
Likewise, for $\rho_c$ minimisation, use  $f=-\frac{\sigma_{XY}}{MSE}$ and $g=N\cdot MkE-\sum_{i=1}^{N} d_i^k$ results in \Cref{eqnlagrange_funcs}. This leads to identical constraints given by \Cref{eqnlag2m}, leading to the same solution \Cref{eqnsolvethisk}, and thus, \Cref{eqnvalidk2}. However, in this case, the ratio 
$\frac{y_{z_i}}{d_i}$ needs to negative for the solution to correspond to $\rho_{c_{min}}$, resulting in \Cref{eqn_rhocmin} from \Cref{sec_m2mm}. 

\subsection*{\Cref{eqnl2rel,eqnl2rellek}:}
 Substituting $k=2$ in  \Cref{eqnl2rel,eqnl2rellek}), gives us $\theta_{max}=1$. Consequently,
%$k=2 \implies \theta_{max}=1$ ($\because$ \Cref{eqnl2rel,eqnl2rellek}).
%Correspondingly, 
\Cref{fig_maxminrhoLp} gets transformed into \Cref{fig_maxminrho}, consistent with \Cref{sec_m2mm} findings.

\section{Resources in the interest of reproducibility}
The findings and the results presented are fully reproducible with an easy-to-use scripts and data made available at \urlstyle{rm}\url{https://github.com/vedhasua/mse_ccc_corollary/}. 
The provided script uses a sample sequence from a publicly available \emph{SEWA/AVEC'17} dataset. One may also use any other ordinal dataset as an input to the provided script, and witness the results consistent to the findings presented in the manuscript.

\vskip 0.2in
%\bibliography{natcombib}
\bibliography{jmlr_ccc_mse_loss.bbl}
\end{document}